\documentclass[10pt,twocolumn,letterpaper]{article}
\usepackage[T1]{fontenc}

\usepackage{cvpr}              % To produce the CAMERA-READY version
\usepackage{anyfontsize}

\newcommand{\emphScore}[1]{\emph{#1}}

\usepackage{xargs}

\newcommand{\aka}{\emph{a.k.a.}\xspace}

\newcommand{\first}{1\textsuperscript{st}\xspace}
\newcommand{\second}{2\textsuperscript{nd}\xspace}
\newcommand{\third}{3\textsuperscript{rd}\xspace}

\usepackage{multicol}
\usepackage{longtable}
\usepackage{rotating}
\usepackage{multirow}
\usepackage{nicefrac}
\usepackage{amsthm}
\usepackage{xargs}[2008/03/08]
\usepackage{pifont}

\usepackage{thmtools}

\newtheorem{axiom}{Axiom}
\newtheorem{definition}{Definition}
\newtheorem{property}{Property}
\newtheorem{example}{Example}
\newtheorem{lemma}{Lemma}

\newcommand{\mysection}[1]{\vspace{2pt}\noindent\textbf{#1}}

\usepackage{xcolor}
\usepackage{wasysym}
\usepackage{stmaryrd} % for \shortuparrow et \shortdownarrow

\usepackage{titletoc}

\definecolor{cvprblue}{rgb}{0.21,0.49,0.74}
\usepackage[pagebackref,breaklinks,colorlinks,allcolors=cvprblue]{hyperref}

\def\paperID{***} % *** Enter the Paper ID here
\def\confName{CVPR}
\def\confYear{***}

\title{The \emph{Tile}: A 2D Map of Ranking Scores for Two-Class Classification}

\author{S\'ebastien Pi\'erard, Ana\"is Halin, Anthony Cioppa, Adrien Deli\`ege, and Marc Van Droogenbroeck\\
Montefiore Institute, University of Li\`ege, Li\`ege, Belgium\\
{\tt\small \{S.Pierard,Anais.Halin,Anthony.Cioppa,Adrien.Deliege,M.VanDroogenbroeck\}@uliege.be}
}

\begin{document}

% ---- Specific to CVPR -----------
\newcommand{\paperA}{paper~A~\cite{Pierard2024Foundations}\xspace}
\newcommand{\paperB}{paper~B~\cite{Pierard2024TheTile}\xspace}
\newcommand{\paperC}{paper~C~\cite{Halin2024AHitchhikers}\xspace}
\newcommand{\PaperA}{Paper~A~\cite{Pierard2024Foundations}\xspace}
\newcommand{\PaperB}{Paper~B~\cite{Pierard2024TheTile}\xspace}
\newcommand{\PaperC}{Paper~C~\cite{Halin2024AHitchhikers}\xspace}

% ---- All about the probability theory
\global\long\def\sampleSpace{\Omega}%
\global\long\def\aSample{\omega}%
\global\long\def\eventSpace{\Sigma}%
\global\long\def\anEvent{E}%
\global\long\def\measurableSpace{(\sampleSpace,\eventSpace)}%
\global\long\def\expectedValueSymbol{\mathbf{E}}%

% ---- All about our probabilistic framework for performances
\global\long\def\aPerformance{P}%
\global\long\def\allPerformances{\mathbb{\aPerformance}_{\measurableSpace}}%
\global\long\def\aSetOfPerformances{\Pi}%
\global\long\def\randVarSatisfaction{S}%
\global\long\def\aScore{X}%
\global\long\def\allScores{\mathbb{\aScore}_{\measurableSpace}}%
\newcommandx\domainOfScore[1][usedefault, addprefix=\global, 1=\aScore]{\mathrm{dom}(#1)}%
\global\long\def\evaluation{\mathrm{eval}}%
\global\long\def\opFilter{\mathrm{filter}_\randVarImportance}%
\global\long\def\opNoSkill{\mathrm{no\text{\textendash{}}skill}}%
\global\long\def\opPriorShift{\mathrm{shift}_{\pi\rightarrow\pi'}}%
\global\long\def\opChangePredictedClass{\mathrm{change}_{\randVarPredictedClass}}%
\global\long\def\opChangeGroundtruthClass{\mathrm{change}_{\randVarGroundtruthClass}}%
\global\long\def\opSwapGroundtruthAndPredictedClasses{\mathrm{swap}_{\randVarGroundtruthClass\leftrightarrow\randVarPredictedClass}}%
\global\long\def\opSwapClasses{\mathrm{swap}_{\classNeg\leftrightarrow\classPos}}%
\global\long\def\allWorstPerformances{\frownie}
\global\long\def\allBestPerformances{\smiley}

% ---- All about two-class classifications
\global\long\def\randVarGroundtruthClass{Y}%
\global\long\def\randVarPredictedClass{\hat{Y}}%
\global\long\def\allClasses{\mathbb{C}}%
\global\long\def\aClass{c}%
\global\long\def\classNeg{c_-}%
\global\long\def\classPos{c_+}%
\global\long\def\sampleTN{tn}%
\global\long\def\sampleFP{fp}%
\global\long\def\sampleFN{fn}%
\global\long\def\sampleTP{tp}%
\global\long\def\eventTN{\{\sampleTN\}}%
\global\long\def\eventFP{\{\sampleFP\}}%
\global\long\def\eventFN{\{\sampleFN\}}%
\global\long\def\eventTP{\{\sampleTP\}}%
\global\long\def\scorePTN{PTN}%
\global\long\def\scorePFP{PFP}%
\global\long\def\scorePFN{PFN}%
\global\long\def\scorePTP{PTP}%
\global\long\def\scoreAccuracy{A}%
\global\long\def\scoreExpectedSatisfaction{\aScore_{\randVarSatisfaction}}%
\global\long\def\scoreTNR{TNR}%
\global\long\def\scoreFPR{FPR}%
\global\long\def\scoreTPR{TPR}%
\global\long\def\scoreFNR{FNR}%
\global\long\def\scoreNPV{NPV}%
\global\long\def\scoreFOR{FOR}%
\global\long\def\scorePPV{PPV}%
\global\long\def\scorePrecision{\scorePPV}
\global\long\def\scoreFDR{FDR}%
\global\long\def\scoreJaccardNeg{J_-}%
\global\long\def\scoreJaccardPos{J_+}%
\global\long\def\scoreCohenKappa{\kappa}%
\global\long\def\scoreScottPi{\pi}%
\global\long\def\scoreFleissKappa{\kappa}%
\global\long\def\scoreBalancedAccuracy{BA}%
\global\long\def\scoreWeightedAccuracy{WA}%
%\global\long\def\scoreWeightedAccuracy{WA_{\lambda}}%
\global\long\def\scoreYoudenJ{J_Y}
\global\long\def\scorePLR{PLR}%
\global\long\def\scoreNLR{NLR}%
\global\long\def\scoreOR{OR}%
\global\long\def\scoreSNPV{SNPV}%
\global\long\def\scoreSPPV{SPPV}%
\global\long\def\scoreACP{ACP}%
\newcommandx\scoreFBeta[1][usedefault, addprefix=\global, 1=\beta]{F_{#1}}%
\global\long\def\scoreFOne{\scoreFBeta[1]}%
\global\long\def\priorpos{\pi_+}%
\global\long\def\priorneg{\pi_-}%
\global\long\def\scoreBiasIndex{BI}%
\global\long\def\ratepos{\tau_+}%
\global\long\def\rateneg{\tau_-}%
\global\long\def\scoreACP{ACP}%
\global\long\def\scorePFour{P_4}%
\global\long\def\normalizedConfusionMatrix{C}%
\global\long\def\scoreConfusionMatrixDeterminant{|\normalizedConfusionMatrix|}

% ---- All about the ranking
\global\long\def\allEntities{\mathbb{E}}%
\global\long\def\entitiesToRank{\mathbb{E}}%
\global\long\def\anEntity{\epsilon}%
% Importance-related
\global\long\def\randVarImportance{I}%
\global\long\def\randVarCanonicalImportance{\randVarImportance_{a,b}}
% ranking-related
\global\long\def\canonicalRankingScore{\rankingScore[\randVarCanonicalImportance]}
\newcommandx\rankingScore[1][usedefault, addprefix=\global, 1=\randVarImportance]{R_{#1}}%
\global\long\def\canonicalRankingScore{\rankingScore[\randVarImportance_{a,b}]}
\global\long\def\scoreVUT{VUT}%
\global\long\def\tileCurvePriors{\gamma_\pi}
\global\long\def\tileCurveRates{\gamma_\tau}
\global\long\def\relWorseOrEquivalent{\lesssim}%
\global\long\def\relBetterOrEquivalent{\gtrsim}%
\global\long\def\relEquivalent{\sim}%
\global\long\def\relBetter{>}%
\global\long\def\relWorse{<}%
\global\long\def\relIncomparable{\not\lesseqqgtr}%
\global\long\def\rank{\mathrm{rank}_\entitiesToRank}%
\global\long\def\ordering{\relWorseOrEquivalent}%
\global\long\def\invertedOrdering{\relBetterOrEquivalent}%

% ---- All about the experiment symbols
\global\long\def\LScityscapes{\ding{171}}%spade
\global\long\def\LSade{\ding{170}}%heart
\global\long\def\LSvoc{\ding{169}}%diamond
\global\long\def\LScoco{\ding{168}}%club

% ---- Other mathematical things
\global\long\def\indicatorSymbol{\mathbf{1}}% see https://en.wikipedia.org/wiki/Indicator_function
\global\long\def\realNumbers{\mathbb{R}}%
\global\long\def\aRelation{\mathcal{R}}%
\global\long\def\achievableByCombinations{\Phi}%
\global\long\def\allConvexCombinations{\mathrm{conv}}%
\newcommand{\indep}{\perp \!\!\! \perp}

% ---- Specific to paper 3 -------

% datasets
\global\long\def\cityscapes{\LScityscapes{}~Cityscapes}
\global\long\def\ade{\LSade{}~ADE20K}
\global\long\def\voc{\LSvoc{}~Pascal VOC 2012}
\global\long\def\coco{\LScoco{}~COCO-Stuff 164k}

% Scenarios
\newcommand{\MethodDesigner}{Bernadette}
\newcommand{\Benchmarker}{Leonard}
\newcommand{\AppDeveloper}{Howard}
\newcommand{\TheoreticalAnalyst}{Sheldon}

\newcommand{\tile}{Tile\xspace}
\newcommand{\tiles}{Tiles\xspace}
\newcommand{\valueTile}{Value Tile\xspace}
\newcommand{\baselineTile}{Baseline Value Tile\xspace}
\newcommand{\SOTATile}{State-of-the-Art Value Tile\xspace}
\newcommand{\noSkillTile}{No-Skill Tile\xspace}
\newcommand{\skillTile}{Relative-Skill Tile\xspace}
\newcommand{\correlationTile}{Correlation Tile\xspace}
\newcommand{\rankingTile}{Ranking Tile\xspace}
\newcommand{\entityTile}{Entity Tile\xspace}

\global\long\def\aNonSkilledPerformance{\aPerformance_{\indep}}% % useful ?????
\global\long\def\allNonSkilledPerformances{\mathbb{\aPerformance}^{\randVarGroundtruthClass\indep\randVarPredictedClass}_{\measurableSpace}}%

\global\long\def\allPriorFixedPerformances{\mathbb{\aPerformance}^{\priorpos}_{\measurableSpace}}%

% Added by MVD, for the display of punctuations in equations
\newcommand{\comma}{\,,}
\newcommand{\point}{\,.}

% ---------------------- about Scores -----------
% Unconditional scores
\newcommandx\unconditionalProbabilisticScore[1]{\aScore_{#1}^{U}}%
\global\long\def\formulaPTN{\unconditionalProbabilisticScore{\eventTN}}%
\global\long\def\formulaPFP{\unconditionalProbabilisticScore{\eventFP}}%
\global\long\def\formulaPFN{\unconditionalProbabilisticScore{\eventFN}}%
\global\long\def\formulaPTP{\unconditionalProbabilisticScore{\eventTP}}%
\global\long\def\formulapriorneg{\unconditionalProbabilisticScore{\{\sampleTN,\sampleFP\}}}%
\global\long\def\formulapriorpos{\unconditionalProbabilisticScore{\{\sampleFN,\sampleTP\}}}%
\global\long\def\formularateneg{\unconditionalProbabilisticScore{\{\sampleTN,\sampleFN\}}}%
\global\long\def\formularatepos{\unconditionalProbabilisticScore{\{\sampleFP,\sampleTP\}}}%
\global\long\def\formulaAccuracy{\unconditionalProbabilisticScore{\{\sampleTN,\sampleTP\}}}%

% Conditional scores
\newcommandx\conditionalProbabilisticScore[2]{\aScore_{#1 \vert #2}^{C}}%
\global\long\def\formulaTNR{\conditionalProbabilisticScore{\{\sampleTN\}}{\{\sampleTN,\sampleFP\}}}%
\global\long\def\formulaTPR{\conditionalProbabilisticScore{\{\sampleTP\}}{\{\sampleFN,\sampleTP\}}}%
\global\long\def\formulaNPV{\conditionalProbabilisticScore{\{\sampleTN\}}{\{\sampleTN,\sampleFN\}}}%
\global\long\def\formulaPPV{\conditionalProbabilisticScore{\{\sampleTP\}}{\{\sampleFP,\sampleTP\}}}%
\global\long\def\formulaJaccardNeg{\conditionalProbabilisticScore{\{\sampleTN\}}{\{\sampleTN,\sampleFP,\sampleFN\}}}%
\global\long\def\formulaJaccardPos{\conditionalProbabilisticScore{\{\sampleTP\}}{\{\sampleFP,\sampleFN,\sampleTP\}}}%

% Non-probabilistic scores
\global\long\def\scoreBennettS{S}

% -------------- end of about Scores -----------

\renewcommand{\paperB}{paper~B\xspace}
\renewcommand{\PaperB}{Paper~B\xspace}

\maketitle

\begin{abstract}

In the computer vision and machine learning communities, as well as in many other research domains, rigorous evaluation of any new method, including classifiers, is essential. One key component of the evaluation process is the ability to compare and rank methods.
However, ranking classifiers and accurately comparing their performances, especially when taking application-specific preferences into account, remains challenging. 
For instance, commonly used evaluation tools like Receiver Operating Characteristic (ROC) and Precision/Recall (PR) spaces display performances based on two scores.
Hence, they are inherently limited in their ability to compare classifiers across a broader range of scores and lack the capability to establish a clear ranking among classifiers.
In this paper, we present a novel versatile tool, named the \emph{\tile}, that organizes an infinity of ranking scores in a single 2D map for two-class classifiers, including common evaluation scores such as the accuracy, the true positive rate, the positive predictive value, Jaccard's coefficient, and all $\scoreFBeta$ scores. 
Furthermore, we study the properties of the underlying ranking scores, such as the influence of the priors or the correspondences with the ROC space, and depict how to characterize any other score by comparing them to the \tile. %} 
Overall, we demonstrate that the \tile is a powerful tool that effectively captures all the rankings in a single visualization and allows interpreting them. 
\footnote{This paper is the second of a trilogy. 
In a nutshell, \paperA presents an axiomatic framework and an infinite family of scores for ranking classifiers. In this paper (\paperB), we particularize this framework to two-class classification and introduce the \emph{\tile}, a visual tool that organizes these scores in a single 2D map. Finally, \paperC 
provides a guide to using the \tile according to four practical scenarios. 
For that, we present different \tile flavors that are applied to a real application.
}
\end{abstract}

\section{Introduction}
\label{sec:intro}

\begin{figure}[ht]
\begin{centering}
\resizebox{1\linewidth}{!}{
\includegraphics[width=.6\linewidth]{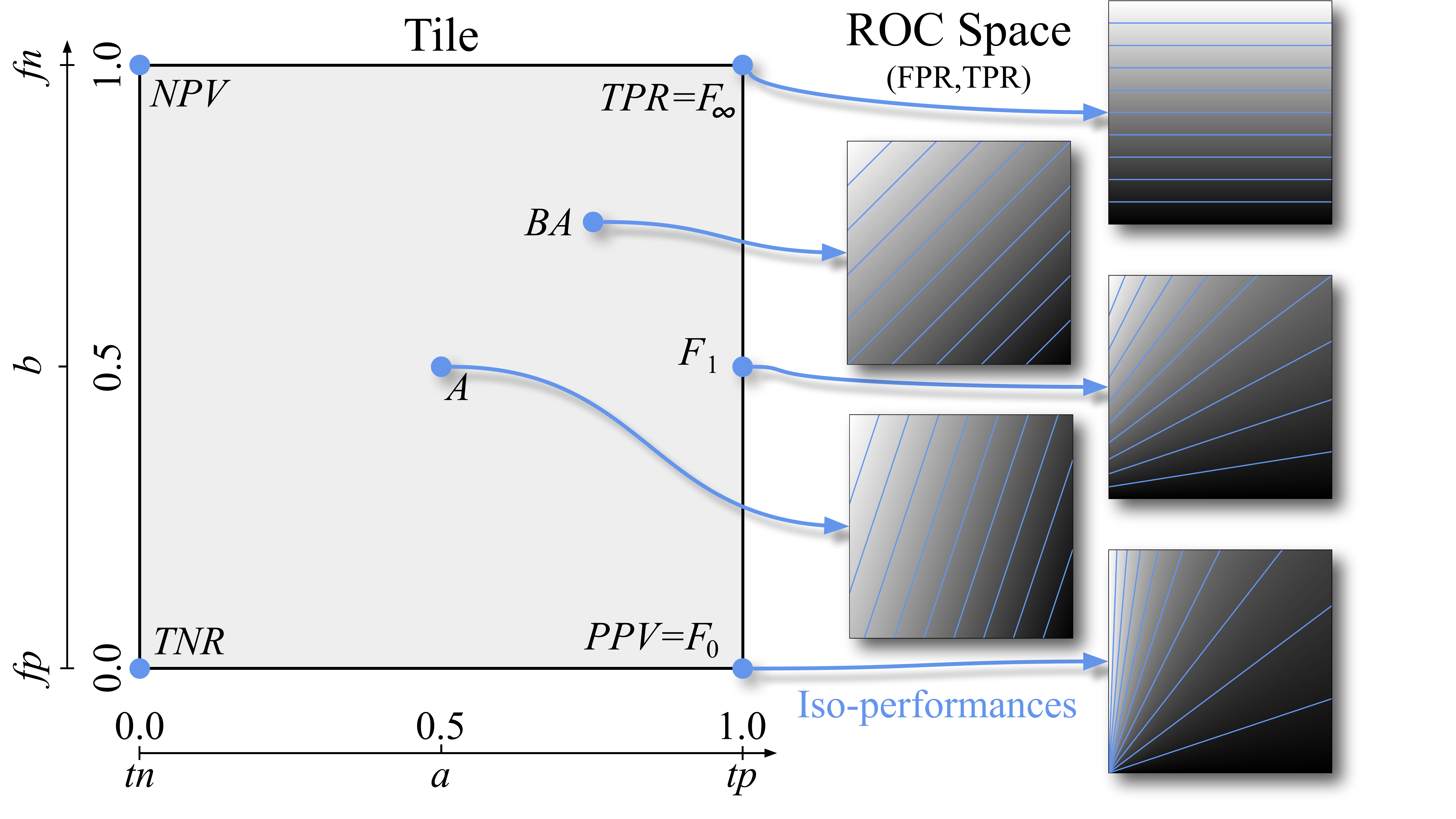}}
\par\end{centering}
\caption{\textbf{Introducing the \tile.} We introduce a new visual tool, called the \tile, representing an infinite family of ranking scores to evaluate the performances of two-class classifiers at a glance. 
In this figure, we highlight the correspondences between specific ranking scores on the \tile{} and their corresponding set of iso-performance lines in the ROC space. Notably, the variation of iso-performance lines along the right border of the \tile{} demonstrates the limitations of the ROC space for ranking performance. This visualization illustrates how the \tile{} simplifies the task of ranking classifiers and enhances the interpretation of performance scores across various evaluation spaces, such as the ROC space.
\label{fig:graphical_abstract}}
\end{figure}

Two-class classification is a fundamental task, encountered in numerous real-world scenarios. For instance, it plays a vital role in medical diagnostics, such as blood tests, MRI scans, and other imaging techniques to determine whether a patient has a disease or is healthy. In security systems, alarms must activate only when intrusions are detected. Similarly, in quality control, identifying defects in manufactured items is crucial to ensure faulty products do not reach the market. To address these challenges, selecting the right classifier is essential. However, this requires ranking classifiers in the context of application-specific preferences. For example, in medical testing, minimizing false negatives is critical since failing to diagnose a patient could have life-threatening consequences. In security systems, the focus is often on maximizing true negatives, accepting occasional false alarms as a trade-off for ensuring safety. Meanwhile, in quality control, false positives can be costly, as they may trigger unnecessary halts in production. Each application thus has unique requirements regarding the types of errors a classifier can tolerate. 

A wide range of scores penalizing different types of errors are available in the literature. However, selecting the appropriate score to rank classifiers taking application-specific preferences into account can be challenging. Additionally, the common practice of ranking based on a single score, as done in most benchmarks, can lead to a suboptimal classifier choice. Moreover, so-called evaluation spaces combining two scores, such as Receiver Operating Characteristic (ROC) and Precision/Recall (PR) spaces, do not allow ranking classifiers. This raises a recurring question: how can classifiers be effectively ranked to best align with the specific needs of each application?

In this paper, we introduce a novel visual tool for two-class classification, called the \emph{\tile}, which organizes an infinity of performance orderings derived from ranking scores into a two-dimensional map. 
Our \tile is parametrized by two parameters reflecting application-specific preferences: the first controls the trade-off between true positives and true negatives, while the second balances false positives and false negatives. This parametrization allows mapping common performance scores from the literature, such as the accuracy or $\scoreFOne$, as illustrated on the left side of \cref{fig:graphical_abstract}. 
Next, we analyze the correspondences between the \tile and standard evaluation spaces such as ROC and PR, with a particular emphasis on iso-performance lines. As shown in \cref{fig:graphical_abstract}, iso-performance lines can be drawn on the ROC space for each score of the \tile. Even though rankings of classifiers on the ROC space is dependent on the choice of a score, our \tile enables a straightforward, unified interpretation of the ranking of two-class classifiers at a glance.
Finally, we demonstrate that our \tile's organization allows to easily rank classifiers and study some ranking properties such as the orderings induced by all ranking scores, the characterization of any score, and the robustness of ranking scores.

\mysection{Contribution.} We summarize our contributions as follows.
\textbf{(i)} We introduce a novel visual tool, called the \tile, that organizes an infinity of ranking scores on a two-dimensional map.
\textbf{(ii)} We analyze the correspondences between the \tile and common evaluation spaces such as ROC and PR, with a particular focus on iso-performance lines.
\textbf{(iii)} We show that our \tile's organization allows to easily compare ranking scores, rank classifiers, and study ranking properties.

\section{Preliminaries and Related Work}
\label{sec:related-work}

We first present the necessary preliminaries, including the mathematical framework, definitions of the relevant scores, their underlying structure, and a review of related work to set the context. 
By coherence, we adopt the mathematical framework, terminology, and notations introduced in \paperA. Note that all acronyms and mathematical symbols used in this paper are defined where they appear\footnote{A list of them is provided in the supplementary material.}.

\subsection{Mathematical Framework}

The framework is based on the probability theory. A score $\aScore$ is a real-valued function defined on a subset $\domainOfScore$ of the performance space $\allPerformances$. The latter is the set of all possible probability measures on the measurable space $\measurableSpace$, where the sample space (\aka universe) is $\sampleSpace$, and the event space (\aka $\sigma$-algebra) is $\eventSpace$. A performance $\aPerformance$ is an element of $\allPerformances$, thus, a probability measure. 

Moreover, the framework is also based on the order theory. The symbols $\ordering_\aScore$ and $\invertedOrdering_\aScore$ are used to denote, respectively, the ordering and dual ordering induced by the score $\aScore$. 
When both performances $\aPerformance_1,\aPerformance_2\in\domainOfScore$, $\aPerformance_1$ is \emph{worse than}, \emph{equivalent to}, or \emph{better than} $\aPerformance_2$ when $\aScore(\aPerformance_1)<\aScore(\aPerformance_2)$, $\aScore(\aPerformance_1)=\aScore(\aPerformance_2)$, or $\aScore(\aPerformance_1)>\aScore(\aPerformance_2)$, respectively. When $\aPerformance_1\notin\domainOfScore$ or $\aPerformance_2\notin\domainOfScore$, they are \emph{equivalent} if $\aPerformance_1=\aPerformance_2$, and \emph{incomparable} otherwise.

\mysection{Two-class crisp classification.}
We particularize this probabilistic framework to the case in which the sample space contains two satisfying and two unsatisfying elements, thus $|\sampleSpace|=4$ and $\eventSpace=2^\sampleSpace$. Using the binary random variable $\randVarSatisfaction:\sampleSpace\rightarrow\{0,1\}$ to denote the satisfaction, we have  $|\{\aSample\in\sampleSpace:\randVarSatisfaction(\aSample)=0\}|=|\randVarSatisfaction^{-1}(\{0\})|=2$, and $|\{\aSample\in\sampleSpace:\randVarSatisfaction(\aSample)=1\}|=|\randVarSatisfaction^{-1}(\{1\})|=2$. 
Variable $\randVarSatisfaction$ determines how the elements of $\sampleSpace$ are interpreted. 
The most popular interpretation is undoubtedly the popular \emph{two-class crisp classification}. 
For this choice, we take $\sampleSpace=\{\sampleTN,\sampleFP,\sampleFN,\sampleTP\}$. The samples $\sampleTN$, $\sampleFP$, $\sampleFN$ and $\sampleTP$ are interpreted as, respectively, a \emph{true negative}, a \emph{false positive} (type I error, false alarm), a \emph{false negative} (type II error, miss), and a \emph{true positive} (hit). In that case, $\randVarSatisfaction(\sampleTN)=\randVarSatisfaction(\sampleTP)=1$ and $\randVarSatisfaction(\sampleFP)=\randVarSatisfaction(\sampleFN)=0$.

\mysection{Classes and predictions}. We can introduce the set of classes $\allClasses=\{\classNeg,\classPos\}$, as well as the %discrete 
random variables $\randVarGroundtruthClass:\sampleSpace\rightarrow\allClasses$ and $\randVarPredictedClass:\sampleSpace\rightarrow\allClasses$ such that $\randVarGroundtruthClass(\sampleTN)=\randVarGroundtruthClass(\sampleFP)=\classNeg$, $\randVarGroundtruthClass(\sampleFN)=\randVarGroundtruthClass(\sampleTP)=\classPos$, $\randVarPredictedClass(\sampleTN)=\randVarPredictedClass(\sampleFN)=\classNeg$, and $\randVarPredictedClass(\sampleFP)=\randVarPredictedClass(\sampleTP)=\classPos$. Indeed, $\randVarGroundtruthClass$ and $\randVarPredictedClass$ can be interpreted as the ground-truth and predicted classes. 
The \emph{no-skill performances} %(\aka random)
are those such that $\aPerformance(\randVarGroundtruthClass,\randVarPredictedClass)=\aPerformance(\randVarGroundtruthClass) \aPerformance(\randVarPredictedClass)$. 
The satisfaction is the indicator $\randVarSatisfaction=\indicatorSymbol_{\randVarGroundtruthClass=\randVarPredictedClass}$.  
% The satisfaction is an indicator: see https://en.wikipedia.org/wiki/Indicator_function

\subsection{Scores}

In the literature, numerous performance scores\footnote{In this paper we choose the term \emph{score} to avoid any possible confusion with the mathematical meaning of the terms \emph{metric}, \emph{measure} and \emph{indicator}.} (also called metrics, measures, indicators, criteria, factors, and indices~\cite{Texel2013Measure,Canbek2017Binary}) have been introduced for two-class crisp classification. In fact, the list of scores is almost endless, as can be seen from numerous reviews~\cite{Ballabio2018Multivariate,Canbek2017Binary,Choi2010Survey,Ferri2009AnExperimental,Powers2011Evaluation,Sokolova2009ASystematic}. Choosing one score over another depends on the application field (medical, machine learning, statistics, \etc).

\global\long\def\toMaximize{}%
\global\long\def\toMinimize{}%
\global\long\def\notOptimizable{}%

\mysection{Unconditional probabilistic scores.} 
Unconditional probabilistic scores, denoted by $\unconditionalProbabilisticScore{\anEvent}$, are parameterized by an event $\anEvent\in\eventSpace$ such that $\emptyset\subsetneq \anEvent\subsetneq \sampleSpace$:
\begin{equation}
    \unconditionalProbabilisticScore{\anEvent}:\allPerformances\rightarrow[0,1]:\aPerformance\mapsto
    \aPerformance(\anEvent) \,.
\end{equation}
There exist only $14$ unconditional probabilistic scores. For singleton events, there is $\scorePTN=\formulaPTN$ \notOptimizable (called \emphScore{rejection rate}), $\scorePFP=\formulaPFP$ \notOptimizable, $\scorePFN=\formulaPFN$ \notOptimizable, and $\scorePTP=\formulaPTP$ \notOptimizable (called \emphScore{detection rate}). There are also the \emphScore{priors} for the negative and positive classes, given by $\priorneg=\formulapriorneg$ \notOptimizable and $\priorpos=\formulapriorpos$ \notOptimizable respectively, as well as the \emphScore{negative} and \emphScore{positive prediction rates}, given by $\rateneg=\formularateneg$ \notOptimizable and $\ratepos=\formularatepos$ \notOptimizable respectively. The \emphScore{prevalence} is a synonym used for $\priorpos$. The \emphScore{accuracy} \toMaximize (\aka \emphScore{matching coefficient}) corresponds to the expected value of $\randVarSatisfaction$ and is given by $\scoreAccuracy=\formulaAccuracy$. Its complement~\cite{Canbek2017Binary} is the \emphScore{error rate} or \emphScore{misclassification rate} \toMinimize.

\mysection{Conditional probabilistic scores.}
The conditional probabilistic scores $\conditionalProbabilisticScore{\anEvent_1}{\anEvent_2}$ are parameterized by two events, $\anEvent_1,\anEvent_2\in\eventSpace$ such that $\emptyset\subsetneq \anEvent_1\subsetneq \anEvent_2\subseteq\sampleSpace$:
\begin{equation}
    \conditionalProbabilisticScore{\anEvent_1}{\anEvent_2}:\domainOfScore[\conditionalProbabilisticScore{\anEvent_1}{\anEvent_2}]\rightarrow[0,1]:\aPerformance\mapsto
    \aPerformance(\anEvent_1\vert \anEvent_2) \comma
\end{equation}
with $\domainOfScore[\conditionalProbabilisticScore{\anEvent_1}{\anEvent_2}]=\{\aPerformance\in\allPerformances:\aPerformance(\anEvent_2)\ne0\}$. There exist only $50$ conditional probabilistic scores (including the $14$ unconditional ones, as $\unconditionalProbabilisticScore{\anEvent}=\conditionalProbabilisticScore{\anEvent}{\sampleSpace}$). The probabilities of making a correct decision for negative and positive inputs are given by the \emphScore{True Negative Rate} $\scoreTNR=\formulaTNR$ \toMaximize (\aka specificity, selectivity, inverse recall) and the \emphScore{True Positive Rate} $\scoreTPR=\formulaTPR$ \toMaximize (\aka \emphScore{sensitivity}, \emphScore{recall}). Their complements are the \emphScore{False Positive Rate} $\scoreFPR$ \toMinimize and the \emphScore{False Negative Rate} $\scoreFNR$ \toMinimize respectively. The probabilities of negative and positive predictions being correct are given by the \emphScore{Negative Predictive Value} $\scoreNPV=\formulaNPV$ \toMaximize (\aka inverse precision) and the \emphScore{Positive Predictive Value} $\scorePPV=\formulaPPV$ \toMaximize (\aka precision). Their complements are the \emphScore{False Omission Rate} $\scoreFOR$ \toMinimize, and the \emphScore{False Discovery Rate} $\scoreFDR$ \toMinimize respectively. \emphScore{Jaccard's coefficient} is $\scoreJaccardNeg=\formulaJaccardNeg$ \toMaximize for the negative class and $\scoreJaccardPos=\formulaJaccardPos$ \toMaximize for the positive class. The latter is also called \emphScore{Tanimoto coefficient}, \emphScore{similarity}, \emphScore{intersection over union}, \emphScore{critical success index}~\cite{Hogan2010Equitability}, as well as \emphScore{G-measure} in~\cite{Flach2003TheGeometry}.

\mysection{Non-probabilistic scores.}
There is also an infinite number of scores that have no probabilistic significance.

We start with transformations of probabilistic scores. \emphScore{Bennett's} $\scoreBennettS$ \toMaximize~\cite{Warrens2012TheEffect} is related to the accuracy by $\scoreBennettS=2\scoreAccuracy-1$. The \emphScore{F-one}  score \toMaximize, also called \emphScore{Dice-S{\o}rensen coefficient}, is related to Jaccard by $\scoreFOne=\nicefrac{2\scoreJaccardPos}{\scoreJaccardPos+1}$. The \emphScore{standardized negative and predictive values} are transformations of the negative and predictive values given by $\scoreSNPV=\frac{\scoreTNR}{\scoreTNR+\scoreFNR}=\frac{\scoreNPV\priorpos}{\scoreNPV(\priorpos-\priorneg)+\priorneg}$ \toMaximize and $\scoreSPPV=\frac{\scoreTPR}{\scoreFPR+\scoreTPR}=\frac{\scorePPV\priorneg}{\scorePPV(\priorneg-\priorpos)+\priorpos}$ \toMaximize~\cite{Heston2011Standardizing}. The \emphScore{likelihood ratios}~\cite{Gardner2006Receiver‐operating,Glas2003TheDiagnosticOddsRatio,Powers2011Evaluation,Brown2006ROC} are also transformations of these scores. The \emphScore{Negative Likelihood Ratio} is $\scoreNLR=\frac{\scoreFNR}{\scoreTNR}=\frac{1-\scoreSNPV}{\scoreSNPV}$ \toMinimize. 
The \emphScore{Positive Likelihood Ratio}~\cite{Altman1994Diagnostic} is $\scorePLR=\frac{\scoreTPR}{\scoreFPR}=\frac{\scoreSPPV}{1-\scoreSPPV}$ \toMaximize.
\emphScore{Cohen's $\scoreCohenKappa$ statistic}~\cite{Cohen1960ACoefficient, Kaymak2012TheAUK} \notOptimizable is a transformation of the accuracy: $\scoreCohenKappa=\frac{\scoreAccuracy-\scoreAccuracy\circ\opNoSkill}{1-\scoreAccuracy\circ\opNoSkill}$ where $\circ$ is the function composition operator and $\opNoSkill$ is the operation that transforms a performance $\aPerformance$ into $\aPerformance'$ such that $\aPerformance'(\randVarGroundtruthClass,\randVarPredictedClass)=\aPerformance(\randVarGroundtruthClass)\aPerformance(\randVarPredictedClass)$. It is also known as \emphScore{Heidke Skill Score}~\cite{Canbek2017Binary,Wilks2020Statistical}.
It is also common to combine probabilistic scores, either linearly or by averaging them. For example, the \emphScore{Bias Index} $\scoreBiasIndex$ \notOptimizable has been defined as $\ratepos-\priorpos$ in \cite{Byrt1993Bias}.

Many authors combine $\scoreTNR$ and $\scoreTPR$. Their weighted arithmetic mean is the \emphScore{weighted accuracy} $\scoreWeightedAccuracy$ \toMaximize. When the weights are $0.5$, we obtain the \emphScore{balanced accuracy} $\scoreBalancedAccuracy$ \toMaximize, and when the weights correspond to the priors, we get the \emphScore{accuracy}. Instead of taking an arithmetic mean, it has been proposed to take the geometric mean~\cite{Barandela2003Strategies,Guo2008OnTheClass}, which leads to the \emphScore{G-measure}~\cite{Canbek2017Binary} \toMaximize (not to be confused with the G-measure of~\cite{Flach2003TheGeometry}). 
\emphScore{Youden's index} \cite{Youden1950Index, Fluss2005Estimation} or \emphScore{Youden's} $\scoreYoudenJ$ \emphScore{statistic} is $\scoreYoudenJ=\scoreTNR+\scoreTPR-1=2\scoreBalancedAccuracy-1$ \toMaximize. It is also called \emphScore{informedness} and \emphScore{Peirce Skill Score}~\cite{Canbek2017Binary,Wilks2020Statistical}. The \emphScore{determinant} $\scoreConfusionMatrixDeterminant$ \notOptimizable~\cite{Wimmer2006APerson} of the (normalized) confusion matrix is $\scoreConfusionMatrixDeterminant=\priorneg\priorpos\scoreYoudenJ$. 

Some authors prefer to combine $\scorePPV$ with $\scoreTPR$. Their weighted harmonic mean is the \emphScore{F-score} $\scoreFBeta$ \toMaximize.
Others prefer to combine $\scoreNPV$ with $\scorePPV$. The \emphScore{markedness}~\cite{Powers2020Evaluation-arxiv} is defined as $\scoreNPV+\scorePPV-1$ \toMaximize and is also known as the \emphScore{Clayton Skill Score}~\cite{Canbek2017Binary,Wilks2020Statistical}.
It has also been proposed to combine the $4$ probabilistic scores $\scoreTNR$, $\scoreTPR$, $\scoreNPV$, and $\scorePPV$. Their arithmetic mean is called \emphScore{Average Conditional Probability} $\scoreACP$ \toMaximize \cite{Burset1996Evaluation}, and their harmonic mean is $\scorePFour$ \toMaximize \cite{Sitarz2023Extending}.
Finally, a plethora of other scores can be found in the literature concerning similarity scores~\cite{Batyrshin2016Visualization,Brusco2021AComparison,Choi2010Survey,Gower1986Metric,Mejia2018Towards,Warrens2013Comparison}. Many of them are peculiar cases of the ranking scores that are discussed in detail in this paper.

\subsection{Structuring the Scores}
\label{sub:structuring-the-scores}

In~\cite{Ferri2009AnExperimental}, $18$ scores have been experimentally structured in the form of histograms of linear and rank correlations, based on $30$ datasets. 
Similarly, Choi \etal~\cite{Choi2010Survey} made a hierarchical clustering of $76$ binary similarity and distance scores based on a random binary dataset. In~\cite{Mejia2018Towards}, $7$ properties have been arbitrarily defined, and $11$ scores have been classified into $5$ classes based on whether the properties are verified. 
In~\cite{Texel2013Measure}, taking the object-oriented software development standpoint, the structure takes the form of a \emph{Unified Modeling Language} (UML) diagram to represent the concepts of measure, metric, and indicator, as well as the relationships between the three concepts. 
\cite{Canbek2017Binary} proposes to distinguish between only two types of scores: the measures and the metrics. 
Based on a sample of $44$ scores and related quantities ($22$ measures and $22$ metrics), they propose to divide measures into $4$ levels (base, \first, \second and \third), and metrics into 3 levels (base, \first and \second). The $4$ basic measures are the elements of the confusion matrix (also known as the contingency matrix). Based on this, the authors draw what they describe as ``the periodic table of elements in binary classification performance''. It is a map of the $44$ scores, organized around the confusion matrix, the vertical dimension being related to the ground-truth class and the horizontal dimension to the predicted class.
In this work, we avoid mixing scores and only consider those suitable for  performance ordering and performance-based ranking of classifiers. 

\subsection{Reminders on the Axioms of Ranking}

We also leverage the axiomatic definitions around the notion of performance introduced in \paperA. In a nutshell, the \first axiom states that if several performances have been ranked, then removing or adding a performance should not affect the relative order of the previously present performances. We will reuse the second and third axioms later in this paper, rephrased as follows. 

\setcounter{axiom}{1} % The next axiom will be Axiom 2
\begin{axiom}[reminder]
\label{axiom:satisfaction}
If the degree of satisfaction that can be obtained with a \first classifier is for sure less or equal than the degree of satisfaction that can be obtained with a \second one, then the former classifier is not better than the latter.
\end{axiom}
\setcounter{axiom}{2} % The next axiom will be Axiom 3
\begin{axiom}[reminder]
\label{axiom:combinations}
Given a set of classifiers, and a arbitrary set of possible operations to perturb (\eg, add noise to their output) and/or combine them, then, on the basis of their performances only, it must be impossible to determine a sequence of these operations that would lead with certainty to a classifier either better than the best of the initial ones or worse than the worst of the initial ones.
\end{axiom}
Moreover, we show for the first time that it is possible to establish a continuous, two-dimensional map of an infinity of scores, and that the map includes many scores that are widespread in the literature.

\section{Ranking Scores for Two-Class Classification}
\label{sec:ranking-scores}

\subsection{Particularization of Ranking Scores}

To build the \tile, we consider the family of ranking scores in the special case of the two-class crisp classification task:

\begin{equation}
    \rankingScore(\aPerformance)=\frac{
        \sum_{\aSample\in\{\sampleTN,\sampleTP\}} \randVarImportance(\aSample)\aPerformance(\{\aSample\})
    }
    {
        \sum_{\aSample\in\{\sampleTN,\sampleFP,\sampleFN,\sampleTP\}} \randVarImportance(\aSample)\aPerformance(\{\aSample\})
    } \,
\end{equation}
where $\randVarImportance:\sampleSpace\rightarrow\realNumbers_{\ge0}$ is a non-negative random variable, called \emph{importance}, such that $\randVarImportance(\sampleTN)+\randVarImportance(\sampleTP)\ne0$ and $\randVarImportance(\sampleFP)+\randVarImportance(\sampleFN)\ne0$. 
These scores allow for the comparison of all two-class classification performances in $\domainOfScore[\rankingScore]$, even when the classes are imbalanced ($\priorneg\ne\priorpos$).
 
\begin{example}[Probabilistic ranking scores]
    \label{example:probabilistic-ranking-scores}
    The family of ranking scores includes the following $9$ probabilistic scores: 
    $\scoreNPV$, % 1
    $\conditionalProbabilisticScore{\{\sampleTN,\sampleTP\}}{\{\sampleTN,\sampleFN,\sampleTP\}}$, % 2
    $\scoreTPR$, % 3
    $\scoreJaccardNeg$, % 4
    $\scoreAccuracy$, % 5
    $\scoreJaccardPos$, % 6
    $\scoreTNR$, % 7
    $\conditionalProbabilisticScore{\{\sampleTN,\sampleTP\}}{\{\sampleTN,\sampleFP,\sampleTP\}}$, % 8
    and $\scorePPV$. % 9
\end{example}

\begin{example}[F-scores]
    For all $\beta\ge0$, $\scoreFBeta$ is a ranking score.
\end{example}

\begin{example}[PABDC]
    The ranking scores for which the importance values $\randVarImportance(\aSample)$ are rational numbers correspond to the class of \emph{Presence/Absence Based Dissimilarity Coefficients} (PABDC) satisfying the first $9$ properties listed in~\cite{Baulieu1989AClassification} (see Prop.~1 in that paper).
\end{example}

\subsection{Properties}

\global\long\def\samplesNeg{\{\sampleTN,\sampleFP\}}
\global\long\def\samplesPos{\{\sampleFN,\sampleTP\}}

We start by examining the effect of the target/prior shift operation~\cite{Sipka2022TheHitchhikerGuide}, denoted by $\opPriorShift$. It transforms a performance $\aPerformance$ into $\aPerformance'$ such that $\aPerformance'(\anEvent)=\aPerformance(\anEvent)\frac{\priorneg'}{\priorneg}$ for all $\anEvent\in2^{\samplesNeg}$ and $\aPerformance'(\anEvent)=\aPerformance(\anEvent)\frac{\priorpos'}{\priorpos}$ for all $\anEvent\in2^{\samplesPos}$.

\begin{property}
    The composition of a ranking score with a target/prior shift operation is a ranking score. We have $\rankingScore[\randVarImportance]\,\circ\,\opPriorShift=\rankingScore[\randVarImportance']$ with $\randVarImportance'(\aSample)=\randVarImportance(\aSample)\frac{\priorneg'}{\priorneg}$ for all $\aSample\in\samplesNeg$ and $\randVarImportance'(\aSample)=\randVarImportance(\aSample)\frac{\priorpos'}{\priorpos}$ for all $\aSample\in\samplesPos$. 
\end{property}

As we have particularized the axiomatic framework to the particular case of two-class crisp classification, the following property holds. 
\begin{property}
    \label{prop:scale-invariance-per-satisfaction}
    Multiplying both $\randVarImportance(\sampleTN)$ and $\randVarImportance(\sampleTP)$, or both $\randVarImportance(\sampleFP)$ and $\randVarImportance(\sampleFN)$, by the same strictly positive constant leads to another ranking score that is monotonously increasing with the original one. The induced performance ordering is thus unaffected by such a transformation.
\end{property}
\noindent Thanks to the last property, we can get rid of the redundancy between the performance orderings induced by the ranking scores by focusing on the \emph{canonical ranking scores}.
\begin{definition}
    Canonical ranking scores are given by
    \begin{equation*}
        \canonicalRankingScore=\frac{(1-a)\scorePTN+a\scorePTP}{(1-a)\scorePTN+(1-b)\scorePFP+b\scorePFN+a\scorePTP},
    \end{equation*}
    where $a,b\in[0,1]$ and $\randVarCanonicalImportance$ is the importance given by $\randVarCanonicalImportance(\sampleTN)=1-a$, $\randVarCanonicalImportance(\sampleFP)=1-b$, $\randVarCanonicalImportance(\sampleFN)=b$, and $\randVarCanonicalImportance(\sampleTP)=a$.
\end{definition}
\noindent Consequently, the $\scoreNPV$, $\scorePPV$, $\scoreTNR$, $\scoreTPR$, $\scoreAccuracy$, and $\scoreFBeta$ scores belong to the canonical ranking scores. 

To avoid potential issues, it should be noted that averaging ranking scores cannot be done without precautions. In fact, there are only a few special cases in which a score obtained in this way can be used for ranking. The same precaution also applies for canonical ranking scores. For example, because of this issue, the orderings induced by $\scoreACP$, $\scorePFour$, and $\scoreVUT=\int_0^1 \int_0^1 \canonicalRankingScore\,da\,db$ are incompatible with the axioms of ranking. If a compromise has to be found between several canonical ranking scores, we recommend averaging the importances rather than the scores. %themselves.

\subsection{Correspondences with the ROC Space}

We now study the correspondences between the canonical ranking scores and the \emph{Receiver Operating Characteristic} (ROC) space (\ie, $\scoreFPR\times\scoreTPR$) for fixed priors.

\begin{figure}
\begin{centering}
\subfloat[The locus of equivalent performances (\ie, those for which the ranking score takes a given value) are lines (restricted to ROC). These lines form a pencil whose vertex $(fpr_v,tpr_v)$ (in red) is located outside the ROC space, either in the bottom-left or upper-right areas (in gray). The value taken by the score varies linearly along any line of slope $-\priorneg/\priorpos$ (in purple).\label{fig:ranking_scores_in_roc_1}]{\begin{centering}
\includegraphics[width=0.8\linewidth]{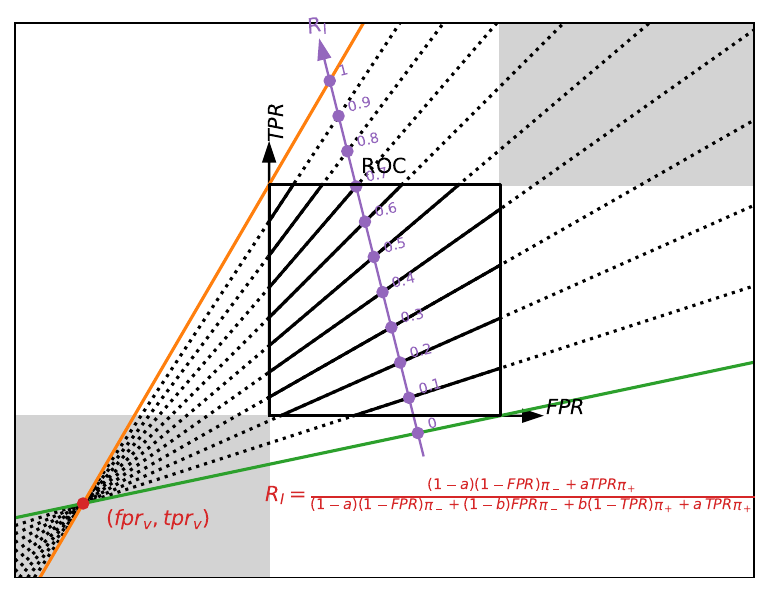}
\par\end{centering}

}
\par\end{centering}
\begin{centering}
\subfloat[The locus of the vertices $(fpr_v,tpr_v)$ (red point) for all ranking scores with a given value of $a$ are lines (restricted to the gray areas). These lines form a pencil whose vertex is located at the bottom-right corner of ROC. The value of $a$ varies linearly along any line of slope $-\priorneg/\priorpos$ (in purple).
\label{fig:ranking_scores_in_roc_2}]{\begin{centering}
\includegraphics[width=0.8\linewidth]{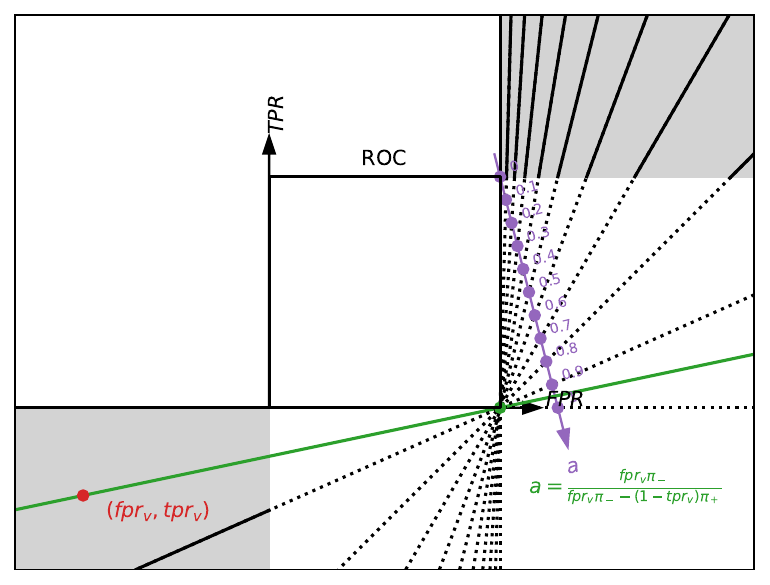}
\par\end{centering}
}
\par\end{centering}
\begin{centering}
\subfloat[The locus of vertices $(fpr_v,tpr_v)$ (red point) for all ranking scores with a given value of $b$ are lines (restricted to the gray areas). These lines form a pencil whose vertex is located at the top-left corner of ROC. The value of $b$ varies linearly along any line of slope $-\priorneg/\priorpos$ (in purple).
\label{fig:ranking_scores_in_roc_3}]{\begin{centering}
\includegraphics[width=0.8\linewidth]{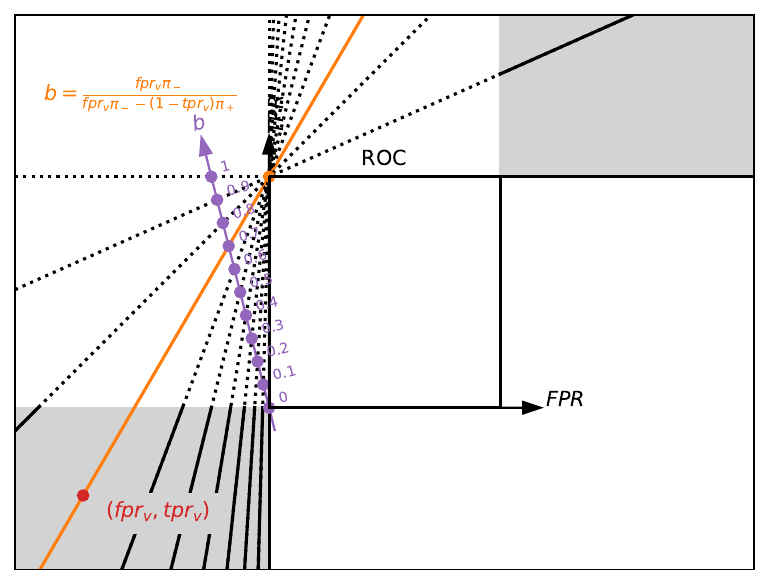}
\par\end{centering}
}
\par\end{centering}
\caption{The geometry of the ranking scores $\canonicalRankingScore$ in the ROC plane ($\scoreFPR,\scoreTPR)$. Example given for the class priors $\priorpos=1-\priorneg=0.2$ and the importance given by $(a,b)=(0.95,0.7)$.\label{fig:ranking_scores_in_roc}}

\end{figure}

Three pencils of lines are depicted in \cref{fig:ranking_scores_in_roc}. The first pencil (see~\cref{fig:ranking_scores_in_roc_1}) can be used to read the value of $\canonicalRankingScore$ in any point of ROC (on any line of slope $-\priorneg/\priorpos$, in purple in the figure), when the location of the red point is known. The red point corresponds to the locus of iso-performance lines for a given ranking score. It corresponds to the intersection of the green and orange lines, that is, the iso-performance lines for which the score is equal to 0 and 1, respectively. The equations of these lines can be obtained by looking at other pencils, as following. 
The second pencil (see~\cref{fig:ranking_scores_in_roc_2}) can be used to find the green line based on the value of the parameter $a$ (the green line is vertical when $a=0$ and horizontal when $a=1$). Finally, the third pencil (see~\cref{fig:ranking_scores_in_roc_3}) can be used to find the orange line based on the value of the parameter $b$ (the orange line is vertical when $b=0$ and horizontal when $b=1$). This procedure generalizes the geometric constructions provided by Flach in~\cite{Flach2003TheGeometry} for $\scoreAccuracy$, $\scoreFOne$, and $\scorePPV$.

\mysection{Contour plots.} \Cref{fig:ranking_scores_in_roc_1} can be considered as a \emph{contour plot} depicting the preorder induced by a score. The depicted curves (lines in our case) are called \emph{iso-performance lines} in~\cite{Provost2001Robust} and \emph{isometrics} in \cite{Flach2003TheGeometry}. Applying any monotonic function to the score leaves these curves unchanged. 
Hence, a line corresponds to a set of \emph{equivalent} performances. All performances that are on the top-left side of the line are \emph{better than} them, and all performances that are on the bottom-right side of the line are \emph{worse than} them. Note that this geometric analysis is not peculiar to the canonical ranking scores. It is valid for all ranking scores, as for any ranking score there exists a canonical ranking score such that the orderings induced by them are equal.

\mysection{Interpretation with respect to Axiom~\ref{axiom:satisfaction}.}
It is also interesting to note that, for any $\randVarImportance$, the red point is in one of the gray areas (extending to infinity). When $a<b$, the red point is on the right and above ROC, and when $a>b$, it is on the left and below ROC. When $a=b$, all lines in \cref{fig:ranking_scores_in_roc_1} (including the green and orange ones) are parallel (\eg, for $\scoreTNR$, $\scoreAccuracy$, and $\scoreTPR$) and the red point is a point at infinity. % See https://en.wikipedia.org/wiki/Point_at_infinity
This is related to Axiom~\ref{axiom:satisfaction}. Indeed, in the case of fixed priors, the performance orderings induced by any score that has a pencil of iso-performance lines is a performance ordering that can be induced by a ranking score. This is because the family of ranking scores covers all cases where the vertex of the pencil is in one of the gray areas. Putting the vertex outside these areas would be illogical as there would either be better performances than the one located in the upper-left corner of ROC, or worse than the one located in the lower-right corner of ROC, which is prohibited by Axiom~\ref{axiom:satisfaction}.

\mysection{Interpretation with respect to Axiom~\ref{axiom:combinations}.}
As the ROC space is, for fixed priors, a linear projection of $\allPerformances$, any convex in the ROC space corresponds to a convex in $\allPerformances$. Therefore, a convex combination of performances cannot be better than the best of the combined performances and cannot be worse than the worst of the combined performances.

\mysection{On the diversity.}
The performance orderings induced by the scores $\canonicalRankingScore$ are all different. For two different ranking scores, one has different pencil vertices (red points), leading to different iso-performance lines, sets of equivalent performances, and performance orderings. Hence, there is no redundancy in the canonical ranking scores.

\section{The \tile for Two-Class Classification}
\label{sec:tile}

% I_tn = 1 - a
% I_fp = 1 - b
% I_fn = b
% I_tp = a

The \tile, which is depicted in \cref{fig:tile}, is defined as follows.
\begin{definition}
     The \tile is the mapping $[0,1]^2\rightarrow\allScores : (a,b)\mapsto\canonicalRankingScore$, where $\allScores$ denotes the set of all scores.
\end{definition}
Although attempts to organize a selection of two-class classification scores in the 2D plane are common, to our knowledge, this is the first time that it is done mathematically, quantitatively, and automatically, without the intervention of a human expert. The \tile is not limited to the spatial organization of scores, however.
\begin{figure}
\newcommand{\xlabel}{$a$}
\newcommand{\ylabel}{$b$}
\newcommand{\rankingScoreA}{\scoreNPV} % In math mode
\newcommand{\rankingScoreB}{\scoreTPR=\scoreFBeta[\infty]} % In math mode
\newcommand{\rankingScoreC}{\scoreTNR} % In math mode
\newcommand{\rankingScoreD}{\scorePPV=\scoreFBeta[0]} % In math mode
\newcommand{\rankingScoreE}{\scoreAccuracy} % In math mode
\newcommand{\rankingScoreF}{\scoreFOne} % In math mode
\newcommand{\rankingScoreI}{} % In math mode
\newcommand{\orderingAa}{\ordering_{\scoreNPV}} % In math mode
\newcommand{\orderingAb}{\ordering_{\scoreSNPV}^\dagger,\invertedOrdering_{\scoreNLR}^\dagger} % In math mode
\newcommand{\orderingBa}{\ordering_{\scoreTPR},\ordering_{\scoreFBeta[\infty]},} % In math mode
\newcommand{\orderingBb}{\ordering_{\scorePTP}^\dagger,\invertedOrdering_{\scorePFN}^\dagger} % In math mode
\newcommand{\orderingCa}{\ordering_{\scoreTNR}} % In math mode
\newcommand{\orderingCb}{\ordering_{\scorePTN}^\dagger,\invertedOrdering_{\scorePFP}^\dagger,} % In math mode
\newcommand{\orderingDa}{\ordering_{\scorePPV},\ordering_{\scoreFBeta[0]}} % In math mode
\newcommand{\orderingDb}{\ordering_{\scorePLR}^\dagger,\ordering_{\scoreSPPV}^\dagger} % In math mode
\newcommand{\orderingE}{\ordering_{\scoreAccuracy}} % In math mode
\newcommand{\orderingF}{\ordering_{\scoreFOne},\ordering_{\scoreJaccardPos}} % In math mode
\newcommand{\orderingG}{\ordering_{\scoreCohenKappa}^\dagger} % In math mode
\newcommand{\orderingH}{\ordering_{\scoreBalancedAccuracy}^\dagger,\ordering_{\scoreYoudenJ}^\dagger,\ordering_{\scoreConfusionMatrixDeterminant}^\dagger} % In math mode
\newcommand{\orderingI}{\ordering_{\scoreJaccardNeg}} % In math mode
\newcommand{\orderingJ}{\ordering_{\scoreWeightedAccuracy}^\dagger}
\newcommand{\titleOrderings}{Performance Orderings}
\newcommand{\titleRankingScores}{Canonical Ranking Scores}

\begin{centering}
\resizebox{1\linewidth}{!}{
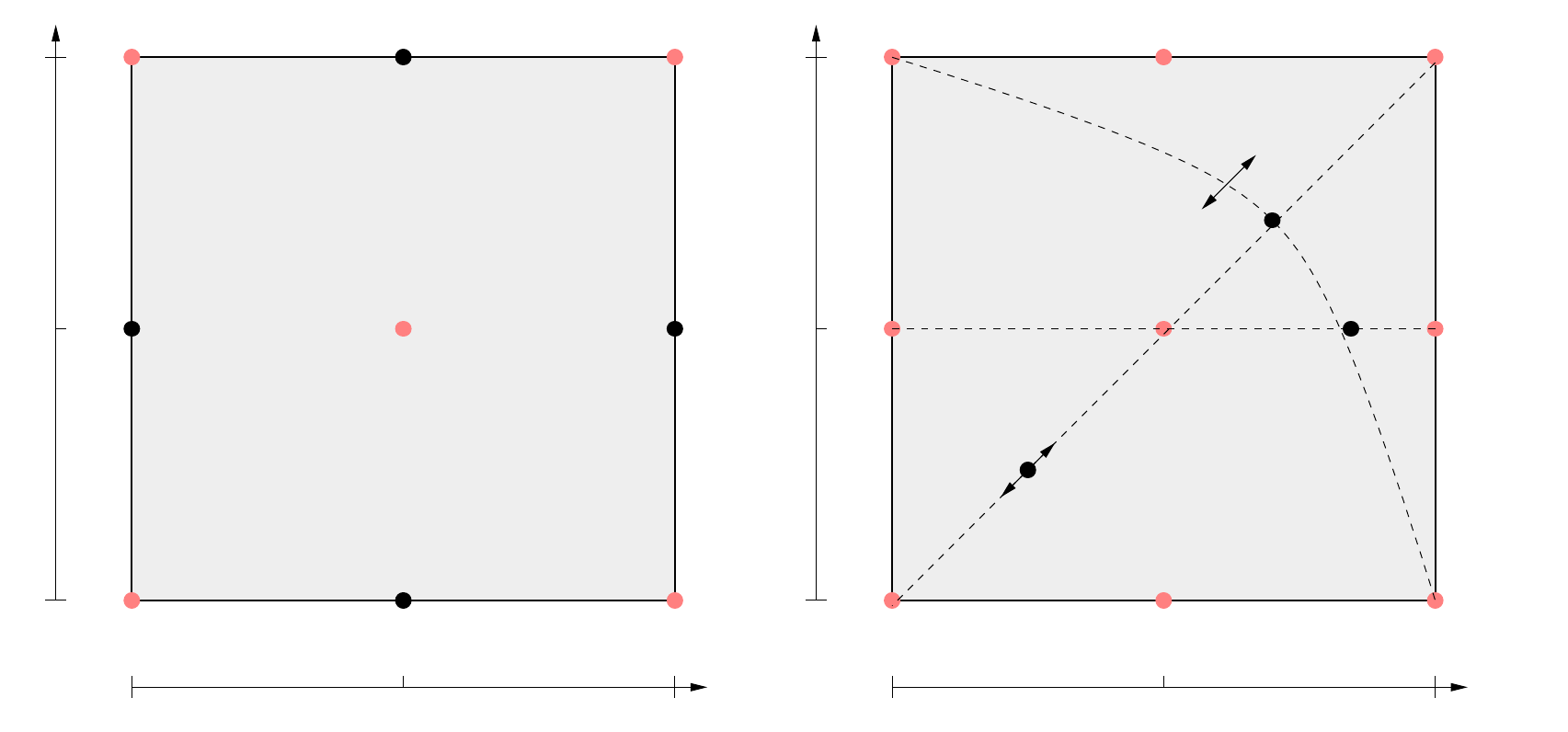
}
\par\end{centering}
\caption{Placement of the canonical ranking scores (left) and of some performance orderings (right) on the \tile. The symbol $\dagger$ indicates the orderings that are specific for given priors. For the orderings whose locations are prior-dependent, we arbitrarily chose a negative prior of $0.7$. Double arrows $\leftrightarrow$ indicate the direction in which $\ordering_{\scoreWeightedAccuracy}$ moves when the weights are tuned and how the curve on which $\ordering_{\scoreBalancedAccuracy}$ and $\ordering_{\scoreCohenKappa}$ moves when the priors are tuned. The colored points correspond to probabilistic scores.\label{fig:tile}}
\end{figure}

\subsection{Canonical Ranking Scores on the \tile}
\label{sec:canonical-ranking-scores-on-tile}

The left-hand side of \cref{fig:tile} shows the layout of the \tile with some canonical ranking scores on it. Two opposite corners correspond to $\scoreNPV$ and $\scorePPV$, and the other corners to $\scoreTNR$ and $\scoreTPR$. The accuracy $\scoreAccuracy$ is in the center. These are the only $5$ canonical ranking scores that are also probabilistic scores. $\scoreFBeta$ scores are between $\scoreTPR$ and $\scorePPV$ ($\beta=\sqrt{\nicefrac{b}{1-b}}$), $\scoreFOne$ being in the middle of the right side.

\mysection{Interpolations.} 
By construction, the canonical scores can be interpolated, in the \tile, as follows: 
    \begin{itemize}
        \item horizontally: with the $f$-mean such that $f:x\mapsto x^{-1}$; % (\ie, the harmonic mean);
        \item vertically: with the $f$-mean such that $f:x\mapsto (1-x)^{-1}$.
    \end{itemize}
\noindent The fact that $\scoreFOne$ (the harmonic mean between $\scoreTPR$ and $\scorePPV$) appears at the mid-position between $\scoreTPR$ and $\scorePPV$ is a consequence of this property.

\mysection{Indistinguishable samples.}
The scores that can be calculated when the two unsatisfying outcomes ($\sampleFP$ and $\sampleFN$) are grouped together (\ie, $\sampleSpace=\{\sampleTN,\sampleTP,incorrect\}$) are those for which $\randVarImportance(\sampleFP)=\randVarImportance(\sampleFN)$. They are located on the median horizontal, passing through $\scoreAccuracy$. 
Likewise, the scores that can be calculated when two satisfying outcomes ($\sampleTN$ and $\sampleTP$) are grouped together (\ie, $\sampleSpace=\{\sampleFP,\sampleFN,correct\}$) are those for which $\randVarImportance(\sampleTN)=\randVarImportance(\sampleTP)$. They are located on the median vertical, passing through $\scoreAccuracy$.

\mysection{Operations on performances.}
Some remarkable geometric properties of the \tile are related to the operations that can be applied on performances. They are given hereafter.
\begin{itemize}
    \item Changing either the predicted or ground-truth class amounts to, respectively, flipping the \tile around the raising diagonal ($\scoreTNR$-$\scoreTPR$ axis) or falling diagonal ($\scoreNPV$-$\scorePPV$ axis), and complementing the scores to 1.
    \item Swapping the predicted and ground-truth classes is equivalent to vertical mirroring.
    \item Swapping the positive and negative classes is equivalent to applying a central symmetry.
    \item The target/prior shift operation~\cite{Sipka2022TheHitchhikerGuide} moves the performance orderings on the \tile. We found it very useful, when priors are fixed, to represent the displacement that would have occurred if we had started with uniform priors and applied the target/prior shift. 
    % For example, in \cref{fig:toy-example}, white grids are superimposed on the Tiles. Arbitrarily, we choose to use a regular grid when $\priorneg=\priorpos=\nicefrac12$.
\end{itemize}

\subsection{Performance Orderings on the \tile}
\label{sec:performance-orderings-on-tile}

\begin{figure}
\hfill{}
\begin{minipage}[t]{0.3\linewidth}%
\begin{center}
%{\scriptsize{}$\frac{\scorePTN}{\scorePTN+\scorePFN}$}
{\scriptsize{}$\scoreNPV$}
\par\end{center}%
\end{minipage}
\hfill{}\hfill{}
\begin{minipage}[t]{0.3\linewidth}%
\begin{center}
%{\scriptsize{}$\frac{\scorePTN+\scorePTP}{\scorePTN+\scorePFN+\scorePTP}$}
{\scriptsize{}$\conditionalProbabilisticScore{\{\sampleTN,\sampleTP\}}{\{\sampleTN,\sampleFN,\sampleTP\}}$}
\par\end{center}%
\end{minipage}
\hfill{}\hfill{}
\begin{minipage}[t]{0.3\linewidth}%
\begin{center}
%{\scriptsize{}$\frac{\scorePTP}{\scorePFN+\scorePTP}$}
{\scriptsize{}$\scoreTPR$}
\par\end{center}%
\end{minipage}
\hfill{}

\hfill{}
\includegraphics[width=0.3\linewidth]{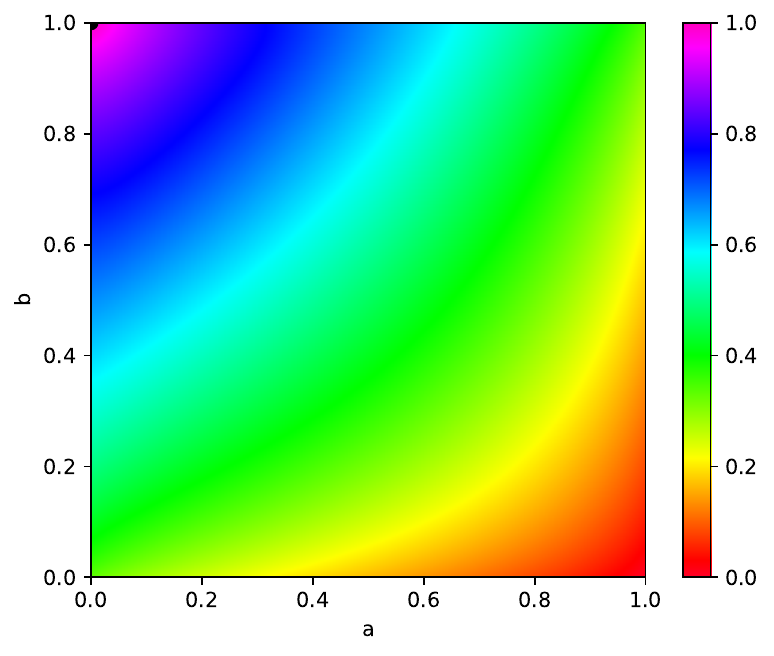}
\hfill{}\hfill{}
\includegraphics[width=0.3\linewidth]{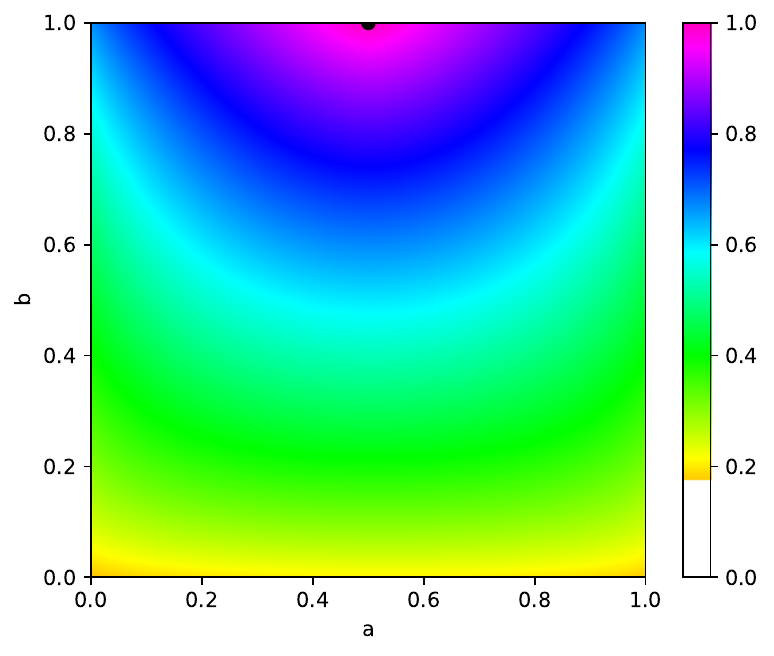}
\hfill{}\hfill{}
\includegraphics[width=0.3\linewidth]{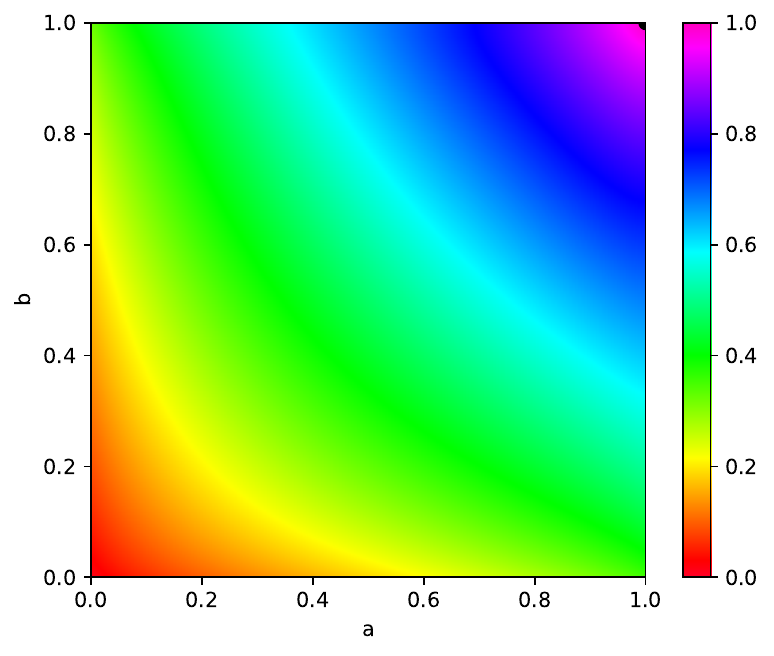}
\hfill{}

\hfill{}
\begin{minipage}[t]{0.3\linewidth}%
\begin{center}
{\scriptsize{}$\scoreJaccardNeg$}
\par\end{center}%
\end{minipage}
\hfill{}\hfill{}
\begin{minipage}[t]{0.3\linewidth}%
\begin{center}
{\scriptsize{}$\scoreAccuracy$}
\par\end{center}%
\end{minipage}
\hfill{}\hfill{}
\begin{minipage}[t]{0.3\linewidth}%
\begin{center}
{\scriptsize{}$\scoreJaccardPos$}
\par\end{center}%
\end{minipage}
\hfill{}

\hfill{}
\includegraphics[width=0.3\linewidth]{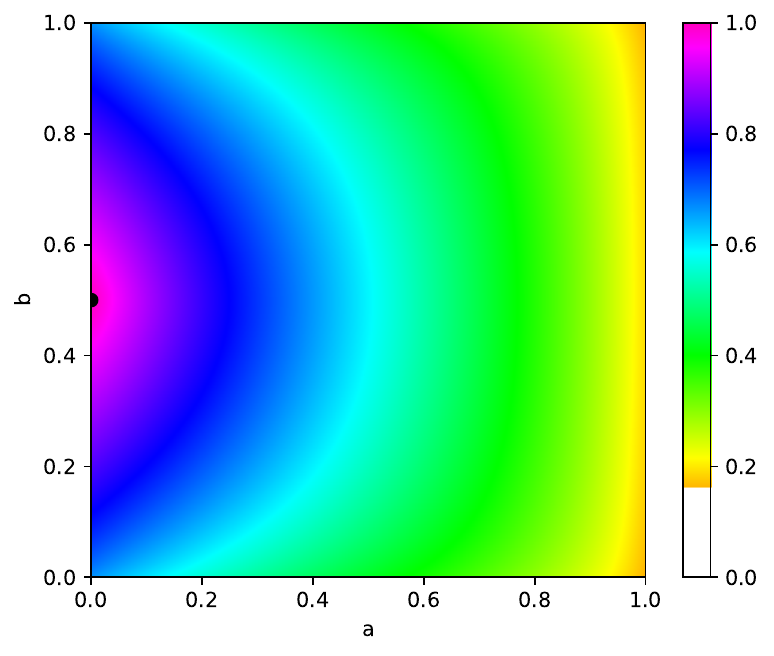}
\hfill{}\hfill{}
\includegraphics[width=0.3\linewidth]{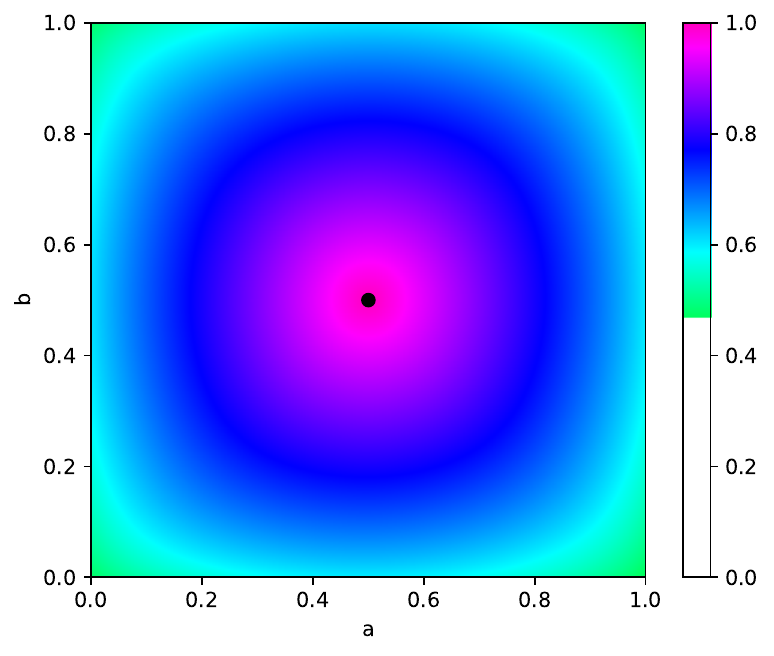}
\hfill{}\hfill{}
\includegraphics[width=0.3\linewidth]{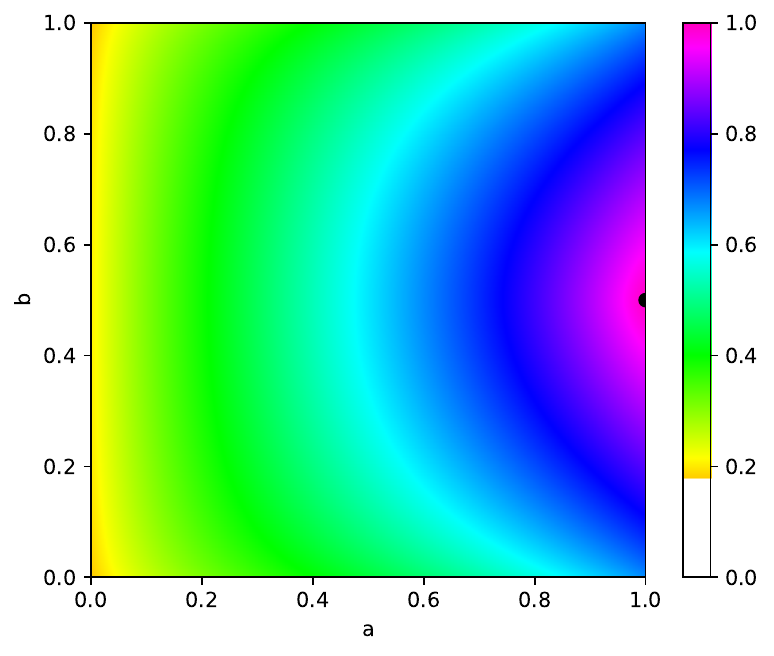}
\hfill{}

\hfill{}
\begin{minipage}[t]{0.3\linewidth}%
\begin{center}
{\scriptsize{}$\scoreTNR$}
\par\end{center}%
\end{minipage}
\hfill{}\hfill{}
\begin{minipage}[t]{0.3\linewidth}%
\begin{center}
{\scriptsize{}$\conditionalProbabilisticScore{\{\sampleTN,\sampleTP\}}{\{\sampleTN,\sampleFP,\sampleTP\}}$}
\par\end{center}%
\end{minipage}
\hfill{}\hfill{}
\begin{minipage}[t]{0.3\linewidth}%
\begin{center}
{\scriptsize{}$\scorePPV$}
\par\end{center}%
\end{minipage}
\hfill{}

\hfill{}
\includegraphics[width=0.3\linewidth]{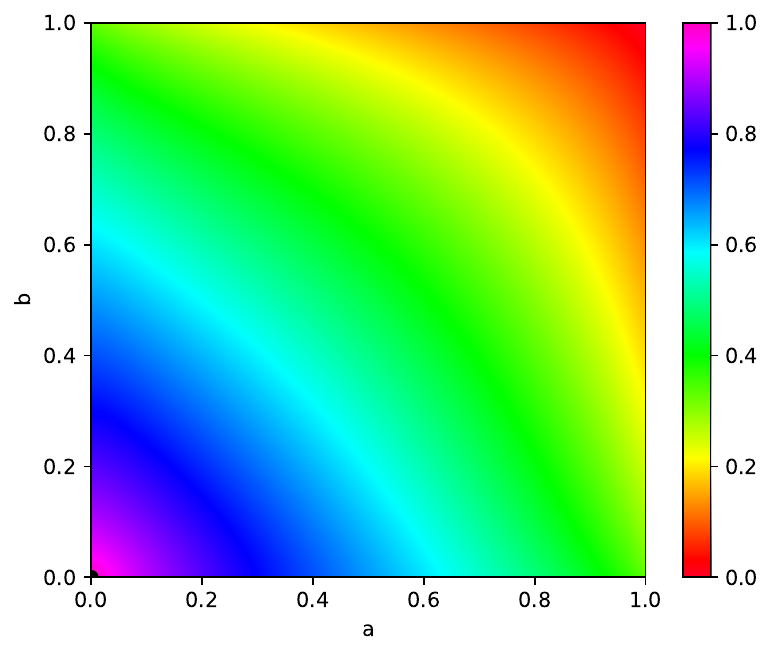}
\hfill{}\hfill{}
\includegraphics[width=0.3\linewidth]{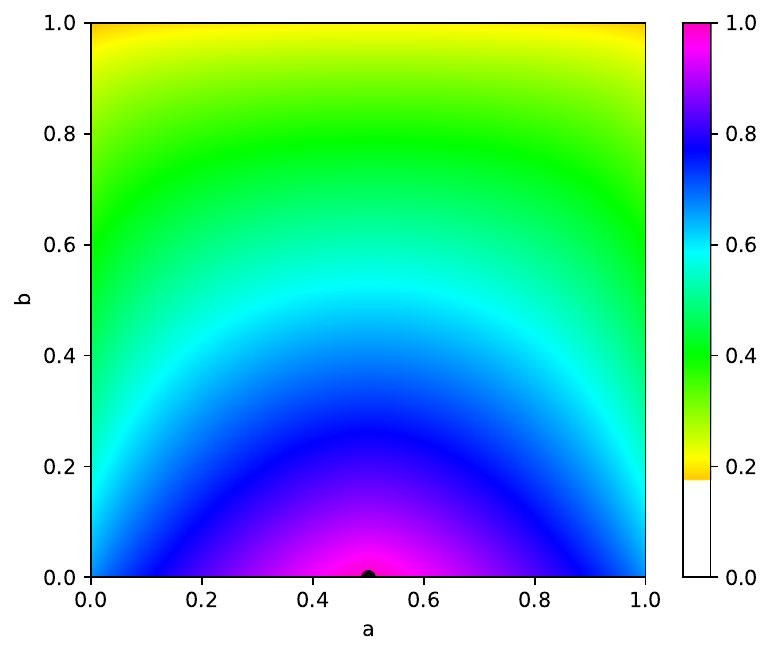}
\hfill{}\hfill{}
\includegraphics[width=0.3\linewidth]{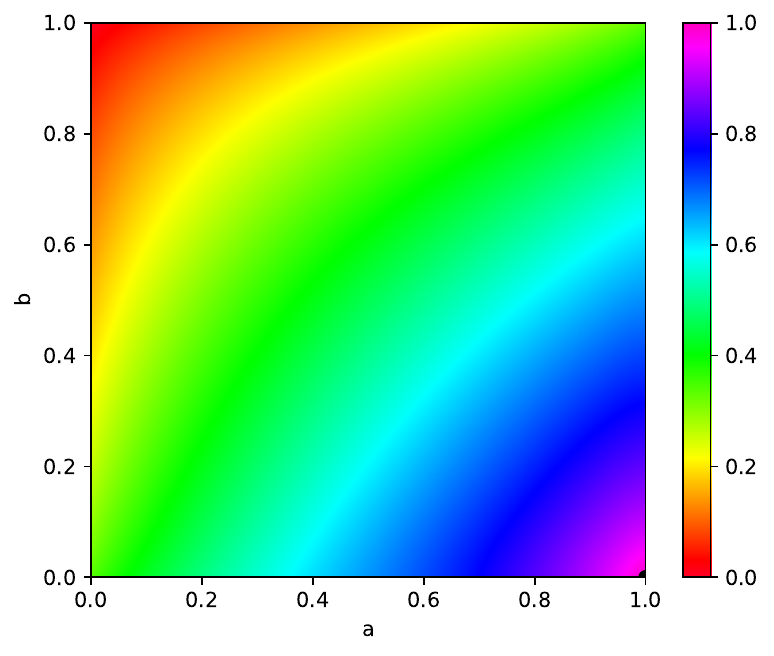}
\hfill{}

\caption{\tiles showing the rank correlations (Kendall $\tau$) between $9$ probabilistic scores (those that belong to the ranking scores, as given in \cref{example:probabilistic-ranking-scores}), and all ranking scores, for a uniform distribution of performances. The correlation values have been estimated based on $10{,}000$ performances drawn at random.\label{fig:correlation-with-probabilistic-ranking-scores}}

\end{figure}

We have put some performance orderings on the right-hand side of \cref{fig:tile} on the \tile. For a position $(a,b)$, we have mentioned the orderings induced by the canonical ranking score $\canonicalRankingScore$, and all the other orderings that are equal to it, either on $\allPerformances$ or on a subset of it (fixed priors).

\global\long\def\aRankingScore{\frac{\randVarImportance(\sampleTP)}{\randVarImportance(\sampleTN)+\randVarImportance(\sampleTP)}}
\global\long\def\bRankingScore{\frac{\randVarImportance(\sampleFP)}{\randVarImportance(\sampleFN)+\randVarImportance(\sampleFP)}}

\mysection{With all performances.} 
The ordering $\ordering_{\rankingScore}$ corresponds to location $(a,b)=(\aRankingScore,\bRankingScore)$. Those induced by the $9$ probabilistic ranking scores given in \cref{example:probabilistic-ranking-scores} are depicted by colored points on \cref{fig:tile}. While it was not possible to place $\scoreJaccardNeg$ and $\scoreJaccardPos$ on the \tile of canonical ranking scores, the corresponding orderings $\ordering_{\scoreJaccardNeg}$ and $\ordering_{\scoreJaccardPos}$ are on the \tile of performance orderings. The well-known fact that $\scoreFOne$ and $\scoreJaccardPos$ lead to the same ordering~\cite{Flach2003TheGeometry} can be easily visualized on the \tile since $\ordering_{\scoreFOne}$ and $\ordering_{\scoreJaccardPos}$ are at the same place. All orderings induced by the similarity coefficients of the $T_\theta$ and $S_\theta$ families, defined in~\cite{Gower1986Metric}, are equal to $\ordering_{\scoreFOne}$ and $\ordering_{\scoreAccuracy}$, respectively. The orderings induced by the family of similarity coefficients defined in~\cite{Batyrshin2016Visualization} are the ones on the line segment between $\ordering_{\scoreAccuracy}$ and $\ordering_{\scoreFOne}$.

\global\long\def\formulaWeightedAccuracy{\lambda_-\scoreTNR+\lambda_+\scoreTPR}
\global\long\def\aWeightedAccuracy{\frac{\lambda_+\priorneg}{\lambda_+\priorneg+\lambda_-\priorpos}}
\global\long\def\bWeightedAccuracy{\aWeightedAccuracy}
\global\long\def\aBalancedAccuracy{\priorneg}
\global\long\def\bBalancedAccuracy{\aBalancedAccuracy}
\global\long\def\aCohen{\frac{\priorneg^2}{\priorneg^2+\priorpos^2}}
\global\long\def\bCohen{\frac12}

\mysection{With fixed priors.}
It is also possible to place, on the \tile, the orderings that are equal to $\ordering_{\canonicalRankingScore}$ for a given subset of performances only. An important practical case is when the priors are fixed. In this case, the orderings induced by the unconditional probabilistic scores $\scorePTN$, $\scorePFP$ (dual order), $\scorePFN$ (dual order), and $\scorePTP$ can be placed on the \tile. 
Also, the standardized negative predictive value $\scoreSNPV$, the negative likelihood ratio $\scoreNLR$ (dual order), the standardized positive predictive value $\scoreSPPV$, the positive likelihood ratio $\scorePLR$, the weighted accuracy $\scoreWeightedAccuracy=\formulaWeightedAccuracy$, the balanced accuracy $\scoreBalancedAccuracy$, the score $\scoreConfusionMatrixDeterminant$, Youden's index $\scoreYoudenJ$, and Cohen's $\scoreCohenKappa$ can be placed on the \tile. 
Some of them have a fixed position, while for others, the position depends on the priors: $\ordering_{\scoreWeightedAccuracy}$, $\ordering_{\scoreBalancedAccuracy}$ and $\ordering_{\scoreYoudenJ}$ sweep the ascending diagonal, while $\ordering_\scoreCohenKappa$ sweeps the median horizontal. The performance ordering $\ordering_{\scoreWeightedAccuracy}$ for $\scoreWeightedAccuracy$ is at $(a,b)=(\aWeightedAccuracy,\bWeightedAccuracy)$, $\ordering_{\scoreBalancedAccuracy}$ is at $(a,b)=(\aBalancedAccuracy,\bBalancedAccuracy)$, and $\ordering_\scoreCohenKappa$ is at $(a,b)=(\aCohen,\bCohen)$.

\mysection{Characterizing scores.}
The \tile can be used to characterize any score, showing the rank correlations between that score and all canonical ranking scores, for a given performance distribution. For example, \cref{fig:correlation-with-probabilistic-ranking-scores} shows the results obtained with the $9$ probabilistic scores that belong to the ranking scores (see \cref{example:probabilistic-ranking-scores}), for a uniform distribution of performances (\ie, a symmetric Dirichlet distribution with all concentration parameters set to one). For this distribution, these rank correlations are either null ($\scoreNPV$ with $\scorePPV$, $\scoreTNR$ with $\scoreTPR$) 
or positive. Such an analysis can be easily performed with any score; the \tile turns out to be a very practical visualization tool to gain intuition about the behavior of the plethora of scores that exist.

\mysection{Visualizing the robustness.}
The importance values $\randVarImportance(\aSample)$ are design choices for competitions. Several recent papers~\cite{MaierHein2018WhyRankings,Nguyen2023HowTrustworthy} have  alerted the scientific community about the necessary robustness: the performance-based rankings should not vary much when parameters are slightly perturbed. \Cref{fig:correlation-with-probabilistic-ranking-scores} makes clear that, for a uniform distribution of performances, the performance orderings do not vary much when the parameters $a$ and $b$ are slightly perturbed.

\begin{figure}
\definecolor{colorEntityA}{rgb}{0.10588235294117647, 0.6196078431372549, 0.4666666666666667}
\definecolor{colorEntityB}{rgb}{0.4588235294117647, 0.4392156862745098, 0.7019607843137254}
\definecolor{colorEntityC}{rgb}{0.9019607843137255, 0.6705882352941176, 0.00784313725490196}
\definecolor{colorEntityD}{rgb}{0.4, 0.4, 0.4}
\newcommand{\legendEntityA}{\textcolor{colorEntityA}{$\CIRCLE$}}
\newcommand{\legendEntityB}{\textcolor{colorEntityB}{$\CIRCLE$}}
\newcommand{\legendEntityC}{\textcolor{colorEntityC}{$\CIRCLE$}}
\newcommand{\legendEntityD}{\textcolor{colorEntityD}{$\CIRCLE$}}
\newcommand{\perfEntityA}{\aPerformance_{-}}
\newcommand{\perfEntityB}{\aPerformance_{1}}
\newcommand{\perfEntityC}{\aPerformance_{2}}
\newcommand{\perfEntityD}{\aPerformance_{+}}
\newcommand{\tableToyA}{
    \begin{minipage}[b][0.3849\linewidth][c]{0.45\linewidth}%
    \begin{center}
    \resizebox{\linewidth}{!}{%
    \begin{tabular}[b]{|c|c|c|c|c|}
    \cline{2-5} \cline{3-5} \cline{4-5} \cline{5-5} 
    \multicolumn{1}{c|}{} & \multicolumn{4}{c|}{performances}\tabularnewline
    \cline{2-5} \cline{3-5} \cline{4-5} \cline{5-5} 
    \multicolumn{1}{c|}{} & $\perfEntityA$ \legendEntityA & $\perfEntityB$ \legendEntityB & $\perfEntityC$ \legendEntityC & $\perfEntityD$ \legendEntityD\tabularnewline
    \hline 
    $\scorePTN$ & 0.80 & 0.56 & 0.40 & 0.00\tabularnewline
    \hline 
    $\scorePFP$ & 0.00 & 0.24 & 0.40 & 0.80\tabularnewline
    \hline 
    $\scorePFN$ & 0.20 & 0.06 & 0.04 & 0.00\tabularnewline
    \hline 
    $\scorePTP$ & 0.00 & 0.14 & 0.16 & 0.20\tabularnewline
    \hline 
    \end{tabular}}
    \par\end{center}%
    \end{minipage}
}
\newcommand{\tableToyB}{
    \begin{minipage}[b][0.3849\linewidth][c]{0.45\linewidth}%
    \begin{center}
    \resizebox{\linewidth}{!}{%
    \begin{tabular}[b]{|c|c|c|c|c|}
    \cline{2-5} \cline{3-5} \cline{4-5} \cline{5-5} 
    \multicolumn{1}{c|}{} & \multicolumn{4}{c|}{performances}\tabularnewline
    \cline{2-5} \cline{3-5} \cline{4-5} \cline{5-5} 
    \multicolumn{1}{c|}{} & $\perfEntityA$ \legendEntityA & $\perfEntityB$ \legendEntityB & $\perfEntityC$ \legendEntityC & $\perfEntityD$ \legendEntityD\tabularnewline
    \hline 
    $\scorePTN$ & 0.50 & 0.35 & 0.25 & 0.00\tabularnewline
    \hline 
    $\scorePFP$ & 0.00 & 0.15 & 0.25 & 0.50\tabularnewline
    \hline 
    $\scorePFN$ & 0.50 & 0.15 & 0.10 & 0.00\tabularnewline
    \hline 
    $\scorePTP$ & 0.00 & 0.35 & 0.40 & 0.50\tabularnewline
    \hline 
    \end{tabular}}
    \par\end{center}%
    \end{minipage}
}
\newcommand{\tableToyC}{
    \begin{minipage}[b][0.3849\linewidth][c]{0.45\linewidth}%
    \begin{center}
    \resizebox{\linewidth}{!}{%
    \begin{tabular}[b]{|c|c|c|c|c|}
    \cline{2-5} \cline{3-5} \cline{4-5} \cline{5-5} 
    \multicolumn{1}{c|}{} & \multicolumn{4}{c|}{performances}\tabularnewline
    \cline{2-5} \cline{3-5} \cline{4-5} \cline{5-5} 
    \multicolumn{1}{c|}{} & $\perfEntityA$ \legendEntityA & $\perfEntityB$ \legendEntityB & $\perfEntityC$ \legendEntityC & $\perfEntityD$ \legendEntityD\tabularnewline
    \hline 
    $\scorePTN$ & 0.20 & 0.14 & 0.10 & 0.00\tabularnewline
    \hline 
    $\scorePFP$ & 0.00 & 0.06 & 0.10 & 0.20\tabularnewline
    \hline 
    $\scorePFN$ & 0.80 & 0.24 & 0.16 & 0.00\tabularnewline
    \hline 
    $\scorePTP$ & 0.00 & 0.56 & 0.64 & 0.80\tabularnewline
    \hline 
    \end{tabular}}
    \par\end{center}%
    \end{minipage}
}
\begin{centering}
\hfill
\includegraphics[width=0.45\linewidth]{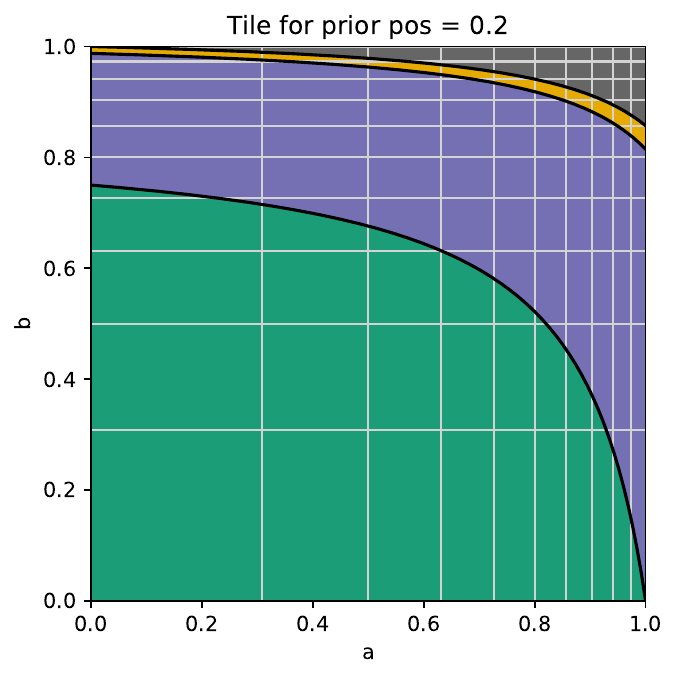}%
\hfill\hfill
\tableToyA
\hfill\end{centering}

\begin{centering}
\hfill
\includegraphics[width=0.45\linewidth]{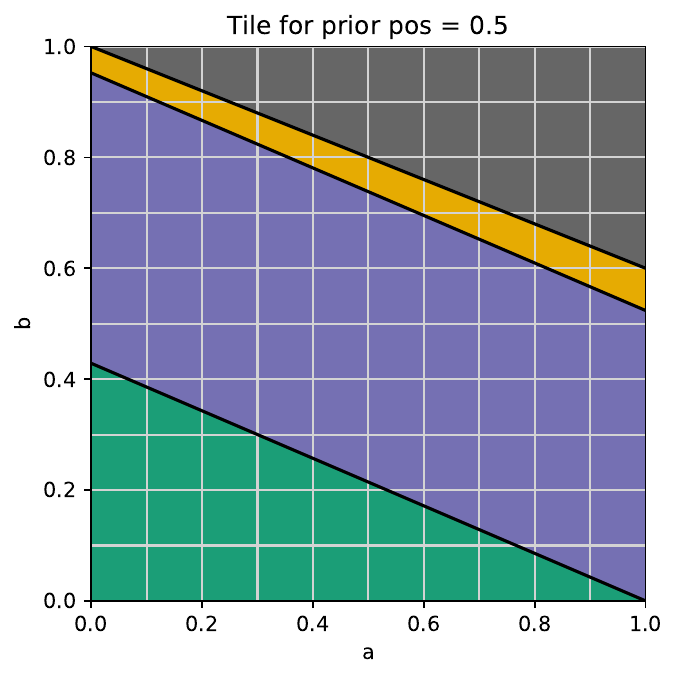}%
\hfill\hfill
\tableToyB
\hfill\end{centering}

\begin{centering}
\hfill
\includegraphics[width=0.45\linewidth]{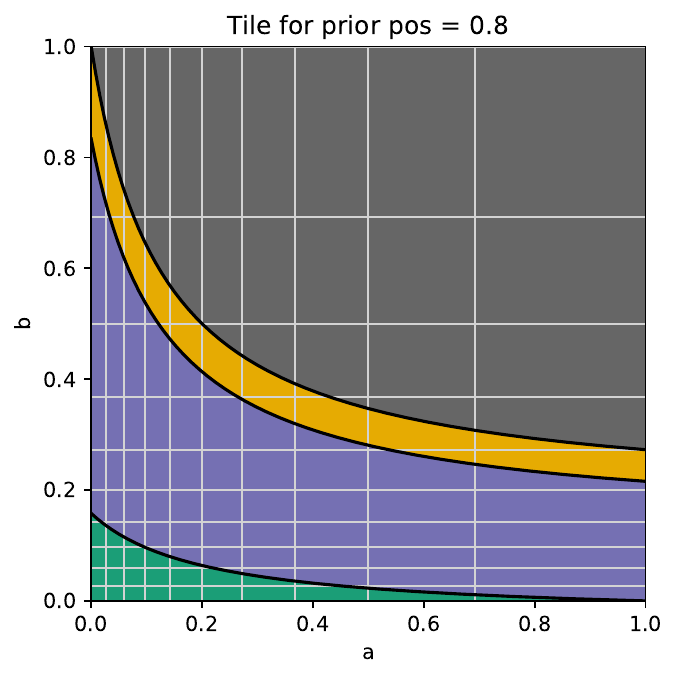}%
\hfill\hfill
\tableToyC
\hfill\end{centering}

\caption{Toy examples showing on the \tile which of the $4$ performances $\perfEntityA$, $\perfEntityB$, $\perfEntityC$, and $\perfEntityD$ is the best. The $3$ examples differ in the class priors (either $0.2$, $0.5$, or $0.8$ for the positive class). In all examples, $\perfEntityA$ (\legendEntityA) is the performance of classifiers predicting always the negative class, $\perfEntityB$ (\legendEntityB) is such that $\scoreTNR(\perfEntityB)=0.7$ and $\scoreTPR(\perfEntityB)=0.7$, $\perfEntityC$ (\legendEntityC) is such that $\scoreTNR(\perfEntityC)=0.5$ and $\scoreTPR(\perfEntityC)=0.8$, and $\perfEntityD$ (\legendEntityD) is the performance of classifiers predicting always the positive class.\label{fig:toy-example}\label{fig:prior-shift}}

\end{figure}

\subsection{Rankings on the \tile}
\label{sec:rankings-on-tile}

For a given set of classifiers to rank, one can use the \tile to show which classifier is ranked first, second, third, \etc, according to $\ordering_{\canonicalRankingScore}$, in $(a,b)$. This is shown in \cref{fig:toy-example} with a toy example. When $\priorneg=\priorpos$, the regions where the different classifiers are ranked first are convex polygons. When $\priorneg\ne\priorpos$, the borders between these regions are curved.

\subsection{More About No-Skill Performances}
\label{sec:no-skill-performances}

\mysection{How can we rank no-skill performances ex aequo?}
The ranking scores allow ranking all no-skill performances (\ie, those for which the groundtruth and predicted classes are independent) ex aequo when some constraints on the compared performances are added. The more common constraint is undoubtedly that the priors are fixed. In this case, the canonical ranking scores that put the no-skill performances on the same footing are located on the curve
$\tileCurvePriors : \priorpos^2 \, a \, b = \priorneg^2 \, (1-a) \, (1-b)$. 
Another interesting constraint is that the rates of predictions are fixed. By symmetry, the canonical ranking scores that put the no-skill performances on the same footing are located on the curve
$\tileCurveRates : \ratepos^2 \, a \, (1-b) = \rateneg^2 \, (1-a) \, b$. 
The $\tileCurvePriors$ and $\tileCurveRates$ curves are depicted in \cref{fig:curves-no-skill}.
\begin{figure}
\hfill{}
\includegraphics[width=0.45\linewidth]{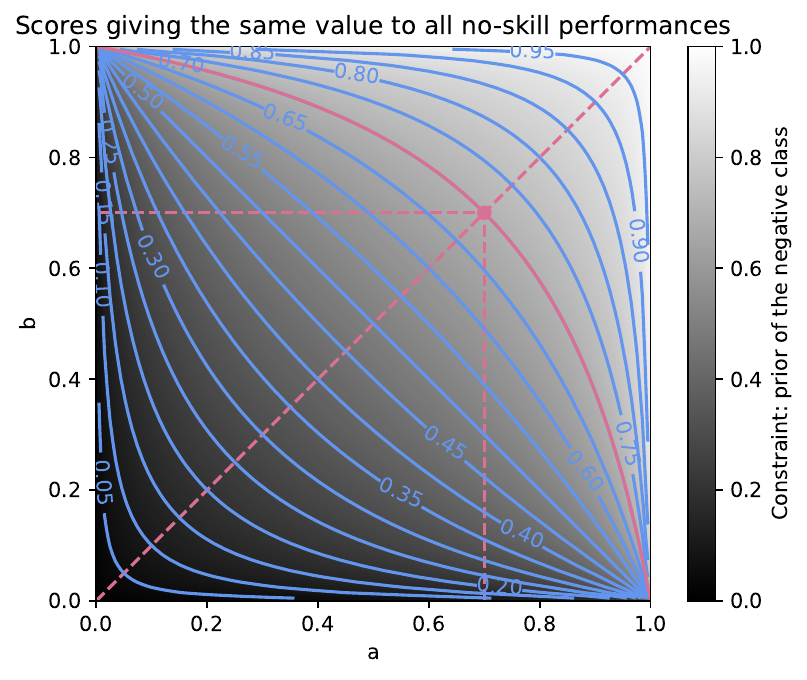}
\hfill{}\hfill{}
\includegraphics[width=0.45\linewidth]{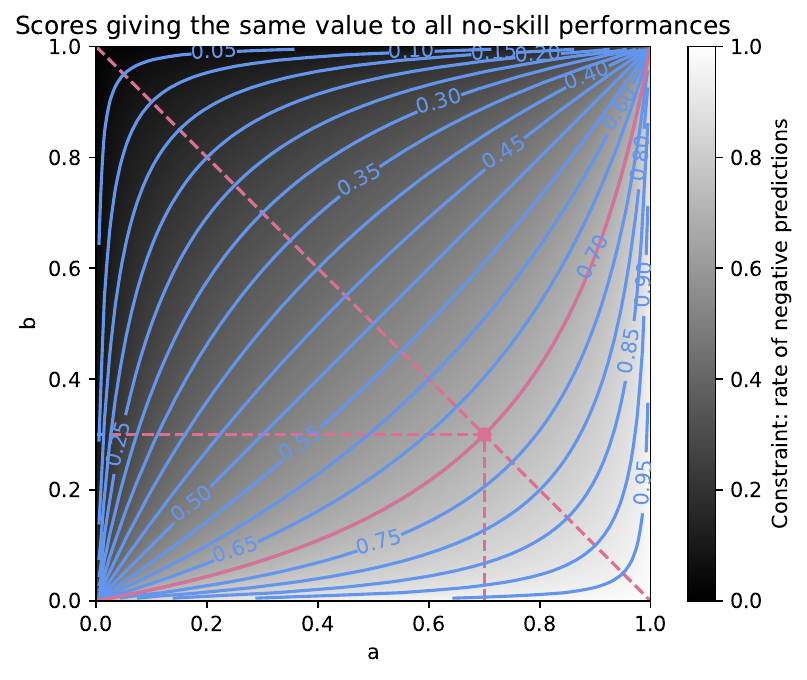}
\hfill{}

\caption{In the \tile, the ranking scores that put all no-skill performances on an equal footing are along a curve $\tileCurvePriors$ between $\scoreNPV$ (upper-left corner) and $\scorePPV$ (lower-right corner) when the class priors $\aPerformance(\randVarGroundtruthClass)$ are fixed, and along a curve $\tileCurveRates$ between $\scoreTNR$ (lower-left corner) and $\scoreTPR$ (upper-right corner) when the rates of predictions $\aPerformance(\randVarPredictedClass)$ are fixed. The pink curves correspond to the constraints $\priorneg=0.7$ (on the left) and $\rateneg=0.7$ (on the right).\label{fig:curves-no-skill}}

\end{figure}

\mysection{A new look at the balanced accuracy $\scoreBalancedAccuracy$ and Cohen's kappa $\scoreCohenKappa$.}
\Cref{fig:correlation-with-other-scores} shows the rank correlations for both $\scoreBalancedAccuracy$ (on the left) and $\scoreCohenKappa$ (on the right). We can see that $\scoreBalancedAccuracy$ is perfectly correlated 
with the ranking scores at the intersection between the curve $\tileCurvePriors$ and the rising diagonal, which is at $(\priorneg,\priorneg)$, and that $\scoreCohenKappa$ is perfectly correlated with the ranking scores at the intersection between the curve $\tileCurvePriors$ and the median horizontal, which is at $(\frac{\priorneg^2}{\priorneg^2+\priorpos^2},\frac{1}{2})$. 
\begin{figure}
\hfill{}
\includegraphics[width=0.45\linewidth]{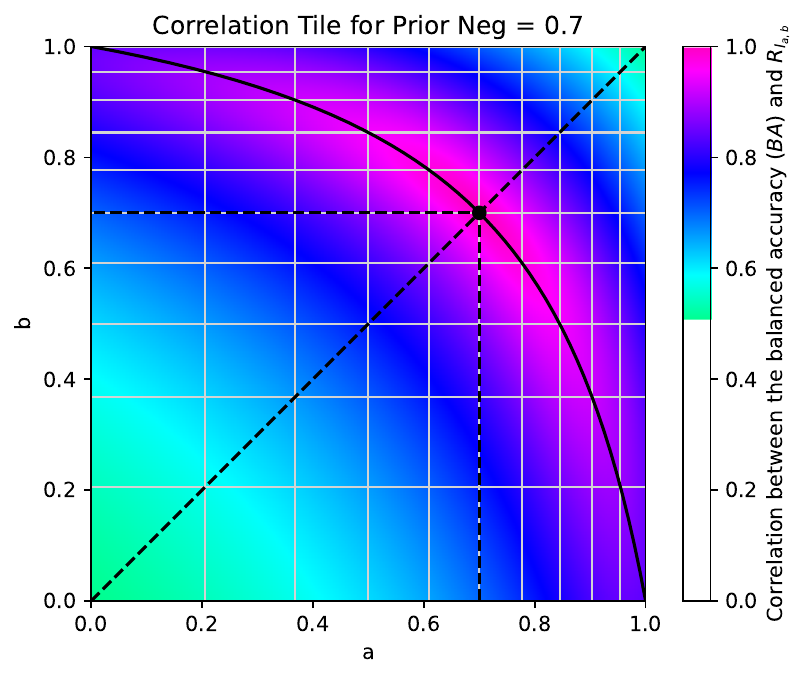}
\hfill{}\hfill{}
\includegraphics[width=0.45\linewidth]{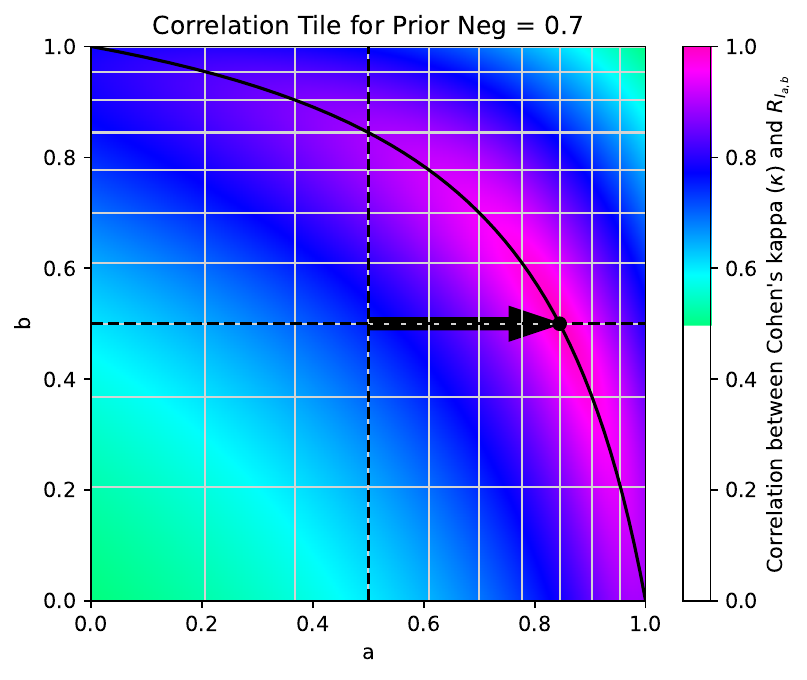}
\hfill{}

\caption{\tiles showing the Kendall rank correlation coefficient $\tau$ for $\scoreBalancedAccuracy$ (left) and $\scoreCohenKappa$ (right), for a uniform distribution of performances $\aPerformance$ such that $\priorneg=0.7$. The correlation values have been estimated based on $10,000$ performances drawn at random.\label{fig:correlation-with-other-scores}}

\end{figure}

\global\long\def\randVarCanonicalImportanceBis{\randVarImportance_{a',b'}}
\global\long\def\canonicalRankingScoreBis{\rankingScore[\randVarCanonicalImportanceBis]}

\mysection{Correction for chance.}
The idea of correcting a score for what can be achieved by chance is common in the literature. Scott~\cite{Scott1955Reliability} and Fleiss~\cite{Fleiss1971Measuring} proposed a correction for accuracy. Cohen~\cite{Cohen1960ACoefficient} proposed another correction for it, that differs in what is considered to be achievable by chance. We noticed that Scott's $\scoreScottPi$ and Fleiss's $\scoreFleissKappa$ do not satisfy the axioms of ranking, even in the case of fixed priors. Cohen's correction, on the other hand, allows nothing more than what we can do with ranking scores, with fixed priors: correcting $\canonicalRankingScore$ in the same way as Cohen~\cite{Cohen1960ACoefficient} did with $\scoreAccuracy$, that is
$
    \frac{
        \canonicalRankingScore-\canonicalRankingScore\circ\opNoSkill
    }{
        1-\canonicalRankingScore\circ\opNoSkill
    }
$ 
leads to a score that is perfectly rank-correlated with $\canonicalRankingScoreBis$, where $a'=\frac{\priorneg^2(1-b)}{\priorneg^2(1-b)+\priorpos^2b}$ and $b'=b$. %Taking $a=b=\nicefrac12$, we recover the result that $\ordering_\scoreCohenKappa$ is located at $(\aCohen,\bCohen)$ on the \tile. 
Geometrically, an entire horizontal line of the \tile is crushed into a single point (the intersection it has with the $\tileCurvePriors$ curve). There is thus an enormous loss of diversity 
when applying Cohen's correction to our scores.

\section{Conclusion}
\label{sec:conclusion}

In this paper, we presented the \emph{\tile}, which organizes ranking scores for two-class classification task, according to the random variable called \emph{importance} capable of considering application-specific preferences. 
The scores organized on this \tile are called the \emph{canonical ranking scores} and include well-known scores, such as the accuracy or $\scoreFBeta$ scores. These canonical ranking scores lead to performance orderings that are different in each point. 
The \tile is a visual tool that can be used in different ways. It can be used to establish correspondences of the \tile with the ROC space and to show how to read the values taken by the canonical ranking scores in any point of the ROC space.
We also showed how to use the \tile to (1)~study the behavior of any score, by depicting rank correlations between a score and all ranking scores, (2)~rank classifiers, using a toy example, (3)~study the influence of priors on the ranking scores, (4)~study properties of no-skill performances, and (5)~clarify what the balanced accuracy and Cohen's kappa are.
In summary, the \tile{} offers a comprehensive visual framework for ranking two-class classifiers, empowering researchers and practitioners to make informed, application-specific decisions, and ultimately driving advancements in the field of machine learning.

\mysection{Acknowledments.}
The work by S. Pi{\'e}rard and A. Halin was supported by the Walloon Region (Service Public de Wallonie Recherche, Belgium) under grant n°2010235 (ARIAC by \href{https://www.digitalwallonia.be/en/}{DIGITALWALLONIA4.AI}). 
A. Deli{\`e}ge is funded by the \href{https://www.frs-fnrs.be}{F.R.S.-FNRS} under project grant T$.0065.22$. A. Cioppa is 
funded by the \href{https://www.frs-fnrs.be}{F.R.S.-FNRS}.

\newpage
\onecolumn
\appendix
\section{Supplementary Material}
\label{sec:supplementary}

%\tableofcontents{}
\section*{Contents}% of the Supplementary Material}
\startcontents[mytoc]
\printcontents[mytoc]{}{0}{}

\clearpage
\subsection{List of symbols}

\subsubsection{Mathematical Symbols}
\begin{itemize}
    \item $\indicatorSymbol_U$: the 0-1 indicator function of subset $U$
    \item $\realNumbers$: the real numbers
    \item $\aRelation$: a relation
    \item $\allConvexCombinations$: the set of convex combinations
    \item $\vee$: the \emph{inclusive disjunction} (\ie, logical or)
    \item $\wedge$: the \emph{conjunction} (\ie, logical and)
    \item $\circ$: the composition of functions, \ie $(g\circ f)(x)=g(f(x))$
    \item $\expectedValueSymbol$: the mathematical expectation
\end{itemize}

\subsubsection{Symbols related to our mathematical framework of \paperA}

We organize these symbols according to the 6 pillars depicted in the graphical abstract of \paperA.

\paragraph{Symbols related to the 1\textsuperscript{st} pillar}
\begin{itemize}
    \item $\sampleSpace$: the sample space (universe)
    \item $\aSample$: a sample (\ie, an element of $\sampleSpace$)
    \item $\eventSpace$: the event space (a $\sigma$-algebra on $\sampleSpace$, \eg $2^\sampleSpace$)
    \item $\anEvent$: an event (\ie, an element of $\eventSpace$)
    \item $\measurableSpace$: the measurable space
    \item $\allPerformances$: all performances on $\measurableSpace$
    \item $\aSetOfPerformances$: a set of performances ($\aSetOfPerformances\subseteq\allPerformances$
    \item $\aPerformance$: a performance (\ie, an element of $\allPerformances$)
\end{itemize}

\paragraph{Symbols related to the 2\textsuperscript{nd} pillar}
\begin{itemize}
%    \item $\relBetterOrEquivalent_{\aScore}$: the inverted ordering induced by the score $\aScore$
    \item $\relWorseOrEquivalent$:  binary relation \emph{worse or equivalent} on $\allPerformances$
    \item $\relBetterOrEquivalent$: binary relation \emph{better or equivalent} on $\allPerformances$
    \item $\relEquivalent$:  binary relation \emph{equivalent} on $\allPerformances$
    \item $\relBetter$:  binary relation \emph{better} on $\allPerformances$
    \item $\relWorse$:  binary relation \emph{worse} on $\allPerformances$
    \item $\relIncomparable$: binary relation \emph{incomparable} on $\allPerformances$
\end{itemize}

\paragraph{Symbols related to the 3\textsuperscript{rd} pillar}
\begin{itemize}
    \item $\randVarSatisfaction$: the random variable \emph{Satisfaction}
\end{itemize}

\paragraph{Symbols related to the 4\textsuperscript{th} pillar}
\begin{itemize}
    \item $\entitiesToRank$: the set of entities to rank
    \item $\anEntity$: an entity, \ie an element of $\entitiesToRank$
    \item $\evaluation$: the performance \emph{evaluation} function
    \item $\achievableByCombinations$: some performances that are for sure achievable
\end{itemize}

\paragraph{Symbols related to the 5\textsuperscript{th} pillar}
\begin{itemize}
    \item $\allScores$: all scores on $\measurableSpace$
    \item $\aScore$: a score
    \item $\domainOfScore$: the domain of the score $\aScore$
    \item $\unconditionalProbabilisticScore{\anEvent}$: the \emph{unconditional probabilistic score} parameterized by the event $\anEvent$
    \item $\conditionalProbabilisticScore{\anEvent_1}{\anEvent_2}$: the \emph{conditional probabilistic score} parameterized by the events $\anEvent_1$ and $\anEvent_2$
\end{itemize}

\paragraph{Symbols related to the 6\textsuperscript{th} pillar}
\begin{itemize}
    \item $\randVarImportance$: the random variable \emph{Importance}
\end{itemize}

\subsubsection{Symbols used for operations on performances}
\begin{itemize}
    \item $\opFilter$: the \emph{filtering} operation, parameterized by a random variable $\randVarImportance$, as defined in \paperA
    \item $\opNoSkill$: the operation that transforms a performance $\aPerformance$ into $\aPerformance'$ such that $\aPerformance'(\randVarGroundtruthClass,\randVarPredictedClass)=\aPerformance(\randVarGroundtruthClass)\aPerformance(\randVarPredictedClass)$
    \item $\opPriorShift$: the \emph{prior/target shift} operation~\cite{Sipka2022TheHitchhikerGuide}
    \item $\opChangePredictedClass$: the operation that changes the predicted class $\randVarPredictedClass$
    \item $\opChangeGroundtruthClass$: the operation that changes the ground-truth class $\randVarGroundtruthClass$
    \item $\opSwapGroundtruthAndPredictedClasses$: the operation that swaps the predicted ($\randVarPredictedClass$)
and ground-truth ($\randVarGroundtruthClass$) classes
    \item $\opSwapClasses$: the operation that swaps the classes $\classNeg$ and $\classPos$.
\end{itemize}

\subsubsection{Symbols used in the performance ordering and performance-based ranking theory}
\begin{itemize}
    \item $\rank$: the \emph{ranking} function, \wrt the set of entities $\entitiesToRank$
    \item $\ordering_{\aScore}$: the ordering induced by the score $\aScore$
    \item $\invertedOrdering_{\aScore}$: the dual (inverted) ordering induced by the score $\aScore$
    \item $\rankingScore$: the \emph{ranking score} parameterized by the importance $\randVarImportance$
    \item $\canonicalRankingScore$: the \emph{canonical ranking score} parameterized by the parameters $a$ and $b$
    \item $a$: the parameter specifying the relative importance given to the incorrect outcomes ($\randVarSatisfaction=0$),\\it corresponds to the horizontal axis of the \tile
    \item $b$: the parameter specifying the relative importance given to the correct outcomes ($\randVarSatisfaction=1$),\\it corresponds to the vertical axis of the \tile
    \item $\tau$: the rank correlation coefficient of Kendall
\end{itemize}

\subsubsection{Symbols used for the particular case of two-class crisp classifications}

\paragraph{Particularization of the mathematical framework}
\begin{itemize}
    \item $\sampleTN$: the sample \emph{true negative}
    \item $\sampleFP$: the sample \emph{false positive}, \aka type I error
    \item $\sampleFN$: the sample \emph{false negative}, \aka type II error
    \item $\sampleTP$: the sample \emph{true positive}
\end{itemize}

\paragraph{Extensions to the mathematical framework}
\begin{itemize}
    \item ROC: the \emph{Receiver Operating Characteristic} space, \ie $\scoreFPR\times\scoreTPR$
    \item PR: the \emph{Precision-Recall} space, \ie $\scoreTPR\times\scorePPV$
    \item $\randVarGroundtruthClass$: the random variable for the ground truth
    \item $\randVarPredictedClass$: the random variable for the prediction
    \item $\allClasses$: the set of classes
    \item $\aClass$: a class (\ie, an element of $\allClasses$)
    \item $\classNeg$: the negative class
    \item $\classPos$: the positive class
    \item $\tileCurvePriors$: the locus, on the \tile, of all the canonical ranking scores that put the no-skill performances on the same footing, for fixed class priors
    \item $\tileCurveRates$: the locus, on the \tile, of all the canonical ranking scores that put the no-skill performances on the same footing, for fixed rates of predictions
\end{itemize}

\paragraph{Some unconditional probabilistic scores}
\begin{itemize}
    \item $\scorePTN$: the probability of the elementary event \emph{true negative}, \aka \emph{rejection rate}
    \item $\scorePFP$: the probability of the elementary event \emph{false positive}
    \item $\scorePFN$: the probability of the elementary event \emph{false negative}
    \item $\scorePTP$: the probability of the elementary event \emph{true positive}, \aka \emph{detection rate}
    \item $\priorneg$: the \emph{prior of the negative class}
    \item $\priorpos$: the \emph{prior of the positive class}, \aka \emph{prevalence}
    \item $\rateneg$: the \emph{rate of negative predictions}
    \item $\ratepos$: the \emph{rate of positive predictions}
    \item $\scoreAccuracy$: the \emph{accuracy}, \aka \emph{matching coefficient}
\end{itemize}

\paragraph{Some conditional probabilistic scores}
\begin{itemize}
    \item $\scoreTNR$: the \emph{True Negative Rate}, \aka \emph{specificity}, \emph{selectivity}, \emph{inverse recall}
    \item $\scoreFPR$: the \emph{False Positive Rate}
    \item $\scoreTPR$: the \emph{True Positive Rate}, \aka \emph{sensitivity}, \emph{recall}
    \item $\scoreFNR$: the \emph{False Negative Rate}
    \item $\scoreNPV$: the \emph{Negative Predictive Value}, \aka \emph{inverse precision}
    \item $\scoreFOR$: the \emph{False Omission Rate}
    \item $\scoreNPV$: the \emph{Positive Predictive Value}, \aka \emph{precision}
    \item $\scoreFDR$: the \emph{False Discovery Rate}
    \item $\scoreJaccardNeg$: Jaccard index for the negative class
    \item $\scoreJaccardPos$: Jaccard index for the positive class, \aka \emph{Tanimoto coefficient}, \emph{similarity}, \emph{intersection over union}, \emph{critical success index}~\cite{Hogan2010Equitability}, \emph{G-measure}~\cite{Flach2003TheGeometry}
\end{itemize}

\paragraph{Some other scores}
\begin{itemize}
    \item $\scoreBennettS$: Bennett, Alpert and Goldstein’s $S$
    \item $\scoreFBeta$: the F-scores
    \item $\scoreFBeta[1]$: the F-one score, \aka \emph{Dice-S{\o}rensen coefficient}
    \item $\scoreSNPV$: the score \emph{Standardized Negative Predictive Value}~\cite{Heston2011Standardizing}
    \item $\scoreSPPV$: the score \emph{Standardized Positive Predictive Value}~\cite{Heston2011Standardizing}
    \item $\scoreNLR$: the score \emph{Negative Likelihood Ratio}
    \item $\scorePLR$: the score \emph{Positive Likelihood Ratio}
    \item $\scoreCohenKappa$: Cohen's kappa statistic, \aka \emph{Heidke Skill Score}~\cite{Wilks2020Statistical}
    \item $\scoreScottPi$: Scott's pi statistic
    \item $\scoreFleissKappa$: Fleiss's kappa statistic
    \item $\scoreBiasIndex$: the \emph{Bias Index}, as defined in~\cite{Byrt1993Bias}
    \item $\scoreWeightedAccuracy$: the \emph{Weighted Accuracy}
    \item $\scoreBalancedAccuracy$: the \emph{Balanced Accuracy}
    \item $\scoreYoudenJ$: Youden's index~\cite{Youden1950Index}, \aka Youden's J statistic, \emph{informedness}, \emph{Peirce Skill Score}~\cite{Wilks2020Statistical}
    \item $\scoreConfusionMatrixDeterminant$: the determinant of the (normalized) confusion matrix or contingency matrix
    \item $\scoreACP$: the \emph{Average Conditional Probability}, \ie the arithmetic mean of the \tile's four corners~\cite{Burset1996Evaluation}
    \item $\scorePFour$: the harmonic mean of the \tile's four corners~\cite{Sitarz2023Extending}
    \item $\scoreVUT$: the score \emph{Volume Under \tile}, \ie the arithmetic mean of all canonical scores (see \cref{sec:score-under-tile})
\end{itemize}

\clearpage
\subsection{Supplementary material about \cref{sec:canonical-ranking-scores-on-tile}}

\subsubsection{Operations on performances}
\label{sec:operations-on-performances}

We present here the proofs for the geometric properties of the \tile that are related to some operations that can be applied to performances. A summary is provided in \cref{tbl:summary-operations}.

\global\long\def\aOld{a_{\mathrm{origin}}}%
\global\long\def\aNew{a_{\mathrm{adapted}}}%
\global\long\def\bOld{b_{\mathrm{origin}}}%
\global\long\def\bNew{b_{\mathrm{adapted}}}%

\begin{table}
\caption{Summary of the effects on the \tile of 5 operations
on performances.\label{tbl:summary-operations}}

\centering{}%
\begin{tabular}{|cc|c|c|c|}
\hline 
\multicolumn{2}{|c|}{operation} & \multicolumn{2}{c|}{the ordering that was at $(\aOld,\bOld)$} & \multirow{2}{*}{note}\tabularnewline
\multicolumn{2}{|c|}{on performances} & \multicolumn{2}{c|}{is moved at $(\aNew,\bNew)$} & \tabularnewline
\hline 
\hline 
$\opChangePredictedClass$ & (see \cref{lemma:opChangePredictedClass}) & $\aNew=\bOld$ & $\bNew=\aOld$ & the preorder is inverted (dual)\tabularnewline
\hline 
$\opChangeGroundtruthClass$ & (see \cref{lemma:opChangeGroundtruthClass}) & $\aNew=1-\bOld$ & $\bNew=1-\aOld$ & the preorder is inverted (dual)\tabularnewline
\hline 
$\opSwapGroundtruthAndPredictedClasses$ & (see \cref{lemma:opSwapGroundtruthAndPredictedClasses}) & $\aNew=\aOld$ & $\bNew=1-\bOld$ & the preorder is unchanged\tabularnewline
\hline 
$\opSwapClasses$ & (see \cref{lemma:opSwapClasses}) & $\aNew=1-\aOld$ & $\bNew=1-\bOld$ & the preorder is unchanged\tabularnewline
\hline 
$\opPriorShift$ & (see \cref{lemma:opPriorShift}) & $\aNew=f^{-1}(\aOld)$ & $\bNew=f^{-1}(\bOld)$ & the preorder is unchanged\tabularnewline
\hline 
\end{tabular}
\end{table}

\begin{lemma}
\label{lemma:opChangePredictedClass}
Let $\opChangePredictedClass:\allPerformances\rightarrow\allPerformances$
be the operation that changes the predicted class $\randVarPredictedClass$.
We have $\canonicalRankingScore\circ\opChangePredictedClass=1-\rankingScore[\randVarImportance_{b,a}]$.
\end{lemma}

\begin{proof}
Let $\aPerformance$ be a two-class classification performance and
$\aPerformance'=\opChangePredictedClass(\aPerformance)$. If $\aPerformance'\in\domainOfScore[{\canonicalRankingScore}]$,
then $\aPerformance\in\domainOfScore[{\rankingScore[\randVarImportance_{b,a}]}]$
and 
\begin{align*}
\canonicalRankingScore(\aPerformance') & =\frac{(1-a)\aPerformance'(\eventTN)+a\aPerformance'(\eventTP)}{(1-a)\aPerformance'(\eventTN)+(1-b)\aPerformance'(\eventFP)+b\aPerformance'(\eventFN)+a\aPerformance'(\eventTP)}\\
 & =\frac{(1-a)\aPerformance(\eventFP)+a\aPerformance(\eventFN)}{(1-a)\aPerformance(\eventFP)+(1-b)\aPerformance(\eventTN)+b\aPerformance(\eventTP)+a\aPerformance(\eventFN)}\\
 & =1-\frac{(1-b)\aPerformance(\eventTN)+b\aPerformance(\eventTP)}{(1-b)\aPerformance(\eventTN)+(1-a)\aPerformance(\eventFP)+a\aPerformance(\eventFN)+b\aPerformance(\eventTP)}\\
 & =1-\rankingScore[\randVarImportance_{b,a}](\aPerformance)
\end{align*}
\end{proof}

\begin{lemma}
\label{lemma:opChangeGroundtruthClass}
Let $\opChangeGroundtruthClass:\allPerformances\rightarrow\allPerformances$
be the operation that changes the ground-truth class $\randVarGroundtruthClass$.
We have $\canonicalRankingScore\circ\opChangeGroundtruthClass=1-\rankingScore[\randVarImportance_{1-b,1-a}]$.
\end{lemma}

\begin{proof}
Let $\aPerformance$ be a two-class classification performance and
$\aPerformance'=\opChangeGroundtruthClass(\aPerformance)$. If $\aPerformance'\in\domainOfScore[{\canonicalRankingScore}]$,
then $\aPerformance\in\domainOfScore[{\rankingScore[\randVarImportance_{1-b,1-a}]}]$
and 
\begin{align*}
\canonicalRankingScore(\aPerformance') & =\frac{(1-a)\aPerformance'(\eventTN)+a\aPerformance'(\eventTP)}{(1-a)\aPerformance'(\eventTN)+(1-b)\aPerformance'(\eventFP)+b\aPerformance'(\eventFN)+a\aPerformance'(\eventTP)}\\
 & =\frac{(1-a)\aPerformance(\eventFN)+a\aPerformance(\eventFP)}{(1-a)\aPerformance(\eventFN)+(1-b)\aPerformance(\eventTP)+b\aPerformance(\eventTN)+a\aPerformance(\eventFP)}\\
 & =1-\frac{b\aPerformance(\eventTN)+(1-b)\aPerformance(\eventTP)}{b\aPerformance(\eventTN)+a\aPerformance(\eventFP)+(1-a)\aPerformance(\eventFN)+(1-b)\aPerformance(\eventTP)}\\
 & =1-\rankingScore[\randVarImportance_{1-b,1-a}](\aPerformance)
\end{align*}
\end{proof}

\begin{lemma}
\label{lemma:opSwapGroundtruthAndPredictedClasses}
Let $\opSwapGroundtruthAndPredictedClasses:\allPerformances\rightarrow\allPerformances$
be the operation that swaps the predicted ($\randVarPredictedClass$)
and ground-truth ($\randVarGroundtruthClass$) classes. We have $\canonicalRankingScore\circ\opSwapGroundtruthAndPredictedClasses=\rankingScore[\randVarImportance_{a,1-b}]$.
\end{lemma}

\begin{proof}
Let $\aPerformance$ be a two-class classification performance and
$\aPerformance'=\opSwapGroundtruthAndPredictedClasses(\aPerformance)$.
If $\aPerformance'\in\domainOfScore[{\canonicalRankingScore}]$,
then $\aPerformance\in\domainOfScore[{\rankingScore[\randVarImportance_{a,1-b}]}]$
and 
\begin{align*}
\canonicalRankingScore(\aPerformance') & =\frac{(1-a)\aPerformance'(\eventTN)+a\aPerformance'(\eventTP)}{(1-a)\aPerformance'(\eventTN)+(1-b)\aPerformance'(\eventFP)+b\aPerformance'(\eventFN)+a\aPerformance'(\eventTP)}\\
 & =\frac{(1-a)\aPerformance(\eventTN)+a\aPerformance(\eventTP)}{(1-a)\aPerformance(\eventTN)+(1-b)\aPerformance(\eventFN)+b\aPerformance(\eventFP)+a\aPerformance(\eventTP)}\\
 & =\frac{(1-a)\aPerformance(\eventTN)+a\aPerformance(\eventTP)}{(1-a)\aPerformance(\eventTN)+b\aPerformance(\eventFP)+(1-b)\aPerformance(\eventFN)+a\aPerformance(\eventTP)}\\
 & =\rankingScore[\randVarImportance_{a,1-b}](\aPerformance)
\end{align*}
\end{proof}

\begin{lemma}
\label{lemma:opSwapClasses}
Let $\opSwapClasses:\allPerformances\rightarrow\allPerformances$
be the operation that swaps the classes $\classNeg$ and $\classPos$.
We have $\canonicalRankingScore\circ\opSwapClasses=\rankingScore[\randVarImportance_{1-a,1-b}]$.
\end{lemma}

\begin{proof}
Let $\aPerformance$ be a two-class classification performance and
$\aPerformance'=\opSwapClasses(\aPerformance)$. If $\aPerformance'\in\domainOfScore[{\canonicalRankingScore}]$,
then $\aPerformance\in\domainOfScore[{\rankingScore[\randVarImportance_{1-a,1-b}]}]$
and 
\begin{align*}
\canonicalRankingScore(\aPerformance') & =\frac{(1-a)\aPerformance'(\eventTN)+a\aPerformance'(\eventTP)}{(1-a)\aPerformance'(\eventTN)+(1-b)\aPerformance'(\eventFP)+b\aPerformance'(\eventFN)+a\aPerformance'(\eventTP)}\\
 & =\frac{(1-a)\aPerformance(\eventTP)+a\aPerformance(\eventTN)}{(1-a)\aPerformance(\eventTP)+(1-b)\aPerformance(\eventFN)+b\aPerformance(\eventFP)+a\aPerformance(\eventTN)}\\
 & =\frac{a\aPerformance(\eventTN)+(1-a)\aPerformance(\eventTP)}{a\aPerformance(\eventTN)+b\aPerformance(\eventFP)+(1-b)\aPerformance(\eventFN)+(1-a)\aPerformance(\eventTP)}\\
 & =\rankingScore[\randVarImportance_{1-a,1-b}](\aPerformance)
\end{align*}
\end{proof}

\begin{lemma}
\label{lemma:opPriorShift}
Let $\opPriorShift$ be the operation that applies a prior/target shift on the distribution $\aPerformance(\randVarGroundtruthClass)$,
transforming the priors $(\priorneg,\priorpos)$ into $(\priorneg',\priorpos')$. The
ordering induced by $\canonicalRankingScore\circ\opPriorShift$
is the same as the ordering induced by $\rankingScore[\randVarImportance_{f(a),f(b)}]$,
where $f$ is the function
\[
f:x\mapsto f(x)=\frac{x\frac{\priorpos'}{\priorpos}}{(1-x)\frac{\priorneg'}{\priorneg}+x\frac{\priorpos'}{\priorpos}}\,.
\]
\end{lemma}

\begin{proof}
Let $\aPerformance$ be a two-class classification performance and
$\aPerformance'=\opPriorShift(\aPerformance)$. We have:
\begin{align*}
\canonicalRankingScore(\aPerformance') & =\frac{(1-a)\aPerformance'(\eventTN)+a\aPerformance'(\eventTP)}{(1-a)\aPerformance'(\eventTN)+(1-b)\aPerformance'(\eventFP)+b\aPerformance'(\eventFN)+a\aPerformance'(\eventTP)}\\
 & =\frac{(1-a)\aPerformance(\eventTN)\frac{\priorneg'}{\priorneg}+a\aPerformance(\eventTP)\frac{\priorpos'}{\priorpos}}{(1-a)\aPerformance(\eventTN)\frac{\priorneg'}{\priorneg}+(1-b)\aPerformance(\eventFP)\frac{\priorneg'}{\priorneg}+b\aPerformance(\eventFN)\frac{\priorpos'}{\priorpos}+a\aPerformance(\eventTP)\frac{\priorpos'}{\priorpos}}\\
 & =\frac{\randVarImportance'(\sampleTN)\aPerformance(\eventTN)+\randVarImportance'(tp)\aPerformance(\eventTP)}{\randVarImportance'(\sampleTN)\aPerformance(\eventTN)+\randVarImportance'(\sampleFP)\aPerformance(\eventFP)+\randVarImportance'(\sampleFN)\aPerformance(\eventFN)+\randVarImportance'(tp)\aPerformance(\eventTP)}\\
 & =\rankingScore[I'](\aPerformance)
\end{align*}
with
\[
\begin{cases}
\randVarImportance'(\sampleTN)=g(a)\left(1-f(a)\right)\\
\randVarImportance'(\sampleFP)=g(b)\left(1-f(b)\right)\\
\randVarImportance'(\sampleFN)=g(b)f(b)\\
\randVarImportance'(tp)=g(a)f(a)
\end{cases}
\]
and
\[
g:x\mapsto g(x)=(1-x)\frac{\priorneg'}{\priorneg}+x\frac{\priorpos'}{\priorpos}\,.
\]
The random variable $\randVarImportance'$ can be rewritten as $\randVarImportance'=(\indicatorSymbol_{\randVarSatisfaction=0}g(b)+\indicatorSymbol_{\randVarSatisfaction=1}g(a))\randVarImportance_{f(a),f(b)}$.
Using Property~\ref{prop:scale-invariance-per-satisfaction}, we
conclude that 
\begin{align*}
 & \frac{\partial\rankingScore[I']}{\partial\rankingScore[\randVarImportance_{f(a),f(b)}]}>0\\
\Rightarrow & \frac{\partial\left(\canonicalRankingScore\circ\opPriorShift\right)}{\partial\rankingScore[\randVarImportance_{f(a),f(b)}]}>0\\
\Rightarrow & \ordering_{\canonicalRankingScore\circ\opPriorShift}=\ordering_{\rankingScore[\randVarImportance_{f(a),f(b)}]}\,.
\end{align*}
\end{proof}

\subsubsection{When the priors are fixed: superimposing a grid on the \tile}

We found it very useful, when priors are fixed, to represent the displacement that would have occurred if we had started with uniform priors and applied the target/prior shift. This can be done thanks to \cref{lemma:opPriorShift}. In practice, we do this by superimposing a grid on the \tile, as done in \cref{fig:toy-example} and \cref{fig:correlation-with-other-scores}. We provide more information on the subject in \cref{fig:tile-prior-grid}. Vertical lines are drawn at
\begin{equation}
    a=f^{-1}(x)=\frac{x\frac{0.5}{\priorpos}}{(1-x)\frac{0.5}{\priorneg}+x\frac{0.5}{\priorpos}} \textrm{ for } x=0.0,0.1,0.2,\ldots,0.9,1.0
    \,,
\end{equation}
and horizontal lines are drawn at
\begin{equation}
    b=f^{-1}(y)=\frac{y\frac{0.5}{\priorpos}}{(1-y)\frac{0.5}{\priorneg}+y\frac{0.5}{\priorpos}} \textrm{ for } y=0.0,0.1,0.2,\ldots,0.9,1.0
    \,.
\end{equation}

\begin{figure}
\begin{centering}
\subfloat[Shift from $(\priorneg,\priorpos)=(0.5,0.5)$ to $(0.9,0.1)$.]{\begin{centering}
\includegraphics[width=0.17\linewidth]{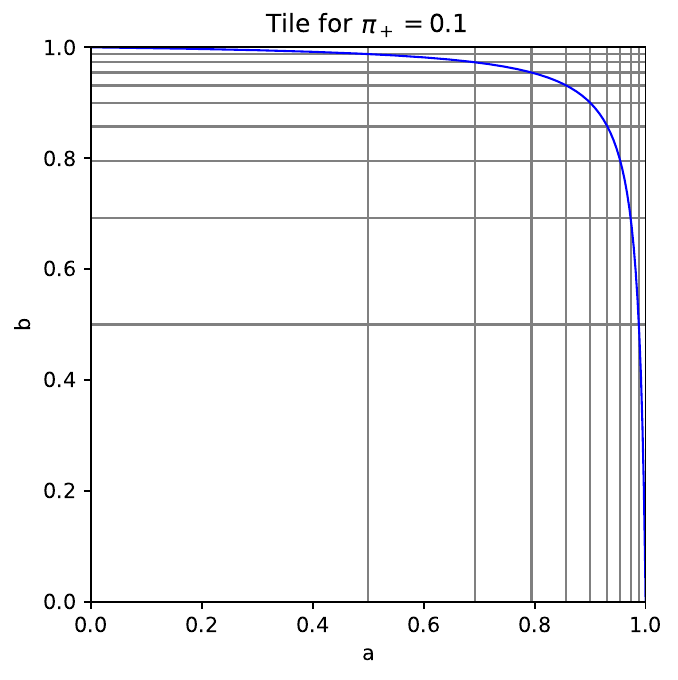}
\par\end{centering}
}
\hfill{}
\subfloat[Shift from $(\priorneg,\priorpos)=(0.5,0.5)$ to $(0.7,0.3)$.]{\begin{centering}
\includegraphics[width=0.17\linewidth]{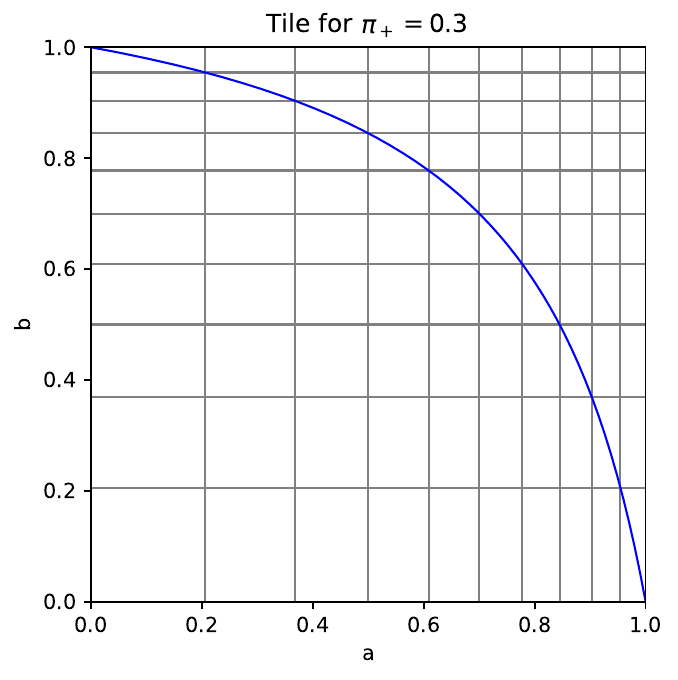}
\par\end{centering}
}
\hfill{}
\subfloat[Reference.]{\begin{centering}
\includegraphics[width=0.17\linewidth]{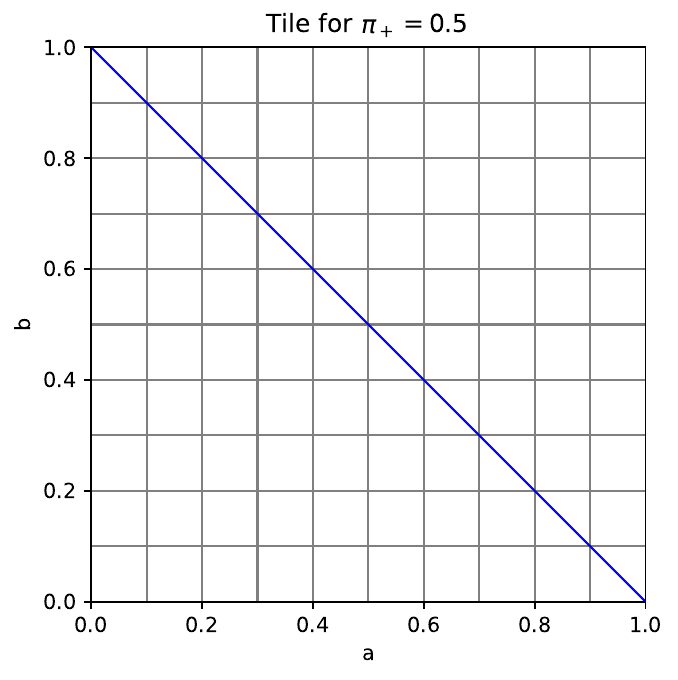}
\par\end{centering}
}
\hfill{}
\subfloat[Shift from $(\priorneg,\priorpos)=(0.5,0.5)$ to $(0.3,0.7)$.]{\begin{centering}
\includegraphics[width=0.17\linewidth]{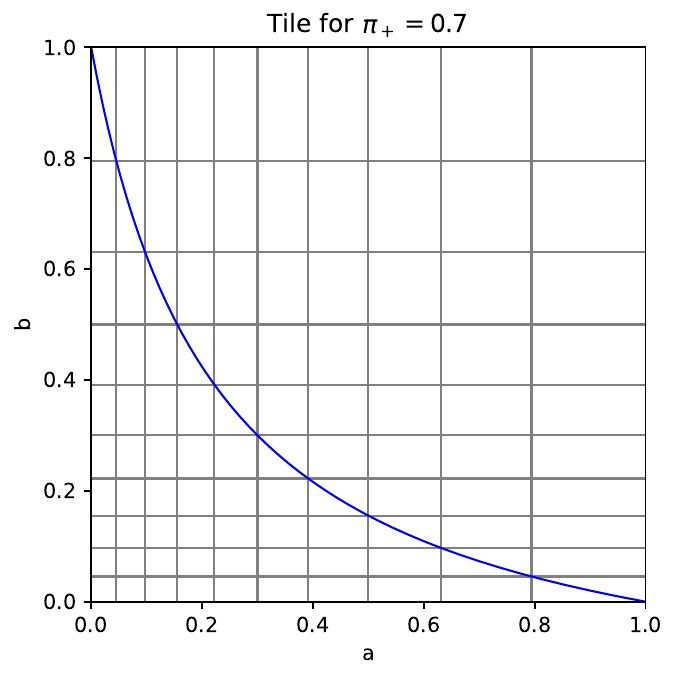}
\par\end{centering}
}
\hfill{}
\subfloat[Shift from $(\priorneg,\priorpos)=(0.5,0.5)$ to $(0.1,0.9)$.]{\begin{centering}
\includegraphics[width=0.17\linewidth]{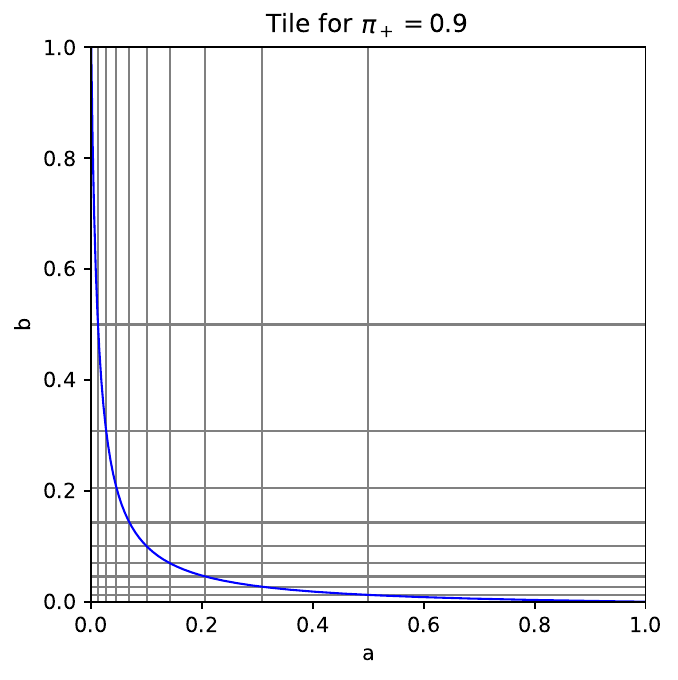}
\par\end{centering}
}
\par\end{centering}
\caption{Visualization of how the \tile deforms with a shift for the distribution $\aPerformance(\randVarGroundtruthClass)$. The central figure shows an initial \tile, in which we have arbitrarily drawn a regular grid. The other figures show the result of \cref{lemma:opPriorShift}, representing the way the \tile deforms. We've also drawn the descending diagonal on our reference \tile. Its deformation gives birth to the family of curves denoted by $\tileCurvePriors$ in our paper.\label{fig:tile-prior-grid}}

\end{figure}

\clearpage
\subsection{Supplementary material about \cref{sec:performance-orderings-on-tile}}

\subsubsection{Placing performance orderings induced by probabilistic scores on the \tile}

The following lemma explains that some probabilistic scores are ranking scores, and specifies the importance values for them. Based on this, placing the performance orderings induced by probabilistic scores is straightforward, as the performance ordering induced by the ranking score $\rankingScore$ is located at $(a,b)=(\frac{\randVarImportance(\sampleTP)}{\randVarImportance(\sampleTN)+\randVarImportance(\sampleTP)},\frac{\randVarImportance(\sampleFN)}{\randVarImportance(\sampleFP)+\randVarImportance(\sampleFN)})$.

\begin{lemma}\label{lemma:proba-scores-ranking-scores}
Let $\randVarSatisfaction$ be a binary-valued satisfaction. A probabilistic
score $\conditionalProbabilisticScore{\anEvent_{1}}{\anEvent_{2}}$, with $\ensuremath{\emptyset\subsetneq\anEvent_{1}\subsetneq\anEvent_{2}\subseteq\sampleSpace}$,
is a ranking score $\rankingScore$ if $\randVarImportance=k\indicatorSymbol_{\anEvent_{2}}$,
$k>0$, and $\anEvent_{1}=\anEvent_{2}\cap\anEvent_{\randVarSatisfaction}$,
with $\anEvent_{\randVarSatisfaction}=\left\{ \aSample\in\sampleSpace:\randVarSatisfaction(\aSample)=1\right\} $.
\end{lemma}

\begin{proof}
Let us first check the equality of the domains. We have 
\[
\domainOfScore[{\rankingScore}]=\left\{ \aPerformance\in\allPerformances:\expectedValueSymbol_{\aPerformance}[\randVarImportance]\ne0\right\} 
\]
 and 
\[
\domainOfScore[\conditionalProbabilisticScore{\anEvent_{1}}{\anEvent_{2}}]=\left\{ \aPerformance\in\allPerformances:\aPerformance(\anEvent_{2})\ne0\right\} \,.
\]
Using the ``fundamental bridge'', we have 
\[
\aPerformance(\anEvent_{2})\ne0\Leftrightarrow\expectedValueSymbol_{\aPerformance}[\indicatorSymbol_{\anEvent_{2}}]\ne0\Leftrightarrow\expectedValueSymbol_{\aPerformance}[\randVarImportance]\ne0\,,
\]
and thus
\[
\domainOfScore[{\rankingScore}]=\domainOfScore[\conditionalProbabilisticScore{\anEvent_{1}}{\anEvent_{2}}]\,.
\]
We now check the equality of the values taken by both scores. We have,
for all $\aPerformance\in\domainOfScore[\conditionalProbabilisticScore{\anEvent_{1}}{\anEvent_{2}}]$,
\[
\conditionalProbabilisticScore{\anEvent_{1}}{\anEvent_{2}}(\aPerformance)=\aPerformance(\anEvent_{1}\vert\anEvent_{2})=\frac{\aPerformance(\anEvent_{1}\cap\anEvent_{2})}{\aPerformance(\anEvent_{2})}=\frac{\aPerformance(\anEvent_{1})}{\aPerformance(\anEvent_{2})}\,.
\]
Using again the ``fundamental bridge'', 
\begin{equation*}
\conditionalProbabilisticScore{\anEvent_{1}}{\anEvent_{2}}(\aPerformance)=\frac{\expectedValueSymbol_{\aPerformance}[\indicatorSymbol_{\anEvent_{1}}]}{\expectedValueSymbol_{\aPerformance}[\indicatorSymbol_{\anEvent_{2}}]}=\frac{\expectedValueSymbol_{\aPerformance}[\indicatorSymbol_{\anEvent_{2}\cap\anEvent_{\randVarSatisfaction}}]}{\expectedValueSymbol_{\aPerformance}[\indicatorSymbol_{\anEvent_{2}}]}=\frac{\expectedValueSymbol_{\aPerformance}[\indicatorSymbol_{\anEvent_{2}}\randVarSatisfaction]}{\expectedValueSymbol_{\aPerformance}[\indicatorSymbol_{\anEvent_{2}}]}
=\frac{\expectedValueSymbol_{\aPerformance}[\frac{\randVarImportance}{k}S]}{\expectedValueSymbol_{\aPerformance}[\frac{\randVarImportance}{k}]}=\frac{\expectedValueSymbol_{\aPerformance}[\randVarImportance S]}{\expectedValueSymbol_{\aPerformance}[\randVarImportance]}=\rankingScore(\aPerformance)\,.
\end{equation*}
\end{proof}

Thanks to Lemma~\ref{lemma:proba-scores-ranking-scores}, we can place on the \tile the probabilistic scores that are ranking scores. In the particular case of two-class classification, there exist 9
pairs $(\anEvent_{1},\anEvent_{2})$ such that $\ensuremath{\emptyset\subsetneq\anEvent_{1}\subsetneq\anEvent_{2}\subseteq\sampleSpace}$
and $\anEvent_{1}=\anEvent_{2}\cap\anEvent_{\randVarSatisfaction}$.
There are thus 9 ranking scores that are also probabilistic scores. Among them, 5 are canonical and can be placed directly on the \tile. For the 4 others, the induced performance ordering can be placed on the \tile. A summary is provided in Table~\ref{tbl:some-ranking-scores}.

\begin{table}
\centering
\caption{The nine ranking scores for two-class classifications that belong to the family of probabilistic scores.\label{tbl:some-ranking-scores}}

\resizebox{1\linewidth}{!}{
%\begin{centering}
\begin{tabular}{|l|l|c|c|c|c|c|c|}
\hline
Ranking score $\rankingScore$ & Probabilistic writing & $\randVarImportance(\sampleTN)$ & $\randVarImportance(\sampleFP)$ & $\randVarImportance(\sampleFN)$ & $\randVarImportance(\sampleTP)$ & Canonical? & location of $\ordering_{\rankingScore}$ on the \tile \tabularnewline
\hline
\hline
Negative Predictive Value $\scoreNPV$ & $\aPerformance(\randVarSatisfaction=1\mid\randVarPredictedClass=\classNeg)$ & $1$ & $0$ & $1$ & $0$ & yes & $(a,b)=(0,1)$ \tabularnewline
\hline 
$\conditionalProbabilisticScore{\{\sampleTN,\sampleTP\}}{\{\sampleTN,\sampleFN,\sampleTP\}}$ & $\aPerformance(\randVarSatisfaction=1\mid\randVarGroundtruthClass=\classPos\vee\randVarPredictedClass=\classNeg)$ & $1$ & $0$ & $1$ & $1$ & no & $(a,b)=(\frac12,1)$ \tabularnewline
\hline 
True Positive Rate $\scoreTPR$ & $\aPerformance(\randVarSatisfaction=1\mid\randVarGroundtruthClass=\classPos)$ & $0$ & $0$ & $1$ & $1$ & yes & $(a,b)=(1,1)$ \tabularnewline
\hline 
Jaccard's index for the negative class $\scoreJaccardNeg$ & $\aPerformance(\randVarSatisfaction=1\mid\randVarGroundtruthClass=\classNeg\vee\randVarPredictedClass=\classNeg)$ & $1$ & $1$ & $1$ & $0$ & no & $(a,b)=(0,\frac12)$ \tabularnewline
\hline 
Accuracy $\scoreAccuracy$ & $\aPerformance(\randVarSatisfaction=1)$ & $\nicefrac12$ & $\nicefrac12$ & $\nicefrac12$ & $\nicefrac12$ & yes & $(a,b)=(\frac12,\frac12)$ \tabularnewline
\hline 
Jaccard's index for the positive class $\scoreJaccardPos$ & $\aPerformance(\randVarSatisfaction=1\mid\randVarGroundtruthClass=\classPos\vee\randVarPredictedClass=\classPos)$ & $0$ & $1$ & $1$ & $1$ & no & $(a,b)=(1,\frac12)$ \tabularnewline
\hline 
True Negative Rate $\scoreTNR$ & $\aPerformance(\randVarSatisfaction=1\mid\randVarGroundtruthClass=\classNeg)$ & $1$ & $1$ & $0$ & $0$ & yes & $(a,b)=(0,0)$ \tabularnewline
\hline 
$\conditionalProbabilisticScore{\{\sampleTN,\sampleTP\}}{\{\sampleTN,\sampleFP,\sampleTP\}}$ & $\aPerformance(\randVarSatisfaction=1\mid\randVarGroundtruthClass=\classNeg\vee\randVarPredictedClass=\classPos)$ & $1$ & $1$ & $0$ & $1$ & no & $(a,b)=(\frac12,0)$ \tabularnewline
\hline 
Positive Predictive Value $\scorePPV$ & $\aPerformance(\randVarSatisfaction=1\mid\randVarPredictedClass=\classPos)$ & $0$ & $1$ & $0$ & $1$ & yes & $(a,b)=(1,0)$ \tabularnewline
\hline 
\end{tabular}
%\end{centering}
}
\end{table}

%%%%%%%%%%%%%%%%%%%%%%%%%%%%%%%%%%%%%%%%%%%%%%%%%%%%%%%%%%%%%%%%%
% Placer les F-scores sur la tuile
%%%%%%%%%%%%%%%%%%%%%%%%%%%%%%%%%%%%%%%%%%%%%%%%%%%%%%%%%%%%%%%%%

\subsubsection{Placing performance orderings induced by the scores $\scoreFBeta$ on the \tile}

\begin{lemma}[F-scores]
In two-class classification, all F-scores are canonical
ranking scores: $\scoreFBeta=\rankingScore$ with $\randVarImportance(\sampleTN)=0$,
$\randVarImportance(\sampleFP)=\frac{1}{1+\beta^{2}}$, $\randVarImportance(\sampleFN)=\frac{\beta^{2}}{1+\beta^{2}}$,
and $\randVarImportance(\sampleTP)=1$.
\end{lemma}

\global\long\def\probaTN{TN}%
\global\long\def\probaFP{FP}%
\global\long\def\probaFN{FN}%
\global\long\def\probaTP{TP}%

\begin{proof}
For the sake of concision, let us pose $\probaTN=\aPerformance(\eventTN)$,
$\probaFP=\aPerformance(\eventFP)$, $\probaFN=\aPerformance(\eventFN)$,
and $\probaTP=\aPerformance(\eventTP)$.

We first check that $\scoreFBeta$ and $\rankingScore$ have the same
domain.
\begin{align*}
\domainOfScore[{\scoreFBeta}] & =\left\{ \aPerformance\in\allPerformances:(1+\beta^{2})\probaTP+\beta^{2}\probaFN+\probaFP\ne0\right\} \\
 & =\left\{ \aPerformance\in\allPerformances:(1+\beta^{2})\expectedValueSymbol_{\aPerformance}[\randVarImportance]\ne0\right\} \\
 & =\left\{ \aPerformance\in\allPerformances:\expectedValueSymbol_{\aPerformance}[\randVarImportance]\ne0\right\} \\
 & =\domainOfScore[{\rankingScore}]
\end{align*}

Then, we check the equality of the values taken by both scores. By
definition of $\scoreFBeta$, we have, for all $\aPerformance\in\domainOfScore[{\scoreFBeta}]$,
\begin{align*}
\scoreFBeta(\aPerformance) & =\frac{(1+\beta^{2})\probaTP}{(1+\beta^{2})\probaTP+\beta^{2}\probaFN+\probaFP}\\
 & =\frac{0\probaTN+1\probaTP}{0\probaTN+\frac{1}{1+\beta^{2}}\probaFP+\frac{\beta^{2}}{1+\beta^{2}}\probaFN+1\probaTP}\\
 & =\frac{\randVarImportance(\sampleTN)\probaTN+\randVarImportance(\sampleTP)\probaTP}{\randVarImportance(\sampleTN)\probaTN+\randVarImportance(\sampleFP)\probaFP+\randVarImportance(\sampleFN)\probaFN+\randVarImportance(\sampleTP)\probaTP}\\
 & =\rankingScore(\aPerformance)\,.
\end{align*}
In conclusion, $\scoreFBeta=\rankingScore$. As Moreover, as $\randVarImportance(\sampleTN)+\randVarImportance(\sampleTP)=\randVarImportance(\sampleFP)+\randVarImportance(\sampleFN)$,
$\rankingScore$ is canonical.
\end{proof}

%%%%%%%%%%%%%%%%%%%%%%%%%%%%%%%%%%%%%%%%%%%%%%%%%%%%%%%%%%%%%%%%%
% Placer le Kappa de Cohen sur la tuile
%%%%%%%%%%%%%%%%%%%%%%%%%%%%%%%%%%%%%%%%%%%%%%%%%%%%%%%%%%%%%%%%%

\subsubsection{Placing performance orderings induced by the score $\scoreCohenKappa$ on the \tile}

\begin{lemma}[Cohen's $\scoreCohenKappa$ statistic]
Cohen's $\scoreCohenKappa$ statistic increases linearly with a canonical
ranking score when the priors are fixed. In two-class classification,
let the priors of the negative and positive classes be fixed and denoted,
respectively, by $\priorneg$ and $\priorpos$. In this case, $(\priorneg^{2}+\priorpos^{2})\scoreCohenKappa+2\priorneg\priorpos=\rankingScore$
with $\randVarImportance(\sampleTN)=\frac{\priorpos^{2}}{\priorneg^{2}+\priorpos^{2}}$,
$\randVarImportance(\sampleFP)=\frac{1}{2}$, $\randVarImportance(\sampleFN)=\frac{1}{2}$,
and $\randVarImportance(\sampleTP)=\frac{\priorneg^{2}}{\priorneg^{2}+\priorpos^{2}}$.
\end{lemma}

\global\long\def\scoreExpectedAccuracy{EA}%
\global\long\def\scoreRateOfNegativePredictions{\tau_{-}}%
\global\long\def\scoreRateOfPositivePredictions{\tau_{+}}%
\global\long\def\coefTN{\lambda_{tn}}%
\global\long\def\coefFP{\lambda_{fp}}%
\global\long\def\coefFN{\lambda_{fn}}%
\global\long\def\coefTP{\lambda_{tp}}%
\global\long\def\performancesWithGivenPriors{\mathbb{P}^{*}}%

\begin{proof}
Let us begin by introducing the set of performances with fixed and
strictly positive priors:
\[
\performancesWithGivenPriors=\left\{ \aPerformance\in\allPerformances:\aPerformance(\{\sampleFN,\sampleTP\})=\priorpos\right\} \qquad\textrm{ with }\priorpos\in\left(0,1\right)\,.
\]
For the sake of concision, let us pose $\probaTN=\aPerformance(\eventTN)$,
$\probaFP=\aPerformance(\eventFP)$, $\probaFN=\aPerformance(\eventFN)$,
$\probaTP=\aPerformance(\eventTP)$, $\scoreRateOfNegativePredictions=\aPerformance(\{\sampleFN,\sampleTN\})$,
and $\scoreRateOfPositivePredictions=\aPerformance(\{\sampleFP,\sampleTP\})$.

Let us first observe that $\domainOfScore[\scoreCohenKappa]\cap\performancesWithGivenPriors=\domainOfScore[{\rankingScore}]\cap\performancesWithGivenPriors$.
By definition of $\rankingScore$,
\[
\domainOfScore[{\rankingScore}]=\left\{ \aPerformance\in\allPerformances:\expectedValueSymbol_{\aPerformance}[\randVarImportance]\ne0\right\} \,,
\]
and by definition of $\scoreCohenKappa$,
\[
\domainOfScore[\scoreCohenKappa]=\left\{ \aPerformance\in\allPerformances:\priorneg\scoreRateOfNegativePredictions+\priorpos\scoreRateOfPositivePredictions\ne1\right\} \,.
\]
We discuss three cases.
\begin{enumerate}
\item When $\priorpos=0$, which implies that $\probaFN=0\wedge\probaTP=0$,
we have
\begin{align*}
\domainOfScore[\scoreCohenKappa]\cap\performancesWithGivenPriors & =\left\{ \aPerformance\in\allPerformances:\priorneg\scoreRateOfNegativePredictions+\priorpos\scoreRateOfPositivePredictions\ne1\right\} \cap\performancesWithGivenPriors\\
 & =\left\{ \aPerformance\in\allPerformances:\scoreRateOfNegativePredictions\ne1\right\} \cap\performancesWithGivenPriors\\
 & =\left\{ \aPerformance\in\allPerformances:\probaFN+\probaTN\ne1\right\} \cap\performancesWithGivenPriors\\
 & =\left\{ \aPerformance\in\allPerformances:\probaTN\ne1\right\} \cap\performancesWithGivenPriors\\
 & =\left\{ \aPerformance\in\allPerformances:\expectedValueSymbol_{\aPerformance}[\randVarImportance]\ne0\right\} \cap\performancesWithGivenPriors\\
 & =\domainOfScore[{\rankingScore}]\cap\performancesWithGivenPriors\,.
\end{align*}
\item When $\priorpos\in\left]0,1\right[$, we have
\begin{align*}
\domainOfScore[\scoreCohenKappa]\cap\performancesWithGivenPriors & =\left\{ \aPerformance\in\allPerformances:\priorneg\scoreRateOfNegativePredictions+\priorpos\scoreRateOfPositivePredictions\ne1\right\} \cap\performancesWithGivenPriors\\
 & =\left\{ \aPerformance\in\allPerformances\right\} \cap\performancesWithGivenPriors\\
 & =\domainOfScore[{\rankingScore}]\cap\performancesWithGivenPriors\,.
\end{align*}
\item When $\priorpos=1$, which implies that $\probaTN=0\wedge\probaFP=0$,
we have
\begin{align*}
\domainOfScore[\scoreCohenKappa]\cap\performancesWithGivenPriors & =\left\{ \aPerformance\in\allPerformances:\priorneg\scoreRateOfNegativePredictions+\priorpos\scoreRateOfPositivePredictions\ne1\right\} \cap\performancesWithGivenPriors\\
 & =\left\{ \aPerformance\in\allPerformances:\scoreRateOfPositivePredictions\ne1\right\} \cap\performancesWithGivenPriors\\
 & =\left\{ \aPerformance\in\allPerformances:\probaFP+\probaTP\ne1\right\} \cap\performancesWithGivenPriors\\
 & =\left\{ \aPerformance\in\allPerformances:\probaTP\ne1\right\} \cap\performancesWithGivenPriors\\
 & =\left\{ \aPerformance\in\allPerformances:\expectedValueSymbol_{\aPerformance}[\randVarImportance]\ne0\right\} \cap\performancesWithGivenPriors\\
 & =\domainOfScore[{\rankingScore}]\cap\performancesWithGivenPriors\,.
\end{align*}
\end{enumerate}
The restricted domains are thus equal in all cases:

\[
\domainOfScore[\scoreCohenKappa]\cap\performancesWithGivenPriors=\domainOfScore[{\rankingScore}]\cap\performancesWithGivenPriors\,.
\]

We now check the equality of the values taken by both scores. By definition,
for all $\aPerformance\in\domainOfScore[\scoreCohenKappa]$,
\begin{align*}
\scoreCohenKappa(\aPerformance) & =\frac{\scoreAccuracy-\scoreExpectedAccuracy}{1-\scoreExpectedAccuracy}\textrm{ avec }\begin{cases}
\scoreAccuracy=\aPerformance(\{\sampleTN,\sampleTP\}) & \textrm{(the accuracy)}\\
\scoreExpectedAccuracy=\priorneg\scoreRateOfNegativePredictions+\priorpos\scoreRateOfPositivePredictions & \textrm{(the expected accuracy)}
\end{cases}\\
 & =\frac{\hphantom{+\probaFP+\probaFN}(\probaTN+\probaTP)-(\priorneg\probaTN+\priorpos\probaFP+\priorneg\probaFN+\priorpos\probaTP)}{(\probaTN+\probaFP+\probaFN+\probaTP)-(\priorneg\probaTN+\priorpos\probaFP+\priorneg\probaFN+\priorpos\probaTP)}\\
 & =\frac{(1-\priorneg)\probaTN-\priorpos\probaFP-\priorneg\probaFN+(1-\priorpos)\probaTP}{(1-\priorneg)\probaTN+(1-\priorpos)\probaFP+(1-\priorneg)\probaFN+(1-\priorpos)\probaTP}\\
 & =\frac{\priorpos\probaTN-\priorpos\probaFP-\priorneg\probaFN+\priorneg\probaTP}{\priorpos\probaTN+\priorneg\probaFP+\priorpos\probaFN+\priorneg\probaTP}
\end{align*}
Thus,
\[
(\priorneg^{2}+\priorpos^{2})\scoreCohenKappa(\aPerformance)+2\priorneg\priorpos=\frac{\coefTN\probaTN+\coefFP\probaFP+\coefFN\probaFN+\coefTP\probaTP}{\priorpos\probaTN+\priorneg\probaFP+\priorpos\probaFN+\priorneg\probaTP}
\]
with
\begin{align*}
\coefTN & =\priorpos(\priorneg^{2}+\priorpos^{2})+\priorpos(2\priorneg\priorpos)\\
 & =\priorpos(\priorneg^{2}+2\priorneg\priorpos+\priorpos^{2})\\
 & =\priorpos(\priorneg+\priorpos)^{2}\\
 & =\priorpos
\end{align*}
\begin{align*}
\coefFP & =-\priorpos(\priorneg^{2}+\priorpos^{2})+\priorneg(2\priorneg\priorpos)\\
 & =\priorpos(\priorneg^{2}-\priorpos^{2})\\
 & =\priorpos(\priorneg-\priorpos)(\priorneg+\priorpos)\\
 & =\priorpos(\priorneg-\priorpos)
\end{align*}
\begin{align*}
\coefFN & =-\priorneg(\priorneg^{2}+\priorpos^{2})+\priorpos(2\priorneg\priorpos)\\
 & =\priorneg(\priorpos^{2}-\priorneg^{2})\\
 & =\priorneg(\priorpos-\priorneg)(\priorpos+\priorneg)\\
 & =\priorneg(\priorpos-\priorneg)
\end{align*}
\begin{align*}
\coefTP & =\priorneg(\priorneg^{2}+\priorpos^{2})+\priorneg(2\priorneg\priorpos)\\
 & =\priorneg(\priorneg^{2}+2\priorneg\priorpos+\priorpos^{2})\\
 & =\priorneg(\priorneg+\priorpos)^{2}\\
 & =\priorneg
\end{align*}
We continue by eliminating $\probaFP$ and $\probaFN$ from the equation.
For all $\aPerformance\in\domainOfScore[\scoreCohenKappa]\cap\performancesWithGivenPriors$,
\begin{align*}
(\priorneg^{2}+\priorpos^{2})\scoreCohenKappa(\aPerformance)+2\priorneg\priorpos & =\frac{\coefTN\probaTN+\coefFP(\priorneg-\probaTN)+\coefFN(\priorpos-\probaTP)+\coefTP\probaTP}{\priorpos\probaTN+\priorneg(\priorneg-\probaTN)+\priorpos(\priorpos-\probaTP)+\priorneg\probaTP}\\
 & =\frac{(\coefTN-\coefFP)\probaTN+(\coefFP\priorneg+\coefFN\priorpos)+(\coefTP-\coefFN)\probaTP}{(\priorpos-\priorneg)\probaTN+(\priorneg^{2}+\priorpos^{2})+(\priorneg-\priorpos)\probaTP}
\end{align*}
We have
\begin{align*}
\coefTN-\coefFP & =\priorpos-\priorpos(\priorneg-\priorpos)\\
 & =\priorpos(\priorneg+\priorpos)-\priorpos(\priorneg-\priorpos)\\
 & =\priorpos\priorneg+\priorpos^{2}-\priorpos\priorneg+\priorpos^{2}\\
 & =2\priorpos^{2}
\end{align*}
\begin{align*}
\coefFP\priorneg+\coefFN\priorpos & =\priorpos(\priorneg-\priorpos)\priorneg+\priorneg(\priorpos-\priorneg)\priorpos\\
 & =\priorneg\priorpos(\priorneg-\priorpos+\priorpos-\priorneg)\\
 & =0
\end{align*}
\begin{align*}
\coefTP-\coefFN & =\priorneg-\priorneg(\priorpos-\priorneg)\\
 & =\priorneg(\priorpos+\priorneg)-\priorneg(\priorpos-\priorneg)\\
 & =\priorneg\priorpos+\priorneg^{2}-\priorneg\priorpos+\priorneg^{2}\\
 & =2\priorneg^{2}
\end{align*}
So, for all $\aPerformance\in\domainOfScore[\scoreCohenKappa]\cap\performancesWithGivenPriors$,
\[
(\priorneg^{2}+\priorpos^{2})\scoreCohenKappa(\aPerformance)+2\priorneg\priorpos=\frac{2\priorpos^{2}\probaTN+2\priorneg^{2}\probaTP}{(\priorpos-\priorneg)\probaTN+(\priorneg^{2}+\priorpos^{2})+(\priorneg-\priorpos)\probaTP}
\]
Let us now rework the denominator.
\begin{align*}
 & (\priorpos-\priorneg)\probaTN+(\priorneg^{2}+\priorpos^{2})+(\priorneg-\priorpos)\probaTP\\
= & (\priorpos-\priorneg)(\priorpos+\priorneg)\probaTN+(\priorneg^{2}+\priorpos^{2})+(\priorneg-\priorpos)(\priorneg+\priorpos)\probaTP\\
= & (\priorpos^{2}-\priorneg^{2})\probaTN+(\priorneg^{2}+\priorpos^{2})+(\priorneg^{2}-\priorpos^{2})\probaTP\\
= & (\priorpos^{2}-\priorneg^{2})\probaTN+(\priorneg^{2}+\priorpos^{2})(\priorneg+\priorpos)+(\priorneg^{2}-\priorpos^{2})\probaTP\\
= & (\priorpos^{2}-\priorneg^{2})\probaTN+(\priorneg^{2}+\priorpos^{2})\priorneg+(\priorneg^{2}+\priorpos^{2})\priorpos+(\priorneg^{2}-\priorpos^{2})\probaTP\\
= & 2\priorpos^{2}\probaTN+(\priorneg^{2}+\priorpos^{2})(\priorneg-\probaTN)+(\priorneg^{2}+\priorpos^{2})(\priorpos-\probaTP)+2\priorneg^{2}\probaTP\\
= & 2\priorpos^{2}\probaTN+(\priorneg^{2}+\priorpos^{2})\probaFP+(\priorneg^{2}+\priorpos^{2})\probaFN+2\priorneg^{2}\probaTP
\end{align*}
And thus, for all $\aPerformance\in\domainOfScore[\scoreCohenKappa]\cap\performancesWithGivenPriors$,
\begin{align*}
(\priorneg^{2}+\priorpos^{2})\scoreCohenKappa(\aPerformance)+2\priorneg\priorpos= & \frac{2\priorpos^{2}\probaTN+2\priorneg^{2}\probaTP}{2\priorpos^{2}\probaTN+(\priorneg^{2}+\priorpos^{2})\probaFP+(\priorneg^{2}+\priorpos^{2})\probaFN+2\priorneg^{2}\probaTP}\\
= & \frac{\frac{\priorpos^{2}}{\priorneg^{2}+\priorpos^{2}}\probaTN+\frac{\priorneg^{2}}{\priorneg^{2}+\priorpos^{2}}\probaTP}{\frac{\priorpos^{2}}{\priorneg^{2}+\priorpos^{2}}\probaTN+\frac{1}{2}\probaFP+\frac{1}{2}\probaFN+\frac{\priorneg^{2}}{\priorneg^{2}+\priorpos^{2}}\probaTP}\\
= & \frac{\randVarImportance(\sampleTN)\probaTN+\randVarImportance(\sampleTP)\probaTP}{\randVarImportance(\sampleTN)\probaTN+\randVarImportance(\sampleFP)\probaFP+\randVarImportance(\sampleFN)\probaFN+\randVarImportance(\sampleTP)\probaTP}\\
= & \rankingScore(\aPerformance)
\end{align*}
In conclusion, $(\priorneg^{2}+\priorpos^{2})\scoreCohenKappa+2\priorneg\priorpos=\rankingScore$
on $\domainOfScore[\scoreCohenKappa]\cap\performancesWithGivenPriors=\domainOfScore[{\rankingScore}]\cap\performancesWithGivenPriors$.
Moreover, as $\randVarImportance(\sampleTN)+\randVarImportance(\sampleTP)=\randVarImportance(\sampleFP)+\randVarImportance(\sampleFN)$,
$\rankingScore$ is canonical.
\end{proof}

%%%%%%%%%%%%%%%%%%%%%%%%%%%%%%%%%%%%%%%%%%%%%%%%%%%%%%%%%%%%%%%%%
% Placer la "weighted accuracy" sur la tuile
%%%%%%%%%%%%%%%%%%%%%%%%%%%%%%%%%%%%%%%%%%%%%%%%%%%%%%%%%%%%%%%%%

\subsubsection{Placing performance orderings induced by the score $\scoreWeightedAccuracy$ on the \tile}

\begin{lemma}[The weighted accuracies] In two-class classification, let the priors of the negative
and positive classes be fixed, strictly positive, and denoted, respectively,
by $\priorneg$ and $\priorpos$. In this case, all weighted accuracies
$\scoreWeightedAccuracy=(1-\alpha)\scoreTNR+\alpha\scoreTPR$, $\alpha\in[0,1]$,
are canonical ranking scores: $\scoreWeightedAccuracy=\rankingScore$
with
\[
\begin{cases}
\randVarImportance(\sampleTN)=\randVarImportance(\sampleFP)=\frac{\frac{1-\alpha}{\priorneg}}{\frac{1-\alpha}{\priorneg}+\frac{\alpha}{\priorpos}}\\
\randVarImportance(\sampleFN)=\randVarImportance(\sampleTP)=\frac{\frac{\alpha}{\priorpos}}{\frac{1-\alpha}{\priorneg}+\frac{\alpha}{\priorpos}} & .
\end{cases}
\]
\end{lemma}

\begin{proof}
Let us begin by introducing the set of performances with fixed and
strictly positive priors:
\[
\performancesWithGivenPriors=\left\{ \aPerformance\in\allPerformances:\aPerformance(\{\sampleFN,\sampleTP\})=\priorpos\right\} \qquad\textrm{ with }\priorpos\in\left(0,1\right)\,.
\]
For the sake of concision, we pose $\probaTN=\aPerformance(\eventTN)$,
$\probaFP=\aPerformance(\eventFP)$, $\probaFN=\aPerformance(\eventFN)$,
and $\probaTP=\aPerformance(\eventTP)$.

We first check the equality of the restricted domains. The domain
of $\scoreWeightedAccuracy$ is $\domainOfScore[\scoreTNR]\cap\domainOfScore[\scoreTPR]$,
that is 
\[
\domainOfScore[\scoreWeightedAccuracy]=\left\{ \aPerformance\in\allPerformances:\aPerformance(\{\sampleFN,\sampleTP\})\notin\{0,1\}\right\} \,,
\]
 and thus the restricted domain of the weighted accuracy $\scoreWeightedAccuracy$
is
\[
\domainOfScore[\scoreWeightedAccuracy]\cap\performancesWithGivenPriors=\performancesWithGivenPriors\,.
\]
For the domain of $\rankingScore$, we have to take into account the
fact that $\randVarImportance$ should be well defined ($\priorneg\ne0$
and $\priorpos\ne0$) and that the mathematical expectation of the
importance should be non-zero, which is always the case: 
\begin{align*}
\expectedValueSymbol_{\aPerformance}[\randVarImportance]\ne0 & \Leftrightarrow\frac{\frac{1-\alpha}{\priorneg}}{\frac{1-\alpha}{\priorneg}+\frac{\alpha}{\priorpos}}(\probaTN+\probaFP)+\frac{\frac{\alpha}{\priorpos}}{\frac{1-\alpha}{\priorneg}+\frac{\alpha}{\priorpos}}(\probaFN+\probaTP)\ne0\\
 & \Leftrightarrow\frac{1-\alpha}{\priorneg}(\probaTN+\probaFP)+\frac{\alpha}{\priorpos}(\probaFN+\probaTP)\ne0\\
 & \Leftrightarrow\frac{1-\alpha}{\priorneg}\priorneg+\frac{\alpha}{\priorpos}\priorpos\ne0\\
 & \Leftrightarrow1\ne0
\end{align*}
Thus, the restricted domain of the ranking score $\rankingScore$
is
\[
\domainOfScore[{\rankingScore}]\cap\performancesWithGivenPriors=\performancesWithGivenPriors\,.
\]
\[
\domainOfScore[{\rankingScore}]=\left\{ \aPerformance\in\allPerformances:\expectedValueSymbol_{\aPerformance}[\randVarImportance]\ne0\right\} 
\]
The restricted domains are thus equal:

\[
\domainOfScore[\scoreWeightedAccuracy]\cap\performancesWithGivenPriors=\domainOfScore[{\rankingScore}]\cap\performancesWithGivenPriors\,.
\]

We now check the equality of the values taken by both scores. For
all $\aPerformance\in\domainOfScore[\scoreWeightedAccuracy]\cap\performancesWithGivenPriors$,
we have:
\begin{align*}
\scoreWeightedAccuracy(\aPerformance) & =(1-\alpha)\scoreTNR(\aPerformance)+\alpha\scoreTPR(\aPerformance)\\
 & =(1-\alpha)\frac{\probaTN}{\priorneg}+\alpha\frac{\probaTP}{\priorpos}\\
 & =\frac{1-\alpha}{\priorneg}\probaTN+\frac{\alpha}{\priorpos}\probaTP\\
 & =\frac{\frac{1-\alpha}{\priorneg}\probaTN+\frac{\alpha}{\priorpos}\probaTP}{\frac{1-\alpha}{\priorneg}\priorneg+\frac{\alpha}{\priorpos}\priorpos}\\
 & =\frac{\frac{1-\alpha}{\priorneg}\probaTN+\frac{\alpha}{\priorpos}\probaTP}{\frac{1-\alpha}{\priorneg}(\probaTN+\probaFP)+\frac{\alpha}{\priorpos}(\probaFN+\probaTP)}\\
 & =\frac{\frac{1-\alpha}{\priorneg}\probaTN+\frac{\alpha}{\priorpos}\probaTP}{\frac{1-\alpha}{\priorneg}\probaTN+\frac{1-\alpha}{\priorneg}\probaFP+\frac{\alpha}{\priorpos}\probaFN+\frac{\alpha}{\priorpos}\probaTP}\\
 & =\frac{\randVarImportance(\sampleTN)\probaTN+\randVarImportance(\sampleTP)\probaTP}{\randVarImportance(\sampleTN)\probaTN+\randVarImportance(\sampleFP)\probaFP+\randVarImportance(\sampleFN)\probaFN+\randVarImportance(\sampleTP)\probaTP}\\
 & =\rankingScore(\aPerformance)
\end{align*}
In conclusion, $\scoreWeightedAccuracy=\rankingScore$ on $\domainOfScore[\scoreWeightedAccuracy]\cap\performancesWithGivenPriors=\domainOfScore[{\rankingScore}]\cap\performancesWithGivenPriors$.
Moreover, as $\randVarImportance(\sampleTN)+\randVarImportance(\sampleTP)=\randVarImportance(\sampleFP)+\randVarImportance(\sampleFN)$,
$\rankingScore$ is canonical.
\end{proof}

The previous lemma can be particularized for the balanced accuracy
and for the accuracy. Taking $\alpha=\frac{1}{2}$, we obtain $\scoreBalancedAccuracy=\scoreWeightedAccuracy=\rankingScore$
with $\randVarImportance(\sampleTN)=\randVarImportance(\sampleFP)=\priorpos$
and $\randVarImportance(\sampleFN)=\randVarImportance(\sampleTP)=\priorneg$.
Taking $\alpha=\priorpos$, we obtain $\scoreAccuracy=\scoreWeightedAccuracy=\rankingScore$
with $\randVarImportance(\sampleTN)=\randVarImportance(\sampleFP)=\nicefrac{1}{2}$
and $\randVarImportance(\sampleFN)=\randVarImportance(\sampleTP)=\nicefrac{1}{2}$.

%%%%%%%%%%%%%%%%%%%%%%%%%%%%%%%%%%%%%%%%%%%%%%%%%%%%%%%%%%%%%%%%%
% Placer d'autres scores sur la tuile.
%%%%%%%%%%%%%%%%%%%%%%%%%%%%%%%%%%%%%%%%%%%%%%%%%%%%%%%%%%%%%%%%%

\subsubsection{Placing performance orderings induced by some other non-probabilistic scores on the \tile}

\begin{lemma}[Other particular scores]
Let us consider performance orderings induced by scores by the mechanism
described in the \first theorem of \paperA. In two-class
classification, Youden's index $\scoreYoudenJ$ leads to the same performance
ordering as the balanced accuracy $\scoreBalancedAccuracy$, and Jaccard's
coefficient for the positive class $\scoreJaccardPos$ leads to the
same performance ordering as $\scoreFBeta[1]$. Moreover, when the
class priors are fixed and non-zero, the standardized negative predictive
value $\scoreSNPV$ leads to the same performance ordering as the
negative predictive value $\scoreNPV$ when $\aPerformance(\{\sampleTN,\sampleFN\})\ne0$,
the negative likelihood ratio leads to the dual performance ordering
of $\scoreNPV$, the standardized positive predictive value $\scoreSPPV$
leads to the same performance ordering as the positive predictive
value $\scorePPV$ when $\aPerformance(\{\sampleTP,\sampleFP\})\ne0$,
and the positive likelihood ratio $\scorePLR$ leads also to the same
performance ordering as $\scorePPV$.
\end{lemma}

\begin{proof}
For the sake of concision, let us pose $\probaTN=\aPerformance(\eventTN)$,
$\probaFP=\aPerformance(\eventFP)$, $\probaFN=\aPerformance(\eventFN)$,
and $\probaTP=\aPerformance(\eventTP)$.
\begin{itemize}
\item Youden's index is defined as $\scoreYoudenJ=\scoreTNR+\scoreTPR-1$,
and the balanced accuracy as $\scoreBalancedAccuracy=\nicefrac{1}{2}\scoreTNR+\nicefrac{1}{2}\scoreTPR$.
They have the same domain: $\domainOfScore[\scoreYoudenJ]=\domainOfScore[\scoreBalancedAccuracy]=\domainOfScore[\scoreTNR]\cap\domainOfScore[\scoreTPR]$.
Trivially, 
\[
\scoreYoudenJ=2\scoreBalancedAccuracy-1\,.
\]
Thus, 
\[
\frac{\partial\scoreYoudenJ}{\partial\scoreBalancedAccuracy}=2>0\,.
\]
As there is a strictly increasing relationship between $\scoreYoudenJ$
and $\scoreBalancedAccuracy$, these two scores lead to same performance
ordering.
\item Jaccard's coefficient for the positive class is defined as $\scoreJaccardPos=\aScore_{\{\sampleTP\}\vert\{\sampleFP,\sampleFN,\sampleTP\}}$
and the F-one score as $\scoreFBeta[1]:\aPerformance\mapsto\frac{2\probaTP}{\probaFP+\probaFN+2\probaTP}$.
Their respective domains
\begin{align*}
\domainOfScore[\scoreJaccardPos] & =\left\{ \aPerformance\in\allPerformances:\aPerformance(\{\sampleFP,\sampleFN,\sampleTP\})\ne0\right\} \\
 & =\left\{ \aPerformance\in\allPerformances:\probaFP\ne0\vee\probaFN\ne0\vee\probaTP\ne0\right\} 
\end{align*}
and
\begin{align*}
\domainOfScore[{\scoreFBeta[1]}] & =\left\{ \aPerformance\in\allPerformances:\probaFP+\probaFN+2\probaTP\ne0\right\} \\
 & =\left\{ \aPerformance\in\allPerformances:\probaFP\ne0\vee\probaFN\ne0\vee\probaTP\ne0\right\} 
\end{align*}
are equal. Trivially, 
\[
\scoreFBeta[1]=\frac{2\scoreJaccardPos}{1+\scoreJaccardPos}\,.
\]
Thus, 
\[
\frac{\partial\scoreFBeta[1]}{\partial\scoreJaccardPos}=\frac{2}{(1+\scoreJaccardPos)^{2}}>0\,.
\]
As there is a strictly increasing relationship between $\scoreFBeta[1]$
and $\scoreJaccardPos$, these two scores lead to same performance
ordering.
\item The standardized negative predictive value is defined as $\scoreSNPV=\frac{\scoreTNR}{\scoreTNR+(1-\scoreTPR)}$
and the negative predictive value as $\scoreNPV=\aScore_{\{\sampleTN\}\vert\{\sampleTN,\sampleFN\}}$.
Let us consider the performances such that $\probaTN+\probaFN\ne0$,
$\probaTN+\probaFP=\priorneg\ne0$, and $\probaFN+\probaTP=\priorpos\ne0$.
We have:
\begin{multline*}
\scoreSNPV=\frac{\scoreTNR}{\scoreTNR+(1-\scoreTPR)}=\frac{\frac{\probaTN}{\priorneg}}{\frac{\probaTN}{\priorneg}+\frac{\probaFN}{\priorpos}}\\
=\frac{\frac{\probaTN}{\probaTN+\probaFN}\frac{1}{\priorneg}}{\frac{\probaTN}{\probaTN+\probaFN}\frac{1}{\priorneg}+\frac{\probaFN}{\probaTN+\probaFN}\frac{1}{\priorpos}}=\frac{\scoreNPV\frac{1}{\priorneg}}{\scoreNPV\frac{1}{\priorneg}+(1-\scoreNPV)\frac{1}{\priorpos}}\,.
\end{multline*}
Thus,
\[
\frac{\partial\scoreSNPV}{\partial\scoreNPV}=\frac{\frac{1}{\priorneg}\frac{1}{\priorpos}}{\left(\scoreNPV\frac{1}{\priorneg}+(1-\scoreNPV)\frac{1}{\priorpos}\right)^{2}}>0\,.
\]
As there is a strictly increasing relationship between $\scoreSNPV$
and $\scoreNPV$, these two scores lead to same performance ordering.
\item The negative likelihood ratio is defined as $\scoreNLR=\frac{1-\scoreTPR}{\scoreTNR}$
and the negative predictive value as $\scoreNPV=\aScore_{\{\sampleTN\}\vert\{\sampleTN,\sampleFN\}}$.
Let us consider the performances such that $\probaTN\ne0$, $\probaTN+\probaFP=\priorneg$,
and $\probaFN+\probaTP=\priorpos\ne0$. We have:
\[
\scoreNLR=\frac{1-\scoreTPR}{\scoreTNR}=\frac{\frac{\probaFN}{\priorpos}}{\frac{\probaTN}{\priorneg}}=\frac{\probaFN}{\probaTN}\frac{\priorneg}{\priorpos}=\frac{1-\scoreNPV}{\scoreNPV}\frac{\priorneg}{\priorpos}\,.
\]
Thus,
\[
\frac{\partial\scoreNLR}{\partial\scoreNPV}=-\frac{1}{\scoreNPV^{2}}\frac{\priorneg}{\priorpos}<0\,.
\]
As there is a strictly decreasing relationship between $\scoreNLR$
and $\scoreNPV$, these two scores lead to dual performance orderings.
\item The standardized positive predictive value is defined as $\scoreSPPV=\frac{\scoreTPR}{(1-\scoreTNR)+\scoreTPR}$
and the positive predictive value as $\scorePPV=\aScore_{\{\sampleTP\}\vert\{\sampleTP,\sampleFP\}}$.
Let us consider the performances such that $\probaTP+\probaFP\ne0$,
$\probaTN+\probaFP=\priorneg\ne0$, and $\probaFN+\probaTP=\priorpos\ne0$.
We have:
\begin{multline*}
\scoreSPPV=\frac{\scoreTPR}{(1-\scoreTNR)+\scoreTPR}=\frac{\frac{\probaTP}{\priorpos}}{\frac{\probaFP}{\priorneg}+\frac{\probaTP}{\priorpos}}\\
=\frac{\frac{\probaTP}{\probaTP+\probaFP}\frac{1}{\priorpos}}{\frac{\probaFP}{\probaTP+\probaFP}\frac{1}{\priorneg}+\frac{\probaTP}{\probaTP+\probaFP}\frac{1}{\priorpos}}=\frac{\scorePPV\frac{1}{\priorpos}}{(1-\scorePPV)\frac{1}{\priorneg}+\scorePPV\frac{1}{\priorpos}}
\end{multline*}
Thus,
\[
\frac{\partial\scoreSPPV}{\partial\scorePPV}=\frac{\frac{1}{\priorneg}\frac{1}{\priorpos}}{\left((1-\scorePPV)\frac{1}{\priorneg}+\scorePPV\frac{1}{\priorpos}\right)^{2}}>0\,.
\]
As there is a strictly increasing relationship between $\scoreSPPV$
and $\scorePPV$, these two scores lead to same performance ordering.
\item The positive likelihood ratio is defined as $\scorePLR=\frac{\scoreTPR}{1-\scoreTNR}$
and the positive predictive value as $\scorePPV=\aScore_{\{\sampleTP\}\vert\{\sampleTP,\sampleFP\}}$.
Let us consider the performances such that $\probaFP\ne0$, $\probaTN+\probaFP=\priorneg$,
and $\probaFN+\probaTP=\priorpos\ne0$. We have:
\[
\scorePLR=\frac{\scoreTPR}{1-\scoreTNR}=\frac{\frac{\probaTP}{\priorpos}}{\frac{\probaFP}{\priorneg}}=\frac{\probaTP}{\probaFP}\frac{\priorneg}{\priorpos}=\frac{\scorePPV}{1-\scorePPV}\frac{\priorneg}{\priorpos}\,.
\]
Thus,
\[
\frac{\partial\scorePLR}{\partial\scorePPV}=\frac{1}{\left(1-\scorePPV\right)^{2}}\frac{\priorneg}{\priorpos}>0\,.
\]
As there is a strictly increasing relationship between $\scorePLR$
and $\scorePPV$, these two scores lead to same performance ordering.
\end{itemize}
\end{proof}

\clearpage
\subsection{Supplementary material about \cref{sec:rankings-on-tile}}

\subsubsection{Algorithmic contributions}

It is easy to calculate, for a given importance $\randVarImportance$, which entity $\anEntity$ is the best among a set $\entitiesToRank$, based on their performances: it is the entity that has the highest score value. By performing this calculation at numerous points on the \tile, it is possible to get a good idea of the set of importances for which a given entity is the best, occupies the $r$-th place, or is ranked last. However, it is very interesting to be able to give an explicit representation of this set, in an exact way, by computing it efficiently. We now present some algorithmic contributions that make this possible. We present them for the particular case in which all compared performances have the same class priors.

\paragraph{When $\priorneg=\priorpos$.}
Let us first consider the case in which two performances with balanced priors, $\aPerformance_1,\aPerformance_2$, are compared. In our toy example (\cref{fig:toy-example}), we made the observation that the regions where the different classifiers are ranked first are convex polygons when $\priorneg=\priorpos$. In fact, one can show that the canonical importances $\randVarCanonicalImportance$ for which $\aPerformance_1$ is better or equivalent than $\aPerformance_2$ are given by a linear inequality in the variables $a$ and $b$:
\begin{equation}
    \aPerformance_1 \relBetterOrEquivalent \aPerformance_2
    \Leftrightarrow
    \lambda_a(\aPerformance_1,\aPerformance_2) a + \lambda_b(\aPerformance_1,\aPerformance_2) b + \lambda_0(\aPerformance_1,\aPerformance_2) \ge 0
\end{equation}
with
\begin{equation}
    \begin{cases}
        \lambda_a(\aPerformance_1,\aPerformance_2)=\scoreFPR(\aPerformance_1) \scoreFNR(\aPerformance_2)-\scoreFNR(\aPerformance_1) \scoreFPR(\aPerformance_2)\\
        \lambda_b(\aPerformance_1,\aPerformance_2)=\scoreTPR(\aPerformance_1) \scoreTNR(\aPerformance_2) - \scoreTNR(\aPerformance_1) \scoreTPR(\aPerformance_2)\\
        \lambda_0(\aPerformance_1,\aPerformance_2)=\scoreTNR(\aPerformance_1) - \scoreTNR(\aPerformance_2)
    \end{cases}\,.
\end{equation}
Given that the \tile is bounded by four other linear inequalities in the variables $a$ and $b$,
\begin{equation}
\begin{cases}
(1)a+(0)b+(0)\ge0\,.\\
(-1)a+(0)b+(1)\ge0\,.\\
(0)a+(1)b+(0)\ge0\,.\\
(0)a+(-1)b+(1)\ge0\,.
\end{cases}
\end{equation}
the zone in the \tile where $\aPerformance_1$ is better or equivalent than $\aPerformance_2$ is the intersection between $5$ half-planes, thus either an empty set or a convex polygon. This can be trivially generalized to the computation of the zone in which all performances in a set $\aSetOfPerformances_1$ are better or equivalent to all performances in a set $\aSetOfPerformances_2$. In this case, we obtain either an empty set or a convex polygon that is the intersection of $|\aSetOfPerformances_1| |\aSetOfPerformances_2|+4$ half planes. If needed, the representation of this polygon as a set of linear inequalities can be converted in an equivalent representation based on its vertices and edges.

\paragraph{When $\priorneg\ne\priorpos$.}
In this case, we proceed in three steps: (1) we apply a tartget/prior shift to the performances to fall back in the case of uniform priors, then (2) we calculate the polygon as described above, and finally (3) we cancel the effect of a tartget/prior shift by applying a deformation to this polygon according to what was described in \cref{sec:operations-on-performances}. This transformation is continuous and invertible, so that the border of the computed zone corresponds to the contour of the polygon. For that reason, it is convenient to represent the polygons based on their vertices and edges. Applying the transformation to them can be done by simply discretizing the edges and applying the transformation to the points. In this way, the resulting zone is approximated as a polygon.

\clearpage
\subsection{Supplementary material about \cref{sec:no-skill-performances}}

\subsubsection{What is achievable by no-skilled classifiers?} 

\Cref{fig:values-no-skill} uses the \tile to depict, for all canonical ranking scores, the maximum value achievable by the no-skill performances.
\begin{figure}
\begin{centering}
\hfill
\includegraphics[width=0.3\linewidth]{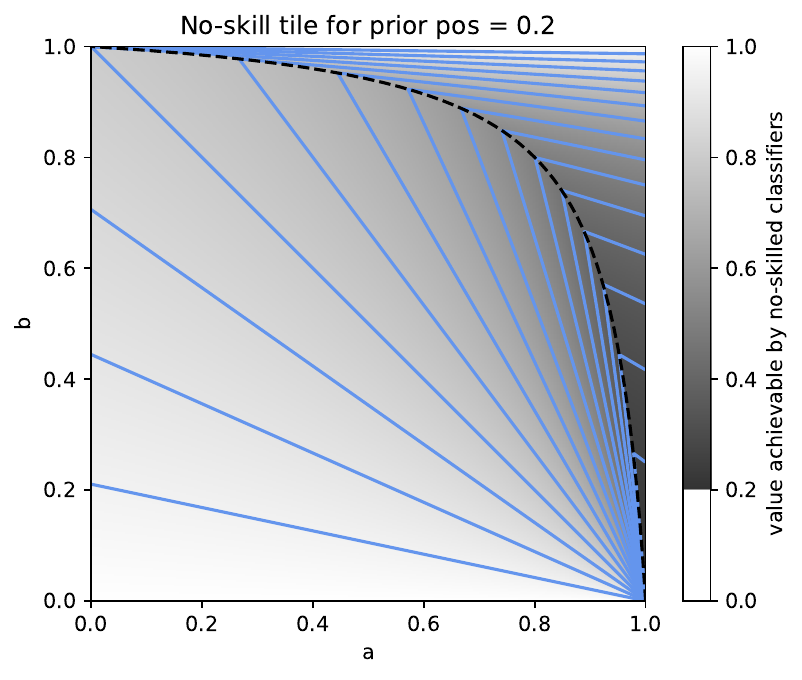}
\hfill
\hfill
\includegraphics[width=0.3\linewidth]{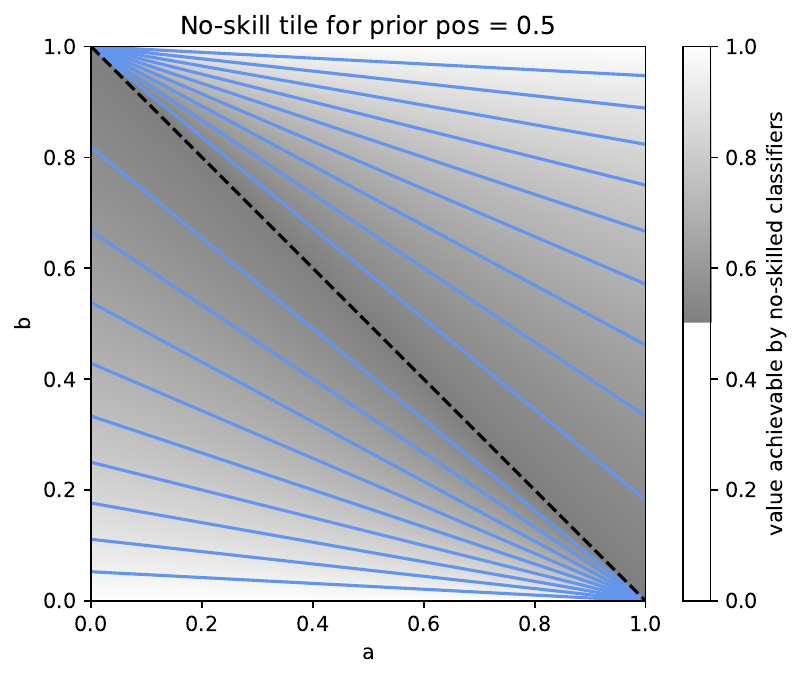}
\hfill
\hfill
\includegraphics[width=0.3\linewidth]{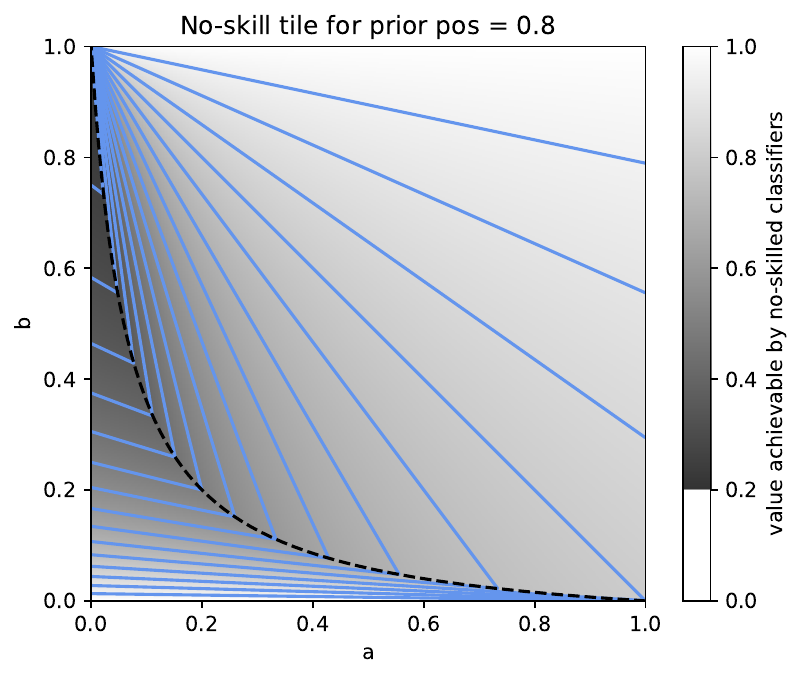}
\hfill
\par\end{centering}
\caption{\tiles showing the values of canonical ranking scores achievable by no-skill classifiers when the prior of the positive class is $0.2$ (left), $0.5$ (center), and $0.8$ (right). In all cases, the \tile is divided into two parts. On the bottom left part, the best no-skill classifier is the one predicting always the negative class. In the upper-right part, the best no-skill classifier is the one predicting always the positive class. Note that the limit between the two parts is the curve $\tileCurvePriors$ given in \cref{fig:curves-no-skill}.\label{fig:values-no-skill}}

\end{figure}

\clearpage
\subsection{Averaging all canonical ranking scores: the \emph{Volume Under Tile}}
\label{sec:score-under-tile}

\subsubsection{Definition}

Let us study the score $\scoreVUT$ (\emph{Volume Under Tile}):
\begin{equation}
    \scoreVUT:\allPerformances\rightarrow[0,1]:\aPerformance\mapsto\scoreVUT(\aPerformance)=\int_{a=0}^{1}\int_{b=0}^{1}\canonicalRankingScore(\aPerformance)\,\mathrm{d}b\,\mathrm{d}a\,.
\end{equation}

\subsubsection{Closed-form expression}

It can be showed that:
\begin{enumerate}
    \item When $\aPerformance(\eventTP)=\aPerformance(\eventTN)$ and $\aPerformance(\eventFN)=\aPerformance(\eventFP)$:
    \begin{align}
        \scoreVUT(\aPerformance) & =\canonicalRankingScore(\aPerformance)\,\forall (a,b)\in [0,1] \\ & \qquad =\scoreAccuracy(\aPerformance)=\scoreTNR(\aPerformance)=\scoreTPR(\aPerformance) \\ & \qquad \qquad =\scoreNPV(\aPerformance)=\scorePPV(\aPerformance)=\scoreFBeta(\aPerformance)=\ldots
    \end{align}
    \item When $\aPerformance(\eventTP)=\aPerformance(\eventTN)$ and $\aPerformance(\eventFN)\ne\aPerformance(\eventFP)$:
    \begin{equation}
        \scoreVUT(\aPerformance) = \frac{\aPerformance(\eventTN)}{\aPerformance(\eventFN)-\aPerformance(\eventFP)}\left(\ln\left(\aPerformance(\{\sampleTN,\sampleFN\})\right)-\ln\left(\aPerformance(\{\sampleTN,\sampleFP\}))\right)\right)
    \end{equation}
    \item When $\aPerformance(\eventTP)\ne\aPerformance(\eventTN)$ and $\aPerformance(\eventFN)=\aPerformance(\eventFP)$:
    \begin{equation}
        \scoreVUT(\aPerformance) = 1-\frac{\aPerformance(\eventFN)}{\aPerformance(\eventTP)-\aPerformance(\eventTN)}\left(\ln\left(\aPerformance(\{\sampleTP,\sampleFN\})\right)-\ln\left(\aPerformance(\{\sampleTN,\sampleFN\}))\right)\right)
    \end{equation}
    \item When $\aPerformance(\eventTP)\ne\aPerformance(\eventTN)$ and $\aPerformance(\eventFN)\ne\aPerformance(\eventFP)$:
    \begin{equation}
        \scoreVUT(\aPerformance)=\frac{1}{2}-\frac{1}{2}\frac{\begin{array}{cc}
             & \left(\aPerformance(\eventTN)^{2}-\aPerformance(\eventFN)^{2}\right)\ln\left(\aPerformance(\{\sampleTN,\sampleFN\})\right)\\
            + & \left(\aPerformance(\eventTP)^{2}-\aPerformance(\eventFP)^{2}\right)\ln\left(\aPerformance(\{\sampleTP,\sampleFP\})\right)\\
            + & \left(\aPerformance(\eventFP)^{2}-\aPerformance(\eventTN)^{2}\right)\ln\left(\aPerformance(\{\sampleFP,\sampleTN\})\right)\\
            + & \left(\aPerformance(\eventFN)^{2}-\aPerformance(\eventTP)^{2}\right)\ln\left(\aPerformance(\{\sampleFN,\sampleTP\})\right)
        \end{array}}{(\aPerformance(\eventTP)-\aPerformance(\eventTN))(\aPerformance(\eventFN)-\aPerformance(\eventFP))}
        \label{eq:Anthony}
    \end{equation}
\end{enumerate}

\begin{proof}[Proof of \cref{eq:Anthony}]

For the sake of concision, let us rewrite under volume under the \tile as the following double integral, where $a,b,c,d$ are positive numbers:

$\int_{x=0}^1 \int_{y=0}^1 \frac{(1-x) a+x d}{(1-x) a+(1-y) b+y c+x d} d x d y$

$=\int_{x=0}^1 \int_{y=0}^1 \frac{(d-a) x+a}{(d-a) x+a+b+(c-b) y} d x d y$

\noindent Let's substitute: $p=(d-a), q=a, r=(d-a)=p, s=a+b+(c-b) y$

\noindent Knowing that$\int \frac{p x+q}{r x+s} d x=\int f(x) d x =\frac{p x}{r}+\frac{1}{r}\left(q-\frac{p s}{r}\right) \ln |r x+s|+C \,$, we have 

$\int_0^1 f(x) d x=\frac{p}{r}+\frac{1}{r}\left(q-\frac{p s}{r}\right) \ln |r+s|-\frac{1}{r}\left(q-\frac{p s}{r}\right) \ln |s|$

\noindent As $p=r$, we can simplify the equation as:

$\int_0^1 f(x) d x=1+\left(\frac{q}{r}-\frac{s}{r}\right) \ln |r+s|-\left(\frac{q}{r}-\frac{s}{r}\right) \ln |s|$

$=1+\frac{a-a-b+(b-c) y}{d-a} \ln |d-a+a+b+(c-b) y|-\left(\frac{-b+(b-c) y}{d-a}\right)\ln |a+b+(c-b) y|$

$=1+\left[\frac{(b-c)}{d-a} y-\frac{b}{d-a}\right] \ln |b+d+(c-b) y| - \left[\frac{(b-c)}{d-a} y-\frac{b}{d-a}\right] \ln |a+b+(c-b) y|
$

\noindent Let's now substitute: $\alpha=\frac{b-c}{d-a}, \beta=-\frac{b}{d-a}, \gamma=b+d, \delta=(c-b), \varepsilon=a+b $

$ \rightarrow \int_0^1 1 d y+\int_0^1(\alpha y+\beta) \ln  \underbrace{|\delta y+\gamma|}_{\geqslant 0} d y-\int_0^1(\alpha y+\beta) \ln \underbrace{|\delta y+\varepsilon|}_{\geqslant 0} d y$

\noindent Knowing that $\int(a x+b) \ln (c x+d) d x=\int g(x) d x $

$=\frac{1}{4 c^2} [2 (c x+d) \ln (c x+d)(a c x-a d+2 b c) -c x(a c x-2 a d+4 b c)]+C \,$, \\we have

$
\begin{aligned}
     \int_0^1 g(x) d x= & 
    \frac{1}{4 c^2}[2(c+d) \ln (c+d)(a c-a d+2 b c)-c(a c-2 a d+4 b c)] \\
   &   -\frac{1}{4 c^2}[2 d \ln (d)(2 b c-a c)] \\
    & 
     \end{aligned}
$

$
\begin{aligned}
=1&+ \frac{1}{4 \delta^2}\left[2 (\delta+\gamma) \ln (\delta+\gamma)(\alpha \delta-\alpha \gamma+2 \beta \delta) \right.\\
&-\delta\left(\alpha \delta-2 \alpha \gamma+4 \beta \delta)-2 \gamma \ln (2 \beta \delta-\alpha \delta)\right]\\
&-\frac{1}{4 \delta^2}\left[2 (\delta+\varepsilon) \ln (\delta+\varepsilon)(\alpha \delta-\alpha \varepsilon+2 \beta \delta) \right.\\
&-\delta\left(\alpha \delta-2 \alpha \varepsilon+4 \beta \delta)-2 \varepsilon \ln (2 \beta \delta-\alpha \delta)\right]
\end{aligned}
$

$
\begin{aligned}
&=1+ \frac{1}{4(c-b)^2(d-a)}\left[2(c+d) \ln (c+d)\left[-(b-c)^2-(b-c)(b+d)+2 b(b-c)\right]\right. \\
& -2(c+a) \ln (c+a)\left[-(b-c)^2-(b-c)(b+a)+2 b(b-c)\right] \\
& -2(b+d) \ln (b+d)[2 (c-b)(-b)+(c-b)(b+d)] \\
& +2(a+b) \ln (a+b)[2(c-b)(-b)+(c-b)(b+a)] \\
& -(c-b)\left[-(c-b)^2-2 (b-c)(b+d)-4 b(c-b)\right] \\
& +(c-b)\left[-(c-b)^2-2(b-c)(b+a)-4 b(c-b)]\right] \\
\end{aligned}
$

$
\begin{aligned}
& =1+\frac{1}{4(c-b)(d-a)}[2 (c+d) \ln (c+d)[(b-c)+(b+d)-2 b] \\
& -2(c+a) \ln (c+a)[(b-c)+(b+a)-2 b] \\
& -2(b+d) \ln (b+d)[d-b] \\
& +2(a+b) \ln (a+b)[a-b] \\
& +\left[(c-b)^2+2(b-c)(b+d)+4 b(c-b)\right] \\
& -\left[(c-b)^2+2(b-c)(b+a)+4 b(c-b)\right] \\
\end{aligned}
$

\noindent The volume under the \tile is thus analytically expressed as:

$\begin{aligned} 
\scoreVUT &=\frac{1}{2} - \frac{1}{2(c-b)(d-a)}\left[\left(c^2-d^2\right) \ln (c+d)\right. \\ 
& +\left(a^2-c^2\right) \ln(a+c) \\ 
& +\left(d^2-b^2\right) \ln (b+d) \\ 
& \left.+\left(b^2-a^2\right) \ln (a+b)\right] \\ 
& \end{aligned}$

\qed\end{proof}

\subsubsection{Discussion: can we use it to rank?}
The ordering induced by $\scoreVUT$ does not satisfy Axiom~\ref{axiom:combinations}. Nevertheless, it is interesting to note that it has a very high rank correlation with the accuracy $\scoreAccuracy$ (Spearman's $\rho$ is about $0.996$ for a uniform performance distribution, \ie a Dirichlet distribution with all concentration parameters set to $\alpha=1$).

\stopcontents[mytoc]

\end{document}